\documentclass[12pt]{colt2026} 
\jmlrpages{}
\jmlrvolume{}
\jmlryear{}
\jmlrproceedings{}{}
\usepackage{hyperref}       
\usepackage{url}      
\usepackage{booktabs}       
\usepackage{amsfonts}       
\usepackage{nicefrac}    
\usepackage{natbib}
\usepackage{microtype}      
\usepackage{xcolor}  
\newcommand{\lippax}{\textsc{LIPPAX}}
\newcommand{\lippaxg}{\textsc{SLIPPAX}}

\usepackage{hyperref}       
\usepackage{url}            
\usepackage{booktabs}       
\usepackage{amsfonts}       
\usepackage{nicefrac}       
\usepackage{microtype}      
\usepackage{xcolor}         
\usepackage{paralist,amsmath, amssymb}
\usepackage{algorithm,algorithmic}
 \usepackage{graphicx}
 \usepackage{float}
 \usepackage{pdfpages}
 \usepackage{tikz}

\usepackage{xcolor}
\usepackage{graphics}
\usepackage{units}
 \usepackage{amssymb, multirow, paralist, color}

 \usepackage{textcase, booktabs,makecell}
 \usepackage{thmtools,thm-restate}
 \usepackage{mathtools}
\usepackage{enumitem}
\usepackage{titlesec}
\usepackage{tikz}
\usetikzlibrary{positioning}

 \usepackage{thmtools,thm-restate}
 \usepackage{mathtools}
\usepackage{enumitem}
\def \y {\mathbf{y}}
\def \E {\mathbb{E}}
\def \x {\mathbf{x}}

\def \z {\mathbf{z}}
\def \u {\mathbf{u}}

\def \s {\mathbf{s}}
\def \R {\mathbb{R}}

\def \Z {\mathcal{Z}}

\def \v {\mathbf{v}}

\def \p {\mathbf{p}}

\newcounter{protocoll}


\newtheorem{ass}{Assumption}

\newtheorem{cor}{Corollary}
\newtheorem{defn}{Definition}
\newtheorem{Remark}{Remark}

\DeclareMathOperator*{\argmax}{argmax}
\usepackage{microtype}

\usepackage{booktabs}

\title[Faster Rates For  Federated Variational Inequalities]{Faster Rates For  Federated Variational Inequalities}
\usepackage{times}

\author{%
 \Name{Guanghui Wang}\thanks{This work was done during an internship at Apple.} \Email{gwang369@gatech.edu}\\
 \addr College of Computing, Georgia Institute of Technology
 \AND
 \Name{Satyen Kale} \Email{satyen@apple.com}\\
 \addr Apple
}

\begin{document}

\maketitle

\begin{abstract}%
\noindent In this paper, we study federated optimization for solving stochastic variational inequalities (VIs), a problem that has attracted growing attention in recent years. Despite substantial progress, a significant gap remains between existing convergence rates and the state-of-the-art bounds known for federated convex optimization. In this work, we address this limitation by establishing a series of improved convergence rates. First, we show that, for general smooth and monotone variational inequalities, the classical Local Extra SGD algorithm admits tighter guarantees under a refined analysis. Next, we identify an inherent limitation of Local Extra SGD, which can lead to excessive client drift. Motivated by this observation, we propose a new algorithm, the Local Inexact Proximal Point Algorithm with Extra Step (\lippax), and show that it mitigates client drift and achieves improved guarantees in several regimes, including bounded Hessian, bounded operator, and low-variance settings. Finally, we extend our results to federated composite variational inequalities and establish improved convergence guarantees.

\end{abstract}

\section{Introduction}
Federated learning \citep{konevcny2016federated,pmlr-v54-mcmahan17a} is a powerful framework for solving large-scale machine learning problems. Unlike traditional approaches \citep{NIPS2012_6aca9700,goyal2017accurate} that aggregate all training data on a central server, federated optimization keeps the data on local devices, such as mobile phones or organizational silos, and performs training locally. Periodically, these devices communicate only model updates (e.g., gradients or parameters), which are then aggregated to update a global model. This decentralized design  enhances privacy and reduces communication costs, making federated learning particularly attractive for modern applications where data are  sensitive  and massive in scale. 

The central problem of interest is federated
empirical risk minimization. The canonical algorithm for this task is \emph{Local Stochastic Gradient Descent} (LSGD), also known as \emph{Federated Averaging} \citep{zinkevich2010parallelized,pmlr-v54-mcmahan17a,stich2018local}: In this scheme, each of the $M$ participating devices performs multiple SGD steps locally with its own data, following by synchronizing its model with the central server every $K$ local steps. By leveraging these \emph{local updates}, LSGD  achieves faster convergence than its centralized counterpart, mini-batch SGD, particularly in settings with many participating machines \citep{pmlr-v119-woodworth20a}. Formally, for optimizing smooth and convex loss functions, with $R$ communication rounds, LSGD attains a convergence rate of $O\left(\frac{1}{KR} + \frac{\sigma}{\sqrt{MKR}} + \frac{\sigma^{2/3}}{K^{1/3}R^{2/3}}\right)$, where $\sigma^2$ is the variance of the stochastic gradients. Building on this foundational work, numerous extensions have been proposed, from advanced strategies to mitigate data heterogeneity across clients \citep{zhao2018federated,pmlr-v119-karimireddy20a} to the integration of adaptive optimizers \citep{xie2019local,reddi2021adaptive}.

\begin{table}[]
\centering
\caption{Summary of convergence rates for federated optimization and variational inequalities. FO: Federated Optimization. FVI: Federated VI. The red terms denote sub-optimal factors compared to the LSGD-optimal bound.}
\label{tab:summary}
\resizebox{0.9\textwidth}{!}{%
\begin{tabular}{c|c|c|c|c}
\Xhline{1.2pt}
\textbf{Setting} & \textbf{Algo.} & \textbf{Reference} & \textbf{Bound} & \textbf{Assumption} \\
\Xhline{1.2pt}
FO & LSGD & \citep{pmlr-v119-woodworth20a} &  
$\textstyle O\left( \frac{1}{{K}R}+\frac{\sigma}{\sqrt{MKR}}+\frac{\sigma^{\frac{2}{3}}}{K^{\frac{1}{3}}R^{\frac{2}{3}}}\right)$ & -- \\ \hline   
\multirow{9}{*}{FVI} 
& \multirow{4}{*}{LESGD} 
    & \citep{beznosikov2022decentralized} 
    & $\textstyle O\!\left(\frac{1}{{\color{red}{\sqrt{K}}}R}+\frac{\sigma}{\sqrt{MKR}}+{\color{red}{\sqrt{\tfrac{\sigma}{\sqrt{K}R}+\tfrac{\sigma^2}{R}}}}\right)$ 
    & -- \\ \cline{3-5} 
&   & Theorem \ref{thm:ESGD}  
    & $\textstyle O\left( \frac{1}{{\color{red}{\sqrt{K}}}R}+\frac{\sigma}{\sqrt{MKR}}+\frac{\sigma^{\frac{2}{3}}}{K^{\frac{1}{3}}R^{\frac{2}{3}}}\right)$ 
    & -- \\ \cline{3-5} 
&   & Corollary \ref{coredsada}  
    & \multirow{2}{*}{$\textstyle O\left( \frac{1}{{K}R}+\frac{\sigma}{\sqrt{MKR}}+\frac{\sigma^{\frac{2}{3}}}{K^{\frac{1}{3}}R^{\frac{2}{3}}}\right)$}
    & $\sigma=\Omega\!\left(\frac{1}{\sqrt{R}K^{1/4}}\right)$ \\ \cline{3-3}\cline{5-5} 
&   & Theorem \ref{thm:affine} 
    &  
    & affine operator \\ \cline{2-5} 
& \multirow{3}{*}{\lippax} 
    & Theorem \ref{thm:inexactppa} 
    & $\textstyle \widetilde{O}\left( \frac{1}{{K}R}+\frac{\sigma}{\sqrt{MKR}}+\frac{\sigma^{\frac{2}{3}}}{K^{\frac{1}{3}}R^{\frac{2}{3}}}+{\color{red}{\frac{\sigma}{\sqrt{KR}}}}\right)$
    & Bounded operator \\ \cline{3-5} 
&   & Corollary \ref{corsdasdsarccccd}  
    & \multirow{2}{*}{$\textstyle \widetilde{O}\left( \frac{1}{{K}R}+\frac{\sigma}{\sqrt{MKR}}+\frac{\sigma^{\frac{2}{3}}}{K^{\frac{1}{3}}R^{\frac{2}{3}}}\right)$}
    & $\sigma=O\!\left(\frac{1}{\sqrt{KR}}\right)$ \\ \cline{3-3}\cline{5-5} 
&   & Theorem \ref{thm:optimalboundppass} 
    & 
    & Bounded Hessian \\ \cline{2-5} 
& {\color{green!50!black}\lippaxg}
    & {\color{green!50!black}Theorem \ref{them:gaussiansmoothing}} 
    & {\color{green!50!black}$\textstyle \widetilde{O}\left( \frac{1}{{K}R}+\frac{\sigma}{\sqrt{MKR}}+\frac{\sigma^{\frac{2}{3}}}{K^{\frac{1}{3}}R^{\frac{2}{3}}}\right)$}
    & {\color{green!50!black}Bounded operator} \\ \cline{2-5} 
& LSGD 
    & Theorem \ref{thm:::cocococococ} 
    & $\textstyle O\left( \frac{1}{{K}R}+\frac{\sigma}{\sqrt{MKR}}+\frac{\sigma^{\frac{2}{3}}}{K^{\frac{1}{3}}R^{\frac{2}{3}}}\right)$
    & Co-coercive \\ \hline 
FCO & LDA & \citep{yuan2021federated} & 
$\textstyle O\left(\frac{1}{KR}+\frac{1}{R^{\frac{2}{3}}}+\frac{\sigma}{\sqrt{MKR}}\right)$ & Bounded gradient \\ \hline
\multirow{2}{*}{FCVI} 
& \multirow{2}{*}{LDE} 
    & \citep{bai2024local} & 
    $\textstyle O\left(\frac{1}{KR} + {\color{red}{ \frac{1}{\sqrt{R}}}} + \frac{\sigma}{\sqrt{MKR}} + {\color{red}{ \frac{1}{K^{\frac{1}{4}}R^{\frac{3}{4}}}}} \right)$
    & Bounded operator \\ \cline{3-5} 
&   & Theorem \ref{thm555} & 
$\textstyle O\left(\frac{1}{KR}+\frac{1}{R^{\frac{2}{3}}}+\frac{\sigma}{\sqrt{MKR}}\right)$
& Bounded operator  \\ 
\Xhline{1.2pt}
\end{tabular}%
}
\end{table}

To date, the theory of federated optimization is primarily restricted to smooth and convex loss functions. A natural question is whether these results can be extended to more general problems. In this work, we investigate federated optimization for \emph{smooth and monotone variational inequalities} \citep{stampacchia1970variational,kinderlehrer2000introduction,juditsky2011solving}, a unifying framework that not only recovers convex and smooth optimization as a special case, but also encompasses a broader range of complex optimization problems, such as convex-concave  optimization, equilibrium computation, and fixed point computation. 
A particular application of interest is smooth convex-concave optimization. Such problems have arise in many modern machine learning applications, including the training of generative adversarial networks \citep{NIPS2014_5423}, robust reinforcement learning \citep{pmlr-v119-jin20f}, and multi-distribution machine learning  \citep{NEURIPS2022_02917ace}. 


The study of federated variational inequalities has gained increasing attention in recent years \citep{deng2021local,hou2021efficient,pmlr-v162-sharma22c,beznosikov2025distributed}. Most of these efforts has focused on LSGD-type methods, the natural extension of LSGD to the VI setting. However, it is well known that vanilla SGD may diverge on general smooth monotone problems. As a result, these works have mostly been restricted to the narrower case of smooth and \emph{strongly} monotone VIs.  To deal with general smooth and monotone VIs,  researchers have adapted the Extra SGD algorithm \citep{juditsky2011solving}, the classical method for general smooth monotone VIs, to the federated context, leading to the LESGD method \citep{beznosikov2022decentralized}. The resulting convergence rate is  
$
\textstyle O\!\left(\frac{1}{\sqrt{K}R}+\frac{\sigma}{\sqrt{MKR}}+{\sqrt{\tfrac{\sigma}{\sqrt{K}R}+\tfrac{\sigma^2}{R}}}\right).
$  
However, there is still a large gap between this bound and the $
O\left(\tfrac{1}{KR} + \tfrac{\sigma}{\sqrt{MKR}} + \tfrac{\sigma^{2/3}}{K^{1/3}R^{2/3}}\right)
$ rate of LSGD for federated convex optimization, especially due to the last term which is of order $O\left(\frac{\sigma}{\sqrt{R}}\right)$. This last term does not decay with the number of local steps $K$ or number of machines $M$, which brings into question the utility of local steps.

In this paper, we revisit the foundational problem of federated optimization for \emph{general smooth and monotone variational inequalities}. We note that solving federated variational inequalities is inherently challenging: Prior theoretical analyses in federated convex optimization often rely on the smoothness of the objective, which does \emph{not} extend to min–max problems or general variational inequalities, since min–max objectives are typically non-smooth when viewed as functions of the primal variable (even when the objective is jointly smooth), due to the inner maximization. More essentially, convex and smooth objectives naturally satisfy \emph{co-coercivity}, a property that is crucial in most analyses of federated optimization. However, co-coercivity generally fails to hold for smooth and monotone variational inequalities. As a result, it remains unclear whether improved convergence rates are achievable in this setting. In this paper, we provide a positive answer to this question by establishing a series of faster convergence rates. More specifically:

\begin{itemize}
    \item \textbf{LESGD}: For LESGD, We show that a faster 
    $\textstyle O\left( \frac{1}{\sqrt{K}R}+\frac{\sigma}{\sqrt{MKR}}+\frac{\sigma^{\frac{2}{3}}}{K^{\frac{1}{3}}R^{\frac{2}{3}}}\right)$ 
rate can be obtained (Theorem \ref{thm:ESGD}),  which is LSGD-optimal\footnote{By an LSGD-optimal bound, we refer to a rate of order
$
O\!\left( \frac{1}{KR}+\frac{\sigma}{\sqrt{MKR}}+\frac{\sigma^{2/3}}{K^{1/3}R^{2/3}} \right),$
which is known to be optimal for LSGD in federated convex and smooth optimization \citep{pmlr-v119-woodworth20a}. Note that federated variational inequalities include federated optimization as a special case.} 
up to an additional ${\sqrt{K}}$ factor in the first term. As a consequence, the LSGD-optimal  rate is achieved when $\sigma=\Omega\left(\frac{1}{\sqrt{R}K^{\frac{1}{4}}}\right)$, i.e., the variance is not too small  (Corollary \ref{coredsada}). We also show that the LSGD-optimal rate can be achieved if the operator is co-coercive (Theorem \ref{thm:::cocococococ}). 
\item \textbf{Composite VI}: We extend our analysis to \emph{composite} VI problems, where the smooth and monotone VI is regularized by a potentially non-smooth function. For this harder problem, \cite{bai2024local} provide the following bound for a generalized version of LESGD:
    $\textstyle O\left(\frac{1}{KR} + { \frac{{1}}{\sqrt{R}}} + \frac{\sigma}{\sqrt{MKR}} +  \frac{\sqrt{1}}{K^{\frac{1}{4}}R^{\frac{3}{4}}} \right).$
In Theorem \ref{thm555}, we improve this to 
$\textstyle O\left(\frac{\sigma}{\sqrt{MKR}}+ \frac{1}{R^{\frac{2}{3}}} + \frac{1}{KR}\right),$
which exactly matches the best known rate for federated composite convex optimization \citep{yuan2021federated}, which is a special case of federated composite VI problems. 
\end{itemize}
Our main technical novelty for obtaining the results above lies in a novel analysis of the convergence bound. We note that in the analysis of federated learning algorithms, an important term is the client drift, which measures the divergence between local models across different clients. Traditional analysis of federated optimization usually establishes a connection between the convergence upper bound and the squared norm of the client drift. However, in federated VI analysis,  previous work establishes connections to both the norm and its square, due to the more complicated nature of VIs compared with convex optimization. In this paper, we show that only the squared client drift matters, which leads to improved convergence rates.

Note that our general bound for LESGD (Theorem~\ref{thm:ESGD}) contains a $\tfrac{1}{\sqrt{K}R}$ term, which is a factor of $\sqrt{K}$ worse than the $\frac{1}{KR}$ term in the LSGD-optimal bound.  This term is also present in the results of prior work~\citep{beznosikov2022decentralized}, although it is difficult to discern from their paper. We find that this term stems from the lack of co-coercivity (and, in the worst case, even anti-cocoercivity) of the extra-gradient operation~\citep{gorbunov2022extragradient}, which leads to looser control of client drift across local updates and reveals a fundamental limitation of this class of methods. We refer to Remark~\ref{remakree222} for a detailed discussion. In Theorem~\ref{thm:affine}, we further show that LSGD-optimal rates can be achieved when the operator is affine, since the extra-gradient operation becomes co-coercive in this case. However, co-coercivity of the extra-gradient operation does not hold in general. This shortcoming is addressed in our next contribution:

\begin{itemize}
    \item \textbf{New Algorithm, \lippax}: We propose a new algorithm called the Local Inexact Proximal Point Algorithm with Extra Step (\lippax, Algorithm~\ref{alg:example}). Our key observation is that the proximal point operation incurs far less client drift than the extra-gradient step. While the exact proximal point is impractical to compute, we approximate it using multiple SGD steps on a regularized objective, followed by one extra gradient step from the starting point. This approach successfully eliminates the $\tfrac{1}{\sqrt{K}R}$ term, but introduces an additional $\tfrac{\sigma}{\sqrt{KR}}$ variance term (Theorem~\ref{thm:inexactppa}). This extra term arises from the variance accumulated during the inner SGD loop, and a direct analysis suggests it is irreducible by $M$. Nevertheless, through a novel bias-variance decomposition, we show that this term can be eliminated when the Hessian of the operator is bounded (Theorem~\ref{thm:optimalboundppass}). Finally, we propose a Gaussian-smoothing variant called \lippaxg\ (Algorithm~\ref{alg:example:lipeG}), 
    which removes the bounded-Hessian assumption, and achieves LSGD-optimal bound (up to logarithmic factors) under only the mild assumption that the operator is bounded  (Theorem~\ref{them:gaussiansmoothing}).
\end{itemize}

This paper primarily focuses on the homogeneous setting of federated learning where all machines sample from the same distribution. In Appendix~\ref{sec:heterogeneous}, we also give an extension of our main result to the heterogeneous case; further extensions are left to future work.

\section{Related Work}
In this section, we briefly review related work in federated optimization and VI problems.
 
\paragraph{Variational Inequalities.} 
The study of VIs dates back to the 1960s \citep{minty1962monotone,stampacchia1964formes,stampacchia1970variational}, with more modern classical monographs in \citep{kinderlehrer2000introduction,tremolieres2011numerical}. Two of the most important algorithms for smooth and monotone VIs are the Extra Gradient (EG) method \citep{korpelevich1976extragradient} and the Optimistic Gradient (OG) method \citep{popov1980modification}. For stochastic smooth and monotone VIs, \cite{juditsky2011solving} proposed the Mirror-Prox algorithm (we refer to it as Extra SGD in the $\ell_2$-space), which can be viewed as the stochastic and mirror descent variant of EG. They show that for stochastic smooth and monotone VIs, Mirror-Prox achieves an $O\!\left(\tfrac{L}{T}+\tfrac{\sigma}{\sqrt{T}}\right)$ convergence rate, where $L$ is the smoothness parameter and $T$ is the total number of iterations. This matches the convergence rate of SGD for stochastic convex and smooth optimization \citep{bubeck2015convex}. In distributed optimization, a mini-batch version of Extra SGD (\emph{without local updates}) has been studied for saddle-point problems \citep{beznosikov2025distributed}, achieving an $O\!\left(\tfrac{L}{R}+\tfrac{\sigma}{\sqrt{MRK}}\right)$ rate, where $R$ is the number of communication rounds, $M$ is the number of machines, and $K$ is the batch size. We refer readers to the recent survey \citep{beznosikov2023smooth} for a detailed review of stochastic smooth and monotone inequalities.

\paragraph{Federated Optimization.} 
Federated optimization has been extensively studied in recent years \citep{wang2021field}. The most related foundational algorithm to our work is by \citet{pmlr-v54-mcmahan17a}, whose Federated Averaging algorithm coincides with the parallel SGD algorithm with identical clients \citep{zinkevich2010parallelized}. The LSGD algorithm has been analyzed by a long line of search \citep{stich2018local,stich2019error,pmlr-v108-bayoumi20a,pmlr-v119-woodworth20a}, with various sharp upper and lower bounds. \cite{yuan2021federated} studied federated optimization for composite problems, where the objective consists of a convex and smooth function plus a convex but non-smooth regularizer. They proposed the local dual averaging algorithm, which achieves an $O\!\left(\tfrac{1}{KR}+\tfrac{L}{R^{2/3}}+\tfrac{\sigma}{\sqrt{MRK}}\right)$ rate under the  bounded gradients assumption.

\paragraph{Federated Variational Inequalities.} 
For federated VIs, as mentioned in the introduction, there is a growing body of work \citep{deng2021local,hou2021efficient,pmlr-v162-sharma22c}, but these studies focus only on smooth and \emph{strongly} monotone VIs and employ SGDA-type methods, which are known to diverge in the merely monotone setting. More recent work \citep{beznosikov2022decentralized,bai2024local} has considered the LESGD algorithm. However, as discussed above, their results still exhibit a gap compared to the LSGD-optimal bounds.

\section{Preliminaries}
In this section, we review some basic concepts. Let $V:\R^d\mapsto\R^d$ be an operator of interest. We first introduce the following definitions. 
\begin{defn}[Monotonicity] 
An operator $V:\R^d\mapsto\R^d$  is monotone iff $\forall \z,\z'\in \R^d$, we have: 
\begin{equation}
    \label{eqn:monotone}
\left\langle V(\z)-V(\z'), \z-\z' \right\rangle\geq 0.
\end{equation}
\end{defn}
\begin{defn}[$L$-Smoothness]
  An operator $V:\R^d\mapsto\R^d$ is $L$-smooth, iff   $\forall \z,\z'\in\R^d$: 
\begin{equation}
\label{eqn:smoothVI}
   \|V(\z)-V(\z')\|\leq L\|\z-\z'\|. 
\end{equation}
\end{defn}
In this work, we are interested in solving variational inequalities. That is, we would like to find a solution $\z^*\in \R^d$,  such that: 
\begin{equation}
\label{eqn:VIeqn}
   \sup_{\z\in\R^d}\left\langle V(\z),\z^*-\z\right\rangle \leq 0. 
\end{equation}
The error of a candidate solution  $\overline{\z}\in\R^d$
is measured by:
$$ \text{err}(\overline{\z})=\sup_{\z}\left\langle V({\z}),\overline{\z}-\z\right\rangle. $$
Finally, we assume that we have access to an unbiased noisy oracle $\widetilde{V}: \mathbb{R}^d \times \Xi \mapsto \mathbb{R}^d$, where $\Xi$ is some sample space with an associated distribution that samples $\xi \in \Xi$ are drawn from, such that for any $\z \in \mathbb{R}^d$, we have $\E_{\xi}[\widetilde{V}(\z,\xi)]=V(\z)$. Moreover, we assume the variance is upper bounded:
$$ \E_{\xi}\|V(\z)-\widetilde{V}(\z,\xi)\|^2\leq \sigma^2.  $$
For simplicity and readability, we drop $\xi$ and use $\widetilde{V}(\z)$ when the context is clear. 

\section{Faster Rates For Federated VIs}
In this section, we provide a series of faster convergence rates for federated variational inequalities. In Section~\ref{sec:localsddd}, we establish tighter bounds for the LESGD algorithm. In Section~\ref{sec:fasterLother methods}, we introduce our proposed \lippax\ and \lippaxg\ algorithms and demonstrate that they achieve improved rates in different regimes. Finally, in Section~\ref{section:4.3}, we prove that LSGD attains faster rates for co-coercive operators. 

\subsection{Faster Rates for LESGD}
\label{sec:localsddd}
In this part, we study the classical LESGD algorithm for solving smooth and monotone variational inequalities, which has been examined in prior work \citep{beznosikov2022decentralized,bai2024local}, and we establish faster convergence rates. The algorithm is outlined below. Let $\eta > 0$ denote the step size. At each round $t = 1, \dots, T$, for each local client $m = 1, \dots, M$, LESGD performs  the following updates:
\begin{equation}
\label{eqn:algorithm 1}
\begin{split}
        \x_{t}^{m} = & 
    \begin{cases}
\z_{t-1}^m - \eta\widetilde{V}(\z^m_{t-1}), & \text{mod}(t,K)\not= 0;\\
\frac{1}{M}\sum_{m=1}^M\left(\z_{t-1}^m - \eta\widetilde{V}(\z^m_{t-1}) \right), & \text{mod}(t,K)= 0.
    \end{cases}\\
    \z^m_{t} = & 
        \begin{cases}
\z_{t-1}^m - \eta\widetilde{V}(\x^m_t), & \text{mod}(t,K)\not= 0;\\
\frac{1}{M}\sum_{m=1}^M\left(\z_{t-1}^m - \eta\widetilde{V}(\x^m_{t}) \right), & \text{mod}(t,K)= 0.
    \end{cases}\\
 \end{split}   
\end{equation}

Intuitively, each client $m = 1, \dots, M$ first executes a stochastic gradient descent step at the current point $\z_{t-1}^m$ to obtain $\x_t^m$, and then updates $\z_{t-1}^m$ to $\z_t^m$ using the gradient at $\x_t^m$. The clients synchronize with each other every $K$ steps. We then obtain the following result. The proof is given in Appendix \ref{sec:app:proof the1}.

\begin{theorem}
\label{thm:ESGD}
Suppose $V$ is $L$-smooth and monotone. Let $\Z_D=\{\z\in\R^d|\|\z-\frac{1}{M}\sum_{m=1}^M\z_{0}^m\|\leq D\}$, where $D>0$ is any constant picked by the user. Let $\x_o=\frac{1}{MT}\sum_{m=1}^M\sum_{t=1}^T\x_t^m$ be the output of the algorithm in \eqref{eqn:algorithm 1}. Set: 
\begin{equation}
\label{eqn:eqtasdsVlocalecsgsd}
    \eta=\min\left\{\frac{1}{\sqrt{14K}L}, \frac{D\sqrt{M}}{\sigma\sqrt{6KR}}, \frac{D^{\frac{2}{3}}}{936^{\frac{1}{3}}e^{\frac{2}{3}}K^{\frac{2}{3}}R^{\frac{1}{3}}\sigma^{\frac{2}{3}}L^{\frac{1}{3}}} \right\}.
\end{equation}
Then  we have  
\begin{equation}
    \begin{split}
\label{eqn:thm:11sew111100000}
& \E\left[\sup_{\z\in\Z_{D}}\left\langle V(\z),\x_o-\z \right\rangle \right] 
\leq \frac{12LD^2}{\sqrt{K}R}+ \frac{6D\sigma}{\sqrt{MKR}} + \frac{26 D^{\frac{4}{3}}\sigma^{\frac{2}{3}}L^{\frac{1}{3}}}{K^{\frac{1}{3}}R^{\frac{2}{3}}}.
\end{split}
\end{equation}
\end{theorem}
Theorem \ref{thm:ESGD} shows an $\textstyle O\left( \frac{L}{\sqrt{K}R}+\frac{\sigma}{\sqrt{MKR}}+\frac{\sigma^{\frac{2}{3}}L^{\frac{1}{3}}}{K^{\frac{1}{3}}R^{\frac{2}{3}}}\right)$  convergence for LESGD,
which nearly matches the LSGD-optimal bound in the classical federated optimization setting 
up to an additional ${\sqrt{K}}$ factor in the first term. Compared with the bound for LESGD established in prior work \citep{beznosikov2022decentralized}, 
$O\left(\frac{L}{\sqrt{K}R}+\frac{\sigma}{\sqrt{MKR}}+{\sqrt{\frac{L\sigma}{\sqrt{K}R}+\frac{\sigma^2}{R}}}\right),$
our result eliminates the undesirable last term.

\begin{Remark}[Discussion on the $\frac{L}{\sqrt{K}R}$ term]
\label{remakree222}
\emph{As discussed above, compared with the bound for LSGD \citep{pmlr-v119-woodworth20a}, the bound for LESGD contains an additional $\sqrt{K}$ factor in the first term. This arises because, as shown in \eqref{eqn:eqtasdsVlocalecsgsd}, LESGD requires a smaller step size of order $\eta=O(\tfrac{1}{\sqrt{K}L})$, whereas classical LSGD uses $\eta=O(\tfrac{1}{L})$. The smaller step size is necessary due to an inherent limitation of the extra-gradient method.} 
\emph{In particular, a key quantity in the analysis of federated optimization is the \emph{client drift}, measured by $\|\z_t^{m}-\z_t^{m'}\|^2$ for any two clients $m,m'\in[M]$. Under LSGD, this drift naturally contracts thanks to the \emph{co-coercivity} of the SGD operator. However, as recently shown by \citet{gorbunov2022extragradient}, the extra-gradient operator can be non-co-coercive, which makes $\|\z_t^{m}-\z_t^{m'}\|^2$ expansive instead of contractive. In fact, \citet{gorbunov2022extragradient} show that there is a particular \emph{co-coercive} operator $V$ and points $\z,\z'\in\R^d$ such that
$$ \| \z-V(\z-V(\z)) -(\z'-V(\z'-V(\z')))\|^2\geq (1+\eta^4)\|\z-\z'\|^2.$$
To control this expansion, a smaller step size is therefore required.}
\end{Remark}

Next, we show that under suitable assumptions, the convergence rate can be further improved to the optimal rate, that is, the rate for federated SGD for the federated optimization setting. 
\begin{cor}[Large Variance]
\label{coredsada}
  Suppose $V$ is $L$-smooth and monotone. Assume 
  $\sigma \geq \frac{14^{\tfrac{3}{4}}\,D\,L}{6e\sqrt{26}\,\sqrt{R}\,K^{1/4}} =\Omega\left( \frac{D L}{R^{\frac{1}{2}}K^{\frac{1}{4}}}\right).$ Let $\eta=\min\left\{ \frac{D\sqrt{M}}{\sigma\sqrt{6KR}}, \frac{D^{\frac{2}{3}}}{936^{\frac{1}{3}}e^{\frac{2}{3}}K^{\frac{2}{3}}R^{\frac{1}{3}}\sigma^{\frac{2}{3}}L^{\frac{1}{3}}} \right\}$. Then  we have  
\begin{equation*}
    \begin{split}
& \E\left[\sup_{\z\in\Z_{D}}\left\langle V(\z),\x_o-\z \right\rangle \right] 
\leq  \frac{26 D^{\frac{4}{3}}\sigma^{\frac{2}{3}}L^{\frac{1}{3}}}{K^{\frac{1}{3}}R^{\frac{2}{3}}} + \frac{6D\sigma}{\sqrt{MKR}}.
\end{split}
\end{equation*}
\end{cor}
The corollary above shows that the optimal rate can be achieved when $\sigma$ is not too small. This naturally occurs, as the undesirable small step size $\eta=O(\frac{1}{\sqrt{K}L})$ (as discussed in Remark \ref{remakree222}) is dominated by the rest of the terms in the configuration of the step size in \eqref{eqn:eqtasdsVlocalecsgsd}, which leads to the LSGD-optimal bound. 

Next, we show that, the optimal rate can be also obtained when $V$ is affine. The proof is provided in Appendix \ref{sec:app:prof:affine}. 
\begin{theorem}[Affine Operator]
   \label{thm:affine}   
   Suppose $V$ is affine, $L$-smooth and monotone. Then consider the algorithm given in \eqref{eqn:algorithm 1} with  $\eta=\min\left\{\frac{1}{L}, \frac{D\sqrt{M}}{\sigma\sqrt{6KR}}, \frac{D^{\frac{2}{3}}}{180^{\frac{1}{3}}K^{\frac{2}{3}}R^{\frac{1}{3}}\sigma^{\frac{2}{3}}L^{\frac{1}{3}}} \right\},$ and we get 
\begin{equation}
    \begin{split}
    \label{eqn:thm:11sew111}
& \E\left[\sup_{\z\in\Z_{D}}\left\langle V(\z),\x_o-\z \right\rangle \right] 
\leq  \frac{20D^{\frac{4}{3}}\sigma^{\frac{2}{3}}L^{\frac{1}{3}}}{K^{\frac{1}{3}}R^{\frac{2}{3}}} + \frac{2\sqrt{6}D\sigma}{\sqrt{MKR}} + \frac{2LD^2}{{K}R}. 
\end{split}
\end{equation}
\end{theorem}
Comparing Theorem \ref{thm:affine} with Theorem \ref{thm:ESGD}, the key difference lies in the step size: $\eta$ can be larger ($O\left(\tfrac{1}{\sqrt{K}L}\right)$ vs.\ $O\!\left(\tfrac{1}{L}\right)$), which enables the LSGD-optimal bound. This improvement is due to the non-expansive property of the extra-gradient operator for affine VIs \citep{gorbunov2022extragradient}, which leads to a tighter bound on the client drift term.

\subsection{Faster Rates By Inexact PPA With Extra Step}
\label{sec:fasterLother methods}
\begin{algorithm}[t!]
  \caption{Local Inexact Proximal Point with Extra Gradient (\lippax)}
  \label{alg:example}
  \begin{algorithmic}[1]     
    \STATE Initialization: $\z_{0}^{m}=\mathbf{0}$, for $m=1,\dots,M$.
    \FOR{$t = 1$ to $T$}
    \FOR{each machine $m=1,\dots,M$}
            \STATE $\x_{t,0}^{m}=\z^{m}_{t-1}$
      \FOR{$\ell=1$ to $H$}
        \STATE $\x^m_{t,\ell}=\x^m_{t,\ell-1}-\gamma\left(\widetilde{V}(\x^m_{t,\ell-1})+\frac{1}{\eta}\left(\x^m_{t,\ell-1}-\z_{t-1}^{m}\right)\right)$
      \ENDFOR
      \STATE $\x^m_t=\x_{t,H}^m$
\STATE $\z^m_{t} =
\begin{cases}
\z_{t-1}^m - \eta\widetilde{V}(\x^m_t), & \text{if } \text{mod}(t,K)\not= 0;\\
\frac{1}{M}\sum_{m=1}^M\left(\z_{t-1}^m - \eta\widetilde{V}(\x^m_{t}) \right), & \text{if } \text{mod}(t,K)= 0.
\end{cases}$
      \ENDFOR
    \ENDFOR
    \RETURN $\x_o=\frac{1}{MT}\sum_{m=1}^M\sum_{t=1}^T\x_t^m$ 
  \end{algorithmic}
\end{algorithm}
Corollary \ref{coredsada} shows that the LESGD algorithm achieves the LSGD-optimal bound  when the variance is not too small. However, for the small variance regime, i.e.,  $\sigma=O\left(\frac{1}{R^{\frac{1}{2}}K^{\frac{1}{4}}}\right)$, the first term in \eqref{eqn:thm:11sew111100000} will dominate the rest, making the overall convergence bound become the sub-optimal rate $O\left(\frac{L}{\sqrt{K}R}\right)$, \emph{even when $\sigma=0$.} As discussed in Remark \ref{remakree222}, this is caused by the expansive property of the extra SGD operator, and can be a fundamental limitation of this kind of method. 

In this section, we provide an alternative approach, which achieves the LSGD-optimal bound up to logarithmic  factors under proper assumptions. The algorithm is given in Algorithm \ref{alg:example}, which we referred to as Local Inexact Proximal Point Algorithm with Extra Gradient (\lippax).  To motivate our algorithm, we first recall the classical proximal point algorithm, which can handle general smooth and monotone VIs and obtain optimal results. Suppose at round $t$, the decision of machine $m$ is $\z_{t-1}^m$. Then one step of the ``exact'' proximal point  algorithm conducts the following implicit update to obtain a decision $\x_t^{*,m}$:
$$ \x_t^{*,m}=\z_{t-1}^m - \eta V(\x_t^{*,m}).$$
However, this update is hard to implement, as both sides involves $\x_t^{*,m}$. To avoid this, notice that $\x_t^{*,m}$ is the optimizer of the variational inequality defined by the following
``regularized'' strongly monotone  operator: 
\begin{equation*}
     F(\x)=V(\x)+\frac{1}{\eta}(\x-\z_{t-1}^m).
\end{equation*}
Then, we can conduct multiple LSGD steps to approximate $\x_{t}^{*,m}$ (Steps 4-7), and then perform one step of inexact proximal point step using the approximation (Steps 8-9).  We have the following conclusion for Algorithm \ref{alg:example}.  The proof is given in Appendix 
\ref{secpfofinscaew333}.

\begin{theorem}
\label{thm:inexactppa}
Suppose $V$ is $L$-smooth and monotone. Moreover, assume $\|V(\x)\|\leq G$ for $\x\in\R^d$. Let 
$$ \eta=\min\left\{\frac{1}{L}, \frac{D\sqrt{M}}{\sigma\sqrt{KR}}, \frac{D^{\frac{2}{5}}}{(60e)^{\frac{1}{5}}K^{\frac{3}{5}}R^{\frac{1}{5}}\sigma^{\frac{2}{5}}L^{\frac{2}{5}}}, \frac{D^{\frac{2}{3}}}{(54e)^{\frac{1}{3}}K^{\frac{2}{3}}R^{\frac{1}{3}}L^{\frac{1}{3}}\sigma^{\frac{2}{3}}},\frac{D}{\sigma\sqrt{ 15KR}}\right\},$$
and let $\gamma=\frac{1}{\eta\left(L+\frac{1}{\eta}\right)^2}$. Then Algorithm \ref{alg:example} ensures that:
\begin{equation}
    \begin{split}
\label{eqn:boundofhraippa}
\E\left[\sup_{\z\in\Z_{D}}\left\langle V(\z),\x_o-\z \right\rangle \right] 
\leq &\frac{LD^2}{KR}+\frac{4{\sigma} D}{\sqrt{KR}}+\frac{2(60e)^{\frac{1}{5}}D^{\frac{8}{5}}\sigma^{\frac{2}{5}}L^{\frac{2}{5}}}{K^{\frac{2}{5}}R^{\frac{4}{5}}}+\frac{2(54e)^{\frac{1}{3}}D^{\frac{4}{3}}L^{\frac{1}{3}}\sigma^{\frac{2}{3}}}{K^{\frac{1}{3}}R^{\frac{2}{3}}}+ \frac{4D\sigma}{\sqrt{MKR}}\\
&+\frac{60eG^2K^2L+30 G^2}{4^H}.  
\end{split}
\end{equation}

\end{theorem}
To see the merit of this bound, firstly note that, the first term becomes $O\left(\frac{L}{KR}\right)$, instead of  $O\left(\frac{L}{\sqrt{K}R}\right)$. Moreover, the last term decreases exponeitnally with $H$, and will be dominated by other terms as long as $H=O(\ln KR)$. However, there is a price paid in the second term, which is on the order of $O\left(\frac{\sigma}{\sqrt{KR}}\right)$. We have the following corollary. 
\begin{cor}
\label{corsdasdsarccccd}
Assume 
$ \sigma \leq \frac{1}{\sqrt{KR}}=O\left(\frac{1}{\sqrt{{K}R}}\right), $ 
then 
\begin{equation}
    \begin{split}    \label{eqn:ssssdsaasfhauih}
\E\left[\sup_{\z\in\Z_{D}}\left\langle V(\z),\x_o-\z \right\rangle \right] 
\leq &\frac{LD^2+4D}{KR}+ \frac{4D\sigma}{\sqrt{MKR}}+\frac{2(54e)^{\frac{1}{3}}D^{\frac{4}{3}}L^{\frac{1}{3}}\sigma^{\frac{2}{3}}}{K^{\frac{1}{3}}R^{\frac{2}{3}}}\\
&+\frac{2(60e)^{\frac{1}{5}}D^{\frac{8}{5}}\sigma^{\frac{2}{5}}L^{\frac{2}{5}}}{K^{\frac{2}{5}}R^{\frac{4}{5}}}
+\frac{60eG^2K^2L+30 G^2}{4^H}  
\end{split}
\end{equation}
\end{cor}
Note that the first three terms in \eqref{eqn:ssssdsaasfhauih} match the optimal rate, whereas the remaining terms are of lower order and thus dominated by the leading ones. Consequently, the bound is optimal when $\sigma = O\left(\tfrac{1}{\sqrt{KR}}\right)$.

While Theorem \ref{thm:ESGD} indicates that LESGD is optimal when $\sigma=\Omega\left(\frac{1}{R^{\frac{1}{2}}K^{\frac{1}{4}}}\right)$, and Corollary \ref{corsdasdsarccccd}  implies that \lippax\ (Algorithm \ref{alg:example}) is optimal when $\sigma=O\left(\frac{1}{\sqrt{RK}}\right)$, it is still unclear what happens in the small gap between $\sigma=\Omega\left(\frac{1}{\sqrt{RK}}\right)$ and $\sigma=O\left(\frac{1}{\sqrt{R}K^{\frac{1}{4}}}\right)$. Moreover, due to the $O\left(\frac{\sigma}{\sqrt{KR}}\right)$ term in \eqref{eqn:boundofhraippa}, the upper bound in Theorem 
 \ref{thm:inexactppa} becomes sub-optimal when $\sigma$ is large, as this term is not accelerated by $M$. In the following, we show that the LSGD-optimal bound can be achieved under the minor assumption that the operator has bounded Hessian.  
\begin{ass}[Second-order Boundedness] 
\label{ass:boundedasddff}
Assume $V:\R^d \to \R^d$ satisfies that there exists $\Lambda>0$ such that for all $\x,\y\in\R^d$, 
$
\| V(\x) - V(\y) - J_V(\y)(\x-\y) \| \leq \Lambda \|\x-\y\|^2,
$
where $J_V(\y)$ denotes the Jacobian of $V$ at $\y$.  
\end{ass}
Note that, if we write $V(\x)=[V_1(\x),\dots,V_d(\x)]$ with $V_i:\R^d\mapsto\R$, then the above assumption is equivalent to requiring that the Hessian of each $V_i$ is bounded. We have the following conclusion based on Assumption \ref{ass:boundedasddff}. The proof is given in Appendix 
\ref{secpfofinscaew}.
\begin{theorem}
\label{thm:optimalboundppass}
Suppose $V$ is $L$-smooth and monotone, and $\|V(\x)\|\leq G$ for all $\x\in\R^d$. Furthermore, suppose Assumption \ref{ass:boundedasddff} hold.  Let 
$$ \eta=\min\left\{\frac{1}{L},\frac{D^{\frac{2}{5}}}{K^{\frac{3}{5}}R^{\frac{1}{5}}\sigma^{\frac{2}{5}}L^{\frac{2}{5}}},\frac{D^{\frac{2}{3}}}{K^{\frac{2}{3}}R^{\frac{1}{3}}\sigma^{\frac{2}{3}}L^{\frac{1}{3}}},\frac{D^{\frac{1}{2}}}{K^{\frac{1}{4}}R^{\frac{1}{4}}\sigma^{\frac{1}{2}}L^{\frac{1}{2}}},\frac{D^{\frac{1}{3}}}{\Lambda^{\frac{1}{3}}\sigma^{\frac{2}{3}}R^{\frac{1}{6}}K^{\frac{1}{6}}},\frac{D\sqrt{M}}{K^{\frac{1}{2}}R^{\frac{1}{2}}\sigma}\right\},$$
and let $\gamma=\frac{1}{\eta\left(L+\frac{1}{\eta}\right)^2}$. Then Algorithm \ref{alg:example} ensures that:
\begin{equation}
    \begin{split}
&\E\left[\sup_{\z\in\Z_{D}}\left\langle V(\z),\x_o-\z \right\rangle \right] 
\leq  240\left[\frac{D^2L}{KR} + \frac{D^{\frac{4}{3}}L^{\frac{1}{3}}\sigma^{\frac{2}{3}}}{K^{\frac{1}{3}}R^{\frac{2}{3}}} + \frac{D\sigma}{\sqrt{MKR}} \right]\\
&+240\left[ \frac{D^{\frac{8}{5}}\sigma^{\frac{2}{5}}L^{\frac{2}{5}}}{K^{\frac{2}{5}}R^{\frac{4}{5}}} +\frac{D^{\frac{3}{2}}\sigma^{\frac{1}{2}}L^{\frac{1}{2}}}{K^{\frac{3}{4}}R^{\frac{3}{4}}}+\frac{D^{\frac{5}{3}}d^{\frac{1}{6}}\Lambda^{\frac{1}{3}}\sigma^{\frac{2}{3}}}{R^{\frac{5}{6}}K^{\frac{5}{6}}}  \right] + \frac{10H^2\Lambda^2G^4+60G^2+120G^2K^2}{4^H}.
\end{split}
\end{equation}

\end{theorem}
Note that, the bound given in  Theorem \ref{thm:inexactppa} can be divided by three parts: 
The first part is exactly on the order of $O\left(\frac{1}{KR}+\frac{\sigma^{\frac{2}{3}}}{K^{\frac{1}{3}}R^{\frac{2}{3}}}+\frac{\sigma}{\sqrt{MKR}}\right)$, which matches the LSGD-optimal rates, while the second term contains only lower order terms that are dominated by the first term. The third term decreases with $H$ exponentially, so setting $H=O(\ln KR)$ making this term become lower order.
\begin{algorithm}[t]
  \caption{\lippax\ with Gaussian Smoothing (\lippaxg)}
  \label{alg:example:lipeG}
  \begin{algorithmic}[1]     
    \STATE Initialization: $\z_{0}^{m}=\mathbf{0}$, for $m=1,\dots,M$.
    \FOR{$t = 1$ to $T$}
    \FOR{each machine $m=1,\dots,M$}
            \STATE $\x_{t,0}^{m}=\z^{m}_{t-1}$
      \FOR{$\ell=1$ to $H$}
      \STATE Sample $\s_{t,\ell}^m\sim\mathcal{N}(\textbf{0},I_d)$, and $\xi_{t,\ell}^m$
        \STATE $\x^m_{t,\ell}=\x^m_{t,\ell-1}-\gamma\left({V}(\x^m_{t,\ell-1}+\delta \s_{t,\ell}^m;\xi_{t,\ell}^m)+\frac{1}{\eta}\left(\x^m_{t,\ell-1}-\z_{t-1}^{m}\right)\right)$
      \ENDFOR
      \STATE $\x^m_t=\x_{t,H}^m$
\STATE $\z^m_{t} =
\begin{cases}
\z_{t-1}^m - \eta\widetilde{V}(\x^m_t), & \text{if } \text{mod}(t,K)\not= 0;\\
\frac{1}{M}\sum_{m=1}^M\left(\z_{t-1}^m - \eta\widetilde{V}(\x^m_{t}) \right), & \text{if } \text{mod}(t,K)= 0.
\end{cases}$
      \ENDFOR
    \ENDFOR
    \RETURN $\x_o=\frac{1}{MT}\sum_{m=1}^M\sum_{t=1}^T\x_t^m$ 
  \end{algorithmic}
\end{algorithm}

\paragraph{Gaussian Smoothing} Theorem \ref{thm:optimalboundppass} assumes the second-order boundedness of the operator $V$. For more general $V$, we propose using \emph{Gaussian smoothing} in the proximal point phase of \lippax\ to smooth the operator. The details of the algorithm is given in Algorithm \ref{alg:example:lipeG}. Compared with Algorithm \ref{alg:example}, the main difference lies in Step 7: instead of directly querying $\widetilde{V}(\x_{t,\ell}^m)$, Algorithm \ref{alg:example:lipeG} queries $\widetilde{V}(\x_{t,\ell}^m+\delta \s_{t,\ell}^m)$, where $\s_{t,\ell}^m\sim \mathcal{N}(\text{0},I_d)$ is a Gaussian random variable. Let $\mathring{V}(\x)=\E[\widetilde{V}(\x+\delta \s_{t,\ell}^m)]$, then the approximate point step of Algorithm \ref{alg:example:lipeG} (Steps 5-8) is essentially solving the variational inequality defined by $\mathring{V}$, which is a smooth operator. On the other hand, the optimizer for this perturbed objective is  very close to $\x_{t}^{*,m}$. We have the following conclusion. The proof is given in Appendix \ref{secpfofinscaew}.

\begin{theorem}
\label{them:gaussiansmoothing}
Suppose $V$ is $L$-smooth and monotone, and $\|V(\x)\|\leq G$ for all $\x\in\R^d$. Let 
$\eta= \min\left\{ \frac{D^{\frac{1}{2}}}{K^{\frac{1}{4}}R^{\frac{1}{4}}L^{\frac{1}{2}}\sigma^{\frac{1}{2}}d^{\frac{1}{8}}},\frac{D^{\frac{2}{5}}}{\sigma^{\frac{2}{5}}L^{\frac{2}{5}}d^{\frac{1}{10}}K^{\frac{3}{5}}R^{\frac{1}{5}}},\frac{D^{\frac{2}{3}}}{K^{\frac{2}{3}}R^{\frac{1}{3}}\sigma^{\frac{2}{3}}L^{\frac{2}{3}}},\frac{1}{L}\right\}, $
let $\delta=\frac{\eta\sigma}{\sqrt{d}},$ and $\gamma=\frac{1}{\eta\left(L+\frac{1}{\eta}\right)^2}$. Then we have 
\begin{equation*}
    \begin{split}
 \E\left[\sup_{\z\in\Z_{D}}\left\langle V(\z),\x_o-\z \right\rangle \right]\leq & O\Bigg(\left[ \frac{L}{KR}+\frac{\sigma}{\sqrt{MKR}}+\frac{D^{\frac{4}{3}}\sigma^{\frac{2}{3}}L^{\frac{2}{3}}}{K^{\frac{1}{3}}R^{\frac{2}{3}}}\right]    +   
 \left[\frac{D^{\frac{3}{2}}L^{\frac{1}{2}}\sigma^{\frac{1}{2}}d^{\frac{1}{8}}}{K^{\frac{3}{4}}R^{\frac{3}{4}}} + \frac{D^{\frac{8}{5}}\sigma^{\frac{2}{5}}L^{\frac{2}{5}}d^{\frac{1}{10}}}{R^{\frac{4}{5}}K^{\frac{2}{5}}}\right]\\
 & +\frac{G^2+K^2L^2G^2+\frac{H^2L^2d^{\frac{1}{2}}G^4}{\sigma^2}}{4^H}\Bigg).
    \end{split}
\end{equation*}

\end{theorem}

\subsection{Faster Rates for Co-Coercive Operators}
\label{section:4.3}
In this section, we show that, when the operator is  co-coercive, LSGD is enough to achieve the LSGD-optimal bound. We first recall the definition of $\beta$-co-coercivity.  
\begin{defn}
An operator is $\beta$-co-coercive, if for any $\z,\z'\in\R^d$, 
$$ \|V(\z)-V(\z')\|^2\leq \beta \left\langle V(\z)-V(\z'),\z-\z' \right\rangle. $$
\end{defn}
For $\beta$-co-coercive operators, we consider the following classical LSGD algorithm:
\begin{equation}
\label{eqn:algorithm 133}
\begin{split}
        \x_{t}^{m} = & 
    \begin{cases}
\x_{t-1}^m - \eta\widetilde{V}(\x^m_{t-1}), & \text{mod}(t,K)\not= 0;\\
\frac{1}{M}\sum_{m=1}^M\left(\x_{t-1}^m - \eta\widetilde{V}(\x^m_{t-1}) \right), & \text{mod}(t,K)= 0,
    \end{cases}
 \end{split}   
\end{equation}
and we have the following convergence guarantees. The proof is provided in Appendix \ref{appendix:cocococococo}. 
\begin{theorem}
\label{thm:::cocococococ}
Suppose $V$ is $L$-smooth, $\beta$-co-coresive and monotone. Consider the algorithm given in \eqref{eqn:algorithm 133} with 
$\eta = \min\left\{ \frac{1}{\beta}, \frac{D\sqrt{M}}{\sqrt{KR}\sigma}, \frac{D^{\frac{2}{3}}}{2^{\frac{1}{2}}K^{\frac{2}{3}}R^{\frac{1}{3}}L^{\frac{1}{3}}\sigma^{\frac{2}{3}}}\right\}.$
Then we have 
\begin{equation}
    \begin{split}
\E\left[\text{\emph{err}}\left(\frac{1}{TM}\sum_{t=1}^T\sum_{m=1}^M\x^{m}_t\right)\right] 
\leq \frac{4\beta D^2}{KR}+\frac{4\sigma D}{\sqrt{MKR}} + \frac{8D^{\frac{4}{3}}L^{\frac{1}{3}}\sigma^{\frac{2}{3}}}{K^{\frac{1}{3}}R^{\frac{2}{3}}}.
\end{split}
\end{equation}

\end{theorem}

\section{Faster Rates For Composite Variational Inequalities}
In this section, we consider solving the following composite variational inequality problem \citep{nesterov2023high,bai2024local}, where we would like to find a $\z^*\in\R^d$ such that:
$$ \sup_{\z\in\R^d}\left\langle V(\z),\z^*-\z \right\rangle + \phi(\z^*)- \phi(\z)\leq 0, $$
where the function $\phi$ is convex but potentially non-smooth. For a candidate solution, the performance is measured by 
$$ \text{err}_c(\overline{\z})=\sup_{\z}\left\langle V({\z}),\overline{\z}-\z\right\rangle + \phi(\z^*)- \phi(\z). $$
This problem is a direct generalization of the smooth and monotone  variational inequality problem, and it is more challenging due to the non-smooth regularizer. One important application for this problem is the composite saddle-point problem:
$ \min_{\x}\max_{\y} f(\x,\y) + g_1(\x)-g_2(\y), $
where $f(\x,\y)$ is smooth  and convex-concave, while $g_1,g_2$ are convex but not necessary smooth. For these kinds of problems, \cite{bai2024local} proposed the  local dual averaging algorithm, which can be considered as a mirror-descent style  generalization of the LESGD algorithm. For this algorithm,  \cite{bai2024local}  show that an 
$\textstyle O\left(\frac{L}{KR} + { \frac{\sqrt{L}}{\sqrt{R}}} + \frac{\sigma}{\sqrt{MKR}} +  \frac{\sqrt{L}}{K^{\frac{1}{4}}R^{\frac{3}{4}}} \right)$ rate is achieved. By contrast, for composite optimization, \cite{yuan2021federated} obtained an $\textstyle O\left(\frac{\sigma}{\sqrt{MKR}}+ \frac{1}{R^{\frac{2}{3}}} + \frac{1}{KR}\right)$ bound. In this section, we show that Local dual averaging can also achieve this bound.  

To present the algorithm, we first define some notation. We refer to the paper by \cite{bai2024local} for more technical details about this algorithm. Let $h:\R^d\mapsto\R_+$ be a distance generating function, and without loss of generality, we assume it is $1$-strongly convex. Let $h_t=h+t\eta\phi$ be the general distance function. Let $h_t^*$ be the convex conjugate of $h_t$, and $\nabla h_t^*:\R^d\mapsto\R^d$ defines a mirror map from the dual space to the primal space. The algorithm (in the dual space) is given by: 
\begin{equation}
\label{eqn:algorithm 12}
\begin{split}
        \x_{t}^{m} = & 
    \begin{cases}
\z_{t-1}^m - \eta\widetilde{V}(\u^m_{t-1}), & \text{mod}(t,K)\not= 0;\\
\frac{1}{M}\sum_{m=1}^M\left(\z_{t-1}^m - \eta\widetilde{V}(\u^m_{t-1}) \right), & \text{mod}(t,K)= 0.
    \end{cases}\\
    \z^m_{t} = & 
        \begin{cases}
\z_{t-1}^m - \eta\widetilde{V}(\v^m_t), & \text{mod}(t,K)\not= 0;\\
\frac{1}{M}\sum_{m=1}^M\left(\z_{t-1}^m - \eta\widetilde{V}(\v^m_{t}) \right), & \text{mod}(t,K)= 0.
    \end{cases}\\
 \end{split}   
\end{equation}
Here, $\u^m_t=\nabla h_t^*(\z^m_t)$, and $\v^m_t=\nabla h_t^*(\x^m_t)$ are the primal variables, and  $\x_t^m$ and $\z_t^m$ are dual variables. In each round $t=1,\dots,T$, each machine $m=1,\dots,M$ still conducts extra gradient step in the dual space, but using the corresponding  operator value in the primal space. After $T$ iterations, the algorithm output the average of primal variables:
$\v_o =\frac{1}{TM}\sum_{t=1}^T\sum_{m=1}^M\v^m_t,$
and we have the following conclusion. The proof is postponed to Appendix \ref{sec:copositeprooffff}.
\begin{theorem}
\label{thm555}
Suppose $V$ is $L$-smooth, monotone, and $\|V(\x)\|\leq G$ for all $\x\in\R^d$. Then the algorithm given in \eqref{eqn:algorithm 12} with 
$\textstyle \eta = \left\{ \frac{D\sqrt{M}}{\sigma\sqrt{6KR}},\frac{D^{\frac{2}{3}}}{17^{\frac{1}{3}}K^{\frac{1}{3}}R^{\frac{1}{3}}L^{\frac{2}{3}}G^{\frac{2}{3}}},\frac{1}{\sqrt{10}L} \right\},$ guarantees that 
\begin{equation}
    \begin{split}
\E[\emph{\text{err}}_c(\v_o)]
     \leq  \frac{2\sqrt{6}D\sigma}{\sqrt{MKR}}+ \frac{17^{\frac{1}{3}}L^{\frac{2}{3}}D^{\frac{4}{3}}G^{\frac{2}{3}}}{R^{\frac{2}{3}}} + \frac{\sqrt{10}D^2L}{KR}. 
    \end{split}
\end{equation} 
\end{theorem}
Note that, this bound exactly matches the optimal results for federated composite optimization \citep{yuan2021federated}, and faster than the bound provided by \cite{bai2024local} as the term $O\left(\frac{1}{\sqrt{R}}\right)$ is improved to $O\left(\frac{1}{R^{\frac{2}{3}}}\right).$

\section{Conclusion}
In this paper, we study federated optimization for stochastic smooth and monotone VIs. We first show that the LESGD enjoys an 
$O\left( \tfrac{1}{\sqrt{K}R} + \tfrac{\sigma}{\sqrt{MKR}} + \tfrac{\sigma^{2/3}}{K^{1/3}R^{2/3}} \right)$ convergence rate. We then prove that the algorithm achieves the optimal rate when $\sigma = \Omega\!\left(\tfrac{1}{\sqrt{R}K^{1/4}}\right)$, or when the operator is affine. Next, we propose a new algorithm, called \lippax, which achieves the optimal rate when $\sigma = O\!\left(\tfrac{1}{\sqrt{KR}}\right)$, or when the Hessian of $V$ is bounded. Then, we propose a Gauss-smoothing variant of \lippax, which achieves an LSGD-optimal bound while only relying on $V$ being bounded. Finally, we present faster rates for co-coercive smooth and monotone VIs as well as for composite VI problems. 

Several open questions remain. The \lippaxg\ algorithm achieves the optimal rate only up to a logarithmic factor and relies on the assumption that $V$ is bounded. In future work, we plan to investigate whether these restrictions can be relaxed. Moreover, for federated composite VIs (and even federated composite optimization), existing results assume bounded gradients and include an $O\!\left(\tfrac{G}{R^{2/3}}\right)$ term in the best-known bounds, which is independent of the variance. This implies that the optimal $O\!\left(\tfrac{1}{KR}\right)$ rate cannot be achieved even when $\sigma = 0$, and it remains unclear whether this limitation can be overcome. Finally, this paper focuses primarily on the homogeneous setting where all local clients share the same data distribution, whereas previous work \citep{beznosikov2022decentralized} also considers the heterogeneous setting. Extending our improved rates to this setting remains an open problem; we refer to Appendix \ref{sec:heterogeneous} for a more detailed discussion.
 
\bibliography{colt}

\newpage 
\appendix

\section{Proof of Theorem \ref{thm:ESGD}}
\label{sec:app:proof the1}
We start by defining the following shadow updates: $\overline{\x}_t=\frac{1}{M}\sum_{m=1}^M\x_t^{m}$, and $\overline{\z}_t=\frac{1}{M}\sum_{m=1}^M\z_t^{m}$. We have 
\begin{equation}
\begin{split}
    \label{eqn:updateshwodwzz00}
  \overline{\x}_t= \frac{1}{M}\sum_{m=1}^M\x_t^{m} = \overline{\z}_{t-1}-\eta \frac{1}{M}\sum_{m=1}^M\widetilde{V}(\z_{t-1}^m)=\overline{\z}_{t-1}-\eta\widetilde{V}_{\z,t-1},  
\end{split}    
\end{equation}
and 
\begin{equation}
    \begin{split}
    \label{eqn:updateshwodwzz}
   \overline{\z}_{t} = \frac{1}{M}\sum_{m=1}^M \z^m_t = \overline{\z}_{t-1} -\eta \frac{1}{M}\sum_{m=1}^M\widetilde{V}(\x^m_{t})=\overline{\z}_{t-1}-\eta\widetilde{V}_{\x,t}.     
    \end{split}
\end{equation}
where we define $\widetilde{V}_{\z,t-1}=\frac{1}{M}\sum_{m=1}^M\widetilde{V}(\z_{t-1}^m)$, and $\widetilde{V}_{\x,t}=\frac{1}{M}\sum_{m=1}^M\widetilde{V}(\x^m_{t})$. Moreover, let ${V}_{\z,t-1}=\frac{1}{M}\sum_{m=1}^M{V}(\z_{t-1}^m)$, and ${V}_{\x,t}=\frac{1}{M}\sum_{m=1}^M{V}(\x^m_{t})$.
We first introduce the following basic lemma. 
\begin{lemma}
\label{lem:poential}
Let $\z^+=\z-\eta\v$, where $\z,\v\in\R^d$ and $\eta\in\R$. Then, for any $\z'\in\R^d$, we have:
$$ \| \z^+-\z'\|^2= \|\z-\z'\|^2-2\eta\left\langle \v, \z^+-\z'\right\rangle-\|\z^+-\z\|^2.$$
\end{lemma}
\begin{proof}
We have
\begin{equation*}
  \begin{split}
\| \z^+-\z'\|^2 = & \| \z-\eta\v-\z'\|^2= \|\z-\z'\|^2-2\eta \left\langle \v,\z-\z' \right\rangle + \| \z^+-\z\|^2\\    = & \|\z-\z'\|^2-2\eta \left\langle \v,\z^+-\z' \right\rangle + \| \z^+-\z\|^2+ 2\eta \left\langle \v,\z^+ -\z \right\rangle\\
= &  \|\z-\z'\|^2-2\eta \left\langle \v,\z^+-\z' \right\rangle + \| \z^+-\z\|^2-2\ \|\z^+-\z\|^2\\
= &  \|\z-\z'\|^2-2\eta \left\langle \v,\z^+-\z' \right\rangle - \| \z^+-\z\|^2.
  \end{split}  
\end{equation*}
\end{proof}

\noindent Apply this lemma for the update in \eqref{eqn:updateshwodwzz00} and \eqref{eqn:updateshwodwzz}, we have $\forall \z\in\R^d$, 
\begin{equation*}
\| \overline{\z}_t - \z\|^2 = \| \overline{\z}_{t-1}-\z\|^2 - 2\eta \left\langle \widetilde{V}_{\x,t}, \overline{\z}_t-\z \right\rangle - \|\overline{\z}_t-\overline{\z}_{t-1}\|^2,    
\end{equation*}
and 
\begin{equation*}
\| \overline{\x}_t - \overline{\z}_t\|^2 = \| \overline{\z}_{t-1}-\overline{\z}_t\|^2 - 2\eta \left\langle \widetilde{V}_{\z,t-1}, \overline{\x}_t-\overline{\z}_t \right\rangle - \|\overline{\x}_t-\overline{\z}_{t-1}\|^2.    
\end{equation*}
Thus 
\begin{equation*}
    \begin{split}
&\|\overline{\z}_t -\z\|^2+\|\overline{\x}_t-\overline{\z}_t\|^2 \\
= & \|\overline{\z}_{t-1}-\z\|^2 -\|\overline{\x}_t-\overline{\z}_{t-1}\|^2  
- 2\eta \left\langle \widetilde{V}_{\x,t}, \overline{\z}_t-\z \right\rangle
- 2\eta \left\langle \widetilde{V}_{\z,t-1}, \overline{\x}_t-\overline{\z}_t \right\rangle \\
=  & \|\overline{\z}_{t-1}-\z\|^2 -\|\overline{\x}_t-\overline{\z}_{t-1}\|^2 - 2\eta \left\langle \widetilde{V}_{\x,t}, \overline{\x}_t-\z\right\rangle + 2\eta \left\langle \widetilde{V}_{\x,t}-\widetilde{V}_{\z,t-1},\overline{\x}_t-\overline{\z}_t \right\rangle \\
\leq & \|\overline{\z}_{t-1}-\z\|^2 -\|\overline{\x}_t-\overline{\z}_{t-1}\|^2 - 2\eta \left\langle \widetilde{V}_{\x,t}, \overline{\x}_t-\z\right\rangle  + \eta^2 \left\| \widetilde{V}_{\x,t} - \widetilde{V}_{\z,t-1}\right\|^2 + \| \overline{\x}_t-\overline{\z}_t\|^2,
    \end{split}
\end{equation*}
where the inequality is due to  Young's inequality. In summary, the inequality above indicates that

\begin{equation}
\label{eqn:poentialll}
\|\overline{\z}_t-\z\|^2\leq \|\overline{\z}_{t-1}-\z\|^2 -\|\overline{\x}_t-\overline{\z}_{t-1}\|^2 - 2\eta \left\langle \widetilde{V}_{\x,t}, \overline{\x}_t-\z\right\rangle 
+ \eta^2 \left\| \widetilde{V}_{\x,t} - \widetilde{V}_{\z,t-1}\right\|^2.
\end{equation}

\noindent To proceed,  we firstly focus on the last term of \eqref{eqn:poentialll}, and we have 
\begin{equation}
\label{eqn:dsaaserdsaawewqqw}
\begin{aligned}
&\left\| \widetilde{V}_{\x,t} - \widetilde{V}_{\z,t-1} \right\|^2 \\
= &\Big\|  \widetilde{V}_{\x,t} - V_{\x,t}
+ V_{\x,t} - V(\overline{\x}_t)  + V(\overline{\x}_t) - V(\overline{\z}_{t-1}) 
+ V(\overline{\z}_{t-1}) -V_{\z,t-1}
+ V_{\z,t-1} - \widetilde{V}_{\z,t-1} \Big\|^2 \\
\leq & 5\left\| \widetilde{V}_{\x,t} - V_{\x,t}\right\|^2 
+ 5\left\| V_{\x,t} - V(\overline{\x}_t) \right\|^2 
 + 5\left\| V(\overline{\x}_t) - V(\overline{\z}_{t-1}) \right\|^2 
+ 5\left\| V(\overline{\z}_{t-1}) - V_{\z,t-1} \right\|^2 \\
& + 5\left\| V_{\z,t-1} - \widetilde{V}_{\z,t-1} \right\|^2\\
\leq & 5\left\| \widetilde{V}_{\x,t} - V_{\x,t}\right\|^2 + \frac{5L^2}{M}\sum_{m=1}^M\|\x_t^m-\overline{\x}_t\|^2 + 5L^2 \|\overline{\x}_t-\overline{\z}_{t-1}\|^2 + \frac{5L^2}{M}\sum_{m=1}^M\|\z_{t-1}^m-\overline{\z}_{t-1}\|^2 \\
&+ 5\left\| V_{\z,t-1} - \widetilde{V}_{\z,t-1} \right\|^2,
\end{aligned}
\end{equation}
where the first inequality is based on Cauchy-Schwarz and Young's inequality, and  the second inequality is based on the $L$-smoothness of $V$, Cauchy-Schwarz and Young's inequality.  Next, we bound the third term of \eqref{eqn:poentialll}. We have 
\begin{equation}
    \begin{split}
    \label{eqn:third termtheorem1}
    &-   2\eta \left\langle \widetilde{V}_{\x,t}, \overline{\x}_t-\z\right\rangle  = {} -\frac{2\eta}{M}\sum_{m=1}^M \left\langle V(\x_t^m),\overline{\x}_t-\z  \right\rangle + 2\eta \left\langle V_{\x,t}-\widetilde{V}_{\x,t},\overline{\x}_t-\z \right\rangle\\
    = {} &  -\frac{2\eta}{M}\sum_{m=1}^M \left\langle V(\x_t^m),\overline{\x}_t-\x_t^m+\x_t^m-\z  \right\rangle + 2\eta \left\langle V_{\x,t}-\widetilde{V}_{\x,t},\overline{\x}_t-\z \right\rangle\\
    = & -\frac{2\eta}{M} \sum_{m=1}^M\left\langle V(\x_t^m),\x_t^m-\z\right\rangle - \frac{2\eta}{M}\sum_{m=1}^M\left\langle V(\x^m_t),\overline{\x}_t-\x_t^m\right\rangle + 2\eta \left\langle V_{\x,t}-\widetilde{V}_{\x,t},\overline{\x}_t-\z \right\rangle\\
    = & -\frac{2\eta}{M} \sum_{m=1}^M\left\langle V(\x_t^m),\x_t^m-\z\right\rangle - \underbrace{\frac{2\eta}{M}\sum_{m=1}^M\left\langle V(\overline{\x    }_t),\overline{\x}_t-\x_t^m\right\rangle}_{=0} + 2\eta \left\langle V_{\x,t}-\widetilde{V}_{\x,t},\overline{\x}_t-\z \right\rangle\\
    & +\frac{2\eta}{M}\sum_{m=1}^M\left\langle V(\overline{\x}_t)-V(\x_t^m),\overline{\x}_t-\x_t^m\right\rangle.
    \end{split}
\end{equation}
Finally, note that, based on the $L$-smoothness of $V$ and  Cauchy-Schwarz, we have 
\begin{equation}
    \begin{split}
    \label{eqn:Evvttheorem1}
\frac{2\eta}{M}\sum_{m=1}^M\left\langle V(\overline{\x}_t)-V(\x_t^m),\overline{\x}_t-\x_t^m\right\rangle\leq &\frac{2\eta}{M}\sum_{m=1}^M \| V(\overline{\x}_t)-V(\x_t^m)\|\|\overline{\x}_t-\x_t^m\|
\leq \frac{2\eta L}{M}\sum_{m=1}^M\|\overline{\x}_t-\x_t^m\|^2.
    \end{split}
\end{equation}
Plugging \eqref{eqn:dsaaserdsaawewqqw}, \eqref{eqn:third termtheorem1} and \eqref{eqn:Evvttheorem1} into  \eqref{eqn:poentialll} and rearrange, we get:
\begin{equation}
    \begin{split}
 &\frac{1}{M}\sum_{m=1}^M \left\langle V(\x_t^m),\x_t^m -\z \right\rangle  \leq \frac{\left\| \overline{\z}_{t-1} -\z\right\|^2  - \|\overline{\z}_t-\z\|^2}{2\eta} + \frac{(5L^2\eta^2-1)}{2\eta}\|\overline{\x}_t-\overline{\z}_{t-1}\|^2\\
 &+   \Gamma_t + \frac{5L^2\eta}{2M}\sum_{m=1}^M\|\z^m_{t-1}-\overline{\z}_{t-1}\|^2+ \frac{(5L^2\eta + 2L)}{2M}\sum_{m=1}^M\|\x_t^m-\overline{\x}_t\|^2,
    \end{split}
\end{equation}
where 
$$\Gamma_t=\left\langle V_{\x,t}-\widetilde{V}_{\x,t},\overline{\x}_t-\z \right\rangle + \frac{5\eta}{2}\left[ \left\| \widetilde{V}_{\x,t} - V_{\x,t}\right\|^2 +\left\| V_{\z,t-1} - \widetilde{V}_{\z,t-1} \right\|^2 \right], $$
which includes the noisy term. 
Sum it over $T=KR$, divide both sides by $T$, taking expectation on both sides, with $\eta\leq \frac{1}{\sqrt{14}L}$, and using the mononetonsity, we have:
\begin{equation}
    \begin{split}
    \label{eqn:prooflocalsgdimp}
& \E\left[\left\langle V(\z^*),\frac{1}{TM}\sum_{t=1}^T\sum_{m=1}^M\x^{m}_t-\z^* \right\rangle\right] \leq  \E\left[\frac{1}{TM}\sum_{t=1}^T\sum_{m=1}^M\left\langle V(\x_t^m),\x^m_t-\z^* \right\rangle \right] \\     
 \leq & \frac{\E[\|\overline{\z}_0-\z^*\|^2]}{2\eta T} + \frac{\E[\sum_{t=1}^T\Gamma_t]}{T} + \frac{5L^2\eta}{2MT}\sum_{t=1}^T\sum_{m=1}^M\E[\|\z^m_{t-1}-\overline{\z}_{t-1}\|^2]+ \frac{(5L^2\eta + 2L)}{2MT}\sum_{t=1}^T\sum_{m=1}^M\E[\|\x_t^m-\overline{\x}_t\|^2]\\
 \leq & \frac{\E[\|\overline{\z}_0-\z^*\|^2]}{2\eta T} + \frac{\E[\sum_{t=1}^T\Gamma_t]}{T} + \frac{6L}{MT}\sum_{t=1}^T\sum_{m=1}^M\left[\E[\|\z^m_{t-1}-\overline{\z}_{t-1}\|^2] +\E[\|\x_t^m-\overline{\x}_t\|^2]\right].
    \end{split}
\end{equation}
where we let  $$\z^*=\argmax_{\z\in\Z_D}\left\langle V(\z),\frac{1}{TM}\sum_{t=1}^T\sum_{m=1}^M\x^{m}_t-\z \right\rangle. $$
To proceed, we bound each term in \eqref{eqn:prooflocalsgdimp} respectively. For the second term, recall that 
$$\Gamma_t=\left\langle V_{\x,t}-\widetilde{V}_{\x,t},\overline{\x}_t-\z^* \right\rangle + \frac{5\eta}{2}\left[ \left\| \widetilde{V}_{\x,t} - V_{\x,t}\right\|^2 +\left\| V_{\z,t-1} - \widetilde{V}_{\z,t-1} \right\|^2 \right]. $$
We start from the first term. Note that $\z^*$ is a random variable that depends on all data. Therefore, this term is not necessarily $0$ in expectation. Nevertheless, there exist standard tricks for dealing this term \citep{juditsky2011solving,beznosikov2025distributed}. Let $\Delta_t=V_{\x,t}-\widetilde{V}_{\x,t}$, and $\p^{t+1}=\p^{t}-\eta \Delta_{t},$ with $\p^1=\overline{\z}_0$. Then we have 
\begin{equation*}
    \begin{split}
\left\langle V_{\x,t}-\widetilde{V}_{\x,t},\overline{\x}_t-\z^* \right\rangle  = &\left\langle V_{\x,t}-\widetilde{V}_{\x,t},\overline{\x}_t-\p^t+\p^t-\z^* \right\rangle \\  
= & \left\langle V_{\x,t}-\widetilde{V}_{\x,t},\overline{\x}_t-\p^t \right\rangle + \left\langle V_{\x,t}-\widetilde{V}_{\x,t},\p^t -\z^* \right\rangle,     
    \end{split}
\end{equation*}
Now the first term is zero in expectation as it does not depend on $\z^*$, and we mainly focus on the second term. Based on the definition of $\p^t$, we have 
\begin{align*}
\langle \eta \Delta_t, \p^t - \z^* \rangle &= \langle \eta \Delta_t, \p^t - \p^{t+1} \rangle + \langle \eta\Delta_t,  \p^{t+1}-\z^* \rangle \\
&= \langle \eta \Delta_t, \p^t - \p^{t+1} \rangle + \langle \p^{t+1} - \p^t, \z^* - \p^{t+1} \rangle \\
&= \langle \eta \Delta_t, \p^t - \p^{t+1} \rangle + \frac{1}{2} \|\p^t - \z^*\|^2 - \frac{1}{2} \|\p^{t+1} - \z^*\|^2 - \frac{1}{2} \|\p^t - \p^{t+1} \|^2 \\
&\leq \frac{\eta^2}{2} \|\Delta_t\|^2 + \frac{1}{2} \|\p^t - \p^{t+1} \|^2 + \frac{1}{2} \|\p^t - \z^*\|^2 - \frac{1}{2} \|\p^{t+1} - \z^*\|^2 - \frac{1}{2} \|\p^t - \p^{t+1} \|^2 \\
&= \frac{\eta^2}{2} \|\Delta_t\|^2 + \frac{1}{2} \|\p^t - \z^*\|^2 - \frac{1}{2} \|\p^{t+1} - \z^*\|^2.
\end{align*}
Here, the second equality is because 
\begin{equation*}
    \begin{split}
\|\p^t-\z^*\|^2 =\|\p^t-\p^{t+1}+\p^{t+1}-\z^*\|^2 =  \|\p^t-\p^{t+1}\|^2  + \|\p^{t+1}-\z^*\|^2 + 2\langle \p^{t+1} - \p^t, \z^* - \p^{t+1} \rangle,
    \end{split}
\end{equation*}
and the inequality is based on Cauchy-Schwarz.  
Finally, note that since $\widetilde{V}(\x_t^m)$ are i.i.d. for each $m\in[M]$, we have 
\begin{equation}
    \begin{split}
    \E\left[\|V_{\x,t}-\widetilde{V}_{\x,t}\|^2\right]\leq \frac{1}{M^2}\sum_{m=1}^M\E\left[\|V(\x_t^m)-\widetilde{V}(\x_t^m)\|^2\right]\leq \frac{\sigma^2}{M},     
    \end{split}
\end{equation}
and 
\begin{equation}
    \begin{split}
    \E\left[\|V_{\z,t-1}-\widetilde{V}_{\z,t-1}\|^2\right]\leq \frac{1}{M^2}\sum_{m=1}^M\E[\|V(\z_{t-1}^m)-\widetilde{V}(\z_{t-1}^m)\|^2]\leq \frac{\sigma^2}{M}.    
    \end{split}
\end{equation}
Combining the above inequalities, we get
\begin{equation}
    \begin{split} 
    \label{eqn:gammanoiczth1}
    \frac{1}{T}\E\left[\sum_{t=1}^{T}\Gamma_t\right] \leq  \frac{\E[\|\overline{\z}_0-\z^*\|^2]}{2\eta T} + \frac{6\sigma^2\eta}{M}.
    \end{split}
\end{equation}
Finally, we deal with the two drift terms in \eqref{eqn:prooflocalsgdimp}, and we introduce the following lemma. The proof is given in \ref{appendix:provde client draft}.
\begin{lemma}[Client Drift]
\label{lem:dfifafart}
For $\eta\leq \frac{1}{\sqrt{14K}L}$, we have 
\begin{equation*}
\frac{6L}{MT}\sum_{t=1}^T\sum_{m=1}^M\left[\E[\|\z^m_{t-1}-\overline{\z}_{t-1}\|^2] +\E[\|\x_t^m-\overline{\x}_t\|^2]\right]\leq  936eL\eta^2\sigma^2K.  
\end{equation*}

\end{lemma}

\noindent Combining \eqref{eqn:gammanoiczth1}, Lemma \ref{lem:dfifafart} with  \eqref{eqn:prooflocalsgdimp}, we finally get: 
\begin{equation*}
    \begin{split}
\E\left[\left\langle V(\z^*),\frac{1}{TM}\sum_{t=1}^T\sum_{m=1}^M\x^{m}_t-\z^* \right\rangle\right] \leq \frac{\E[\|\z_0-\z^*\|^2]}{\eta KR} + \frac{6\sigma^2\eta}{M} +936eL\eta^2\sigma^2K.
\end{split}
\end{equation*}
Let $\eta=\min\left\{\frac{1}{\sqrt{14K}L}, \frac{D\sqrt{M}}{\sigma\sqrt{6KR}}, \frac{D^{\frac{2}{3}}}{936^{\frac{1}{3}}e^{\frac{2}{3}}K^{\frac{2}{3}}R^{\frac{1}{3}}\sigma^{\frac{2}{3}}L^{\frac{1}{3}}} \right\},$ we get 
\begin{equation*}
    \begin{split}
& \E\left[\left\langle V(\z^*),\frac{1}{TM}\sum_{t=1}^T\sum_{m=1}^M\x^{m}_t-\z^* \right\rangle\right] 
\leq  \frac{20D^{\frac{4}{3}}\sigma^{\frac{2}{3}}L^{\frac{1}{3}}}{K^{\frac{1}{3}}R^{\frac{2}{3}}} + \frac{2\sqrt{6}D\sigma}{\sqrt{MKR}} + \frac{2\sqrt{14}LD^2}{\sqrt{K}R}. 
\end{split}
\end{equation*}

\subsection{Proof of Lemma \ref{lem:dfifafart}}
\label{appendix:provde client draft}
We first introduce the following lemma about the drift term in local updates. 
\begin{lemma}[Lemma 4 of \citet{pmlr-v119-woodworth20a}] For any $t\in[T]$ and $m'\not=m$, we have 
$$\E[\|\x_t^m-\overline{\x}_t\|^2]\leq \E[\|\x_t^m-\x_t^{m'}\|^2],\ \ \ \ \text{and},\ \ \ \ \E[\|\z_{t-1}^m-\overline{\z}_{t-1}\|^2]\leq \E[\|\z_{t-1}^m-\z_{t-1}^{m'}\|^2].$$     
\end{lemma}
We first focus on the drift of $\x$. We have:
\begin{equation}
    \begin{split}
    \label{eqn:draftxxx}
       & \E[\|\x_t^m -\x_t^{m'} \|^2]= \E\left[\left\|\z_{t-1}^m - \z_{t-1}^{m'} - \eta\left( \widetilde{V}(\z_{t-1}^m) -
       \widetilde{V}(\z_{t-1}^{m'}) \right) \right\|^2\right]\\
        =& \E\left[\left\|\z_{t-1}^m - \z_{t-1}^{m'}\right\|^2\right] + \eta^2\E\left[\left\| \widetilde{V}(\z_{t-1}^m) -\widetilde{V}(\z_{t-1}^{m'})  \right\|^2\right] - \E\left[2\eta \left\langle \z_{t-1}^m-\z_{t-1}^{m'},\widetilde{V}(\z_{t-1}^m)-\widetilde{V}(\z_{t-1}^{m'})\right\rangle  \right]\\
        =& \E\left[\left\|\z_{t-1}^m - \z_{t-1}^{m'}\right\|^2\right] + \eta^2\E\left[\left\| \widetilde{V}(\z_{t-1}^m) -\widetilde{V}(\z_{t-1}^{m'})  \right\|^2\right] - \E\left[2\eta \underbrace{\left\langle \z_{t-1}^m-\z_{t-1}^{m'},{V}(\z_{t-1}^m)-{V}(\z_{t-1}^{m'})\right\rangle}_{\geq 0}  \right]\\
        \leq &\E\left[\left\|\z_{t-1}^m - \z_{t-1}^{m'}\right\|^2\right] + \eta^2\E\left[\left\| \widetilde{V}(\z_{t-1}^m) -\widetilde{V}(\z_{t-1}^{m'})  \right\|^2\right] \\
        \leq & \E\left[\left\|\z_{t-1}^m - \z_{t-1}^{m'}\right\|^2\right] + 2\eta^2\E\left[\left\| {V}(\z_{t-1}^m) -\widetilde{V}(\z_{t-1}^{m})  \right\|^2\right] + 2\eta^2\E\left[\left\| {V}(\z_{t-1}^{m'}) -\widetilde{V}(\z_{t-1}^{m'})  \right\|^2\right]\\
        & + \eta^2 \E[\| V(\z_{t-1}^m)-V(\z_{t-1}^{m'})\|^2]\\
        \leq & \E\left[\left\|\z_{t-1}^m - \z_{t-1}^{m'}\right\|^2\right] + 4\eta^2\sigma^2 +\eta^2 \E[\| V(\z_{t-1}^m)-V(\z_{t-1}^{m'})\|^2]\\
        \leq & \left(1+\eta^2L^2\right)\E\left[\left\|\z_{t-1}^m - \z_{t-1}^{m'}\right\|^2\right]  + 4\eta^2\sigma^2\leq 2\E\left[\left\|\z_{t-1}^m - \z_{t-1}^{m'}\right\|^2\right]  + 4\eta^2\sigma^2,
    \end{split}
\end{equation}
where the first inequality is due to monotonicity, and the last inequality is based on $\eta\leq \frac{1}{L}$. Next, we bound the draft for $\z$. Let $t_0$, $t-t_0\leq K$, be the communication round such that  $\z_{t_0}^m=\z_{t_0}^{m'}$. Then we have  
\begin{equation}
    \begin{split}
\label{eqn:213esadase2eqe}
&\E[\|\z_{t}^m-\z_{t}^{m'}\|^2]=  \E\left[\left\| \z_{t-1}^m -\z_{t-1}^{m'} - \eta\left( \widetilde{V}(\x_{t}^{m})-\widetilde{V}(\x_{t}^{m'})\right) \right\|^2 \right] \\
= & \E\left[\left\|\z_{t-1}^m - \z_{t-1}^{m'}\right\|^2\right] + \eta^2\E\left[\left\|\widetilde{V}(\x_{t}^{m})-\widetilde{V}(\x_{t}^{m'}) \right\|^2\right] - 2\eta\E\left[\left\langle \z_{t-1}^m -\z_{t-1}^{m'} , \widetilde{V}(\x_{t}^{m})-\widetilde{V}(\x_{t}^{m'})  \right\rangle\right]\\
= & \E\left[\left\|\z_{t-1}^m - \z_{t-1}^{m'}\right\|^2\right] + \eta^2\E\left[\left\|\widetilde{V}(\x_{t}^{m})-\widetilde{V}(\x_{t}^{m'}) \right\|^2\right]- 2\eta\underbrace{\E\left[\left\langle \x_t^m -\x_t^{m'}, \widetilde{V}(\x_{t}^{m})-\widetilde{V}(\x_{t}^{m'})  \right\rangle\right]}_{\geq 0}\\
& - 2\eta\E\left[\left\langle \z_{t-1}^m - \x_t^m+\x_{t}^{m'}-\z_{t-1}^{m'}, \widetilde{V}(\x_{t}^{m})-\widetilde{V}(\x_{t}^{m'})  \right\rangle\right]\\
\leq &\E\left[\left\|\z_{t-1}^m - \z_{t-1}^{m'}\right\|^2\right] + \eta^2\E\left[\left\|\widetilde{V}(\x_{t}^{m})-\widetilde{V}(\x_{t}^{m'}) \right\|^2\right]- 2\eta^2{\E\left[\left\langle \widetilde{V}(\z_{t-1}^m) -\widetilde{V}(\z_{t-1}^{m'}), \widetilde{V}(\x_{t}^{m})-\widetilde{V}(\x_{t}^{m'})  \right\rangle\right]}\\
\leq & \E\left[\left\|\z_{t-1}^m - \z_{t-1}^{m'}\right\|^2\right] + 4\eta^2 \E\left[\left\|\widetilde{V}(\z_{t-1}^{m})-\widetilde{V}(\z_{t-1}^{m'}) \right\|^2\right] + 5\eta^2 \E\left[ \left\|\widetilde{V}(\x_{t}^{m})-\widetilde{V}(\x_{t}^{m'}) \right\|^2\right]\\
\leq & \E\left[\left\|\z_{t-1}^m - \z_{t-1}^{m'}\right\|^2\right] + 4\eta^2 \E\left[ L^2\left\|\z_{t-1}^m-\z_{t-1}^{m'} \right\|^2 + 4\sigma^2\right] + 5\eta^2 \E\left[ L^2\left\|\x_{t}^{m}-\x_{t}^{m'} \right\|^2 + 4\sigma^2\right]\\
\leq & \left(1+14\eta^2L^2\right)\E\left[\left\|\z_{t-1}^m - \z_{t-1}^{m'}\right\|^2\right] + 37 \eta^2\sigma^2\leq \left(1+\frac{1}{K}\right)\E\left[\left\|\z_{t-1}^m - \z_{t-1}^{m'}\right\|^2\right] + 37 \eta^2\sigma^2\\
\leq & 37e\eta^2\sigma^2K.
      \end{split}
\end{equation}
where in the fourth inequality  we used \eqref{eqn:draftxxx}, and the last two inequalities are based on the fact that  $\eta\leq \frac{1}{\sqrt{14K}L}$ and $\left(1+\frac{1}{K}\right)^K\leq e$. Combining with \eqref{eqn:draftxxx}, we have 
$$ \E[\|\x_t^m -\x_t^{m'} \|^2]\leq 78e\eta^2\sigma^2K. $$

\section{Proof of Theorem \ref{thm:affine}}
\label{sec:app:prof:affine}
We first introduce the following lemma,  which is  Lemma D.1 of \citet{gorbunov2022extragradient}.

\begin{lemma}
Suppose the operator $V$ is affine and $L$-smooth, and 
let operator $F(\z)=V(\z-\eta V(\z))$. Then $F$ is $\frac{2}{\eta}$-cocoercive, for $\eta\leq \frac{1}{L}$. That is, for any $\z,\z'\in\R^d$, $$\|F(\z)-F(\z')\|^2\leq \frac{2}{\eta}\left\langle F(\z)-F(\z'),\z-\z'\right\rangle.$$   
\end{lemma}
Next, we extend this conclusion to the stochastic setting. The proof can be found in section \ref{section:prop:seg:co}.

\begin{proposition}[Stochastic co-coercivity of EG]\label{prop:seg-coco}
\label{prop:seg:co}
Assume $V$ is affine and $L$-smooth, $\mathbb{E}[\widetilde V(\z)]=V(\z)$, and $\mathbb{E}\|\widetilde V(\z)-V(\z)\|^2\le \sigma^2$ for $\z\in\R^d$. Let
\[
F_{\text{SEG}}(\z)=V\bigl(\z-\eta\,\widetilde V(\z)\bigr),
\]
and suppose $\eta\le 1/L$. Then for any $\z,\z'\in\R^d$,

\begin{equation*}
\mathbb{E}\|F_{\text{SEG}}(\z)-F_{\text{SEG}}(\z')\|^2
\leq \frac{2}{\eta}\mathbb{E}\big\langle F_{\text{SEG}}(\z)-F_{\text{SEG}}(\z'),\,\z-\z'\big\rangle+  4L^2\eta^2\sigma^2.
\end{equation*}
\end{proposition}
By using this lemma, we provide an alternative proof for upper bound of the drift term in \eqref{eqn:213esadase2eqe}. We have
\begin{equation}
    \begin{split}
\label{eqn:213esadase2eqe2}
&\E[\|\z_{t}^m-\z_{t}^{m'}\|^2]=  \E\left[\left\| \z_{t-1}^m -\z_{t-1}^{m'} - \eta\left( \widetilde{V}(\x_{t}^{m})-\widetilde{V}(\x_{t}^{m'})\right) \right\|^2 \right] \\
= & \E\left[\left\|\z_{t-1}^m - \z_{t-1}^{m'}\right\|^2\right] + \eta^2\E\left[\left\|\widetilde{V}(\x_{t}^{m})-\widetilde{V}(\x_{t}^{m'}) \right\|^2\right] - 2\eta\E\left[\left\langle \z_{t-1}^m -\z_{t-1}^{m'} , \widetilde{V}(\x_{t}^{m})-\widetilde{V}(\x_{t}^{m'})  \right\rangle\right]\\
\leq & \E\left[\left\|\z_{t-1}^m - \z_{t-1}^{m'}\right\|^2\right] + \eta^2\E\left[\left\|{V}(\x_{t}^{m})-{V}(\x_{t}^{m'}) \right\|^2\right] - 2\eta\E\left[\left\langle \z_{t-1}^m -\z_{t-1}^{m'} , \widetilde{V}(\x_{t}^{m})-\widetilde{V}(\x_{t}^{m'})  \right\rangle\right]+4\eta^2\sigma^2\\
\leq &\E\left[\left\|\z_{t-1}^m - \z_{t-1}^{m'}\right\|^2\right] + 8\eta^2\sigma^2\leq 8\eta^2\sigma^2K,
      \end{split}
\end{equation}
where for the second inequality we used Proposition \ref{prop:seg-coco} and $\eta^\leq \frac{1}{L^2}$. Combining with \eqref{eqn:draftxxx}, it also implies that 
$$\E[\|\x_t^m -\x_t^{m'} \|^2]\leq 20 \eta^2\sigma^2K.$$
Therefore, we can immediately obtain the following bound for the client drift.
\begin{lemma}[Client Drift for Affine Operator]
\label{lem:dfifafart222}
For $\eta\leq \frac{1}{L}$, we have 
\begin{equation}
\frac{6L}{MT}\sum_{t=1}^T\sum_{m=1}^M\left[\E[\|\z^m_{t-1}-\overline{\z}_{t-1}\|^2] +\E[\|\x_t^m-\overline{\x}_t\|^2]\right]\leq  180L\eta^2\sigma^2K. 
\end{equation}
\end{lemma}
we finally get: 
\begin{equation}
    \begin{split}
\E\left[\left\langle V(\z^*),\frac{1}{TM}\sum_{t=1}^T\sum_{m=1}^M\x^{m}_t-\z^* \right\rangle\right] \leq \frac{\E[\|\z_0-\z^*\|^2]}{\eta KR} + \frac{6\sigma^2\eta}{M} +180L\eta^2\sigma^2K.
\end{split}
\end{equation}
Let $\eta=\min\left\{\frac{1}{L}, \frac{D\sqrt{M}}{\sigma\sqrt{6KR}}, \frac{D^{\frac{2}{3}}}{180^{\frac{1}{3}}K^{\frac{2}{3}}R^{\frac{1}{3}}\sigma^{\frac{2}{3}}L^{\frac{1}{3}}} \right\},$ we get 
\begin{equation}
    \begin{split}
& \E\left[\left\langle V(\z^*),\frac{1}{TM}\sum_{t=1}^T\sum_{m=1}^M\x^{m}_t-\z^* \right\rangle\right] 
\leq  \frac{20D^{\frac{4}{3}}\sigma^{\frac{2}{3}}L^{\frac{1}{3}}}{K^{\frac{1}{3}}R^{\frac{2}{3}}} + \frac{2\sqrt{6}D\sigma}{\sqrt{MKR}} + \frac{2LD^2}{{K}R}. 
\end{split}
\end{equation}

\subsection{Proof of Proposition \ref{prop:seg:co}}
\label{section:prop:seg:co}
Let $
F(\z)=V\bigl(\z-\eta\,V(\z)\bigr)$. Since $V$ is affine, let 
$V(\z)=A\z+b$. We have
\begin{align}
F_{\text{SEG}}(\z)-F_{\text{SEG}}(\z')
&=V(\z-\eta\widetilde V(\z)) - V(\z'-\eta\widetilde V(\z')) \nonumber\\
&=A\!\left[(\z-\eta V(\z))-(\z'-\eta V(\z'))\right]
-\eta A\!\left[(\widetilde V(\z)-V(\z))-(\widetilde V(\z')-V(\z'))\right]. \label{eq:seg-decomp}
\end{align}
Taking squared norms and expectation, the cross term vanishes by $\mathbb{E}[\widetilde V(\cdot)-V(\cdot)]=0$, hence
\begin{equation*}
\begin{split}
&
\mathbb{E}\|F_{\text{SEG}}(\z)-F_{\text{SEG}}(\z')\|^2
=  \|F(\z)-F(\z')\|^2
+ \eta^2\, \mathbb{E}\big\|A\big((\widetilde V(\z)-V(\z))-(\widetilde V(\z')-V(\z'))\big)\big\|^2\\
\leq  &  \|F(\z)-F(\z')\|^2
+ L^2\eta^2\,\mathbb{E}\|(\widetilde V(\z)-V(\z))-(\widetilde V(\z')-V(\z'))\|^2\\
\leq & \|F(\z)-F(\z')\|^2 + 4L^2\eta^2\sigma^2\\
\leq & \frac{2}{\eta} \left\langle F(\z)-F(\z'),\z-\z'\right\rangle+ 4L^2\eta^2\sigma^2 
= \frac{2}{\eta}\mathbb{E}\big\langle F_{\text{SEG}}(\z)-F_{\text{SEG}}(\z'),\,\z-\z'\big\rangle+  4L^2\eta^2\sigma^2.
\end{split}
\end{equation*}

\section{Proof of Theorem \ref{thm:inexactppa}}
\label{secpfofinscaew333}

Similar to previous proof, we start by defining the following shadow updates: Let $\overline{\x}_t=\frac{1}{M}\sum_{m=1}^M\x_t^{m}$, and $\overline{\z}_t=\frac{1}{M}\sum_{m=1}^M\z_t^{m}$, and 
\begin{equation*}
    \begin{split}
   \overline{\z}_{t} = \frac{1}{M}\sum_{m=1}^M \z^m_t = \overline{\z}_{t-1} -\eta \frac{1}{M}\sum_{m=1}^M\widetilde{V}(\x^m_{t})=\overline{\z}_{t-1}-\widetilde{V}_{\x,t}.     
    \end{split}
\end{equation*}
where we define $\widetilde{V}_{\x,t}=\frac{1}{M}\sum_{m=1}^M\widetilde{V}(\x^m_{t})$. 
and let ${V}_{\x,t}=\frac{1}{M}\sum_{m=1}^M{V}(\x^m_{t})$.

Next, we define another  series of shadow updates for the ``exact'' proximal update. Note that these updates are only used for analysis and are not required to be updated when implementing the algorithm. Specifically, at round $t$, for each machine $m=1,\dots,M$, define 
$$\x_t^{*,m}= \z_{t-1}^m -\eta V(\x_t^{*,m}),$$ 
and let $\overline{\x}_t^{*}=\frac{1}{M}\sum_{m=1}^M\x^{*,m}_t$, and $V_{\x^*,t}=\frac{1}{M}\sum_{m=1}^MV(\x_t^{*,m})$. We first introduce the following lemma about the properties of $\x_t^{*,m}$. The proof is given in Appendix \ref{sec:profx*prprro}. 

\begin{lemma}
\label{lem:sdewix}
Let $\x_t^{*,m} = \z_{t-1}^m -\eta V(\x_t^{*,m}),$  $\gamma=\frac{1}{\eta\left(L+\frac{1}{\eta}\right)^2}$, and $\eta\leq \frac{1}{L}$.  then 
$$ \E\left[\|\x_t^m-\x_{t}^{*,m}\|^2 \right]\leq  0.25^{H}\eta^2G^2 + \eta^2\sigma^2. $$
Moreover, we have 
 $$\|\x_{t}^{*,m}-\z_{t-1}^m\|\leq \eta G.$$
\end{lemma}

\noindent We now begin the proof by first bounding the following potential function.
\begin{equation}
    \begin{split}   \label{eqn:kkkkkk41}
&\| \overline{\z}_{t} - \z\|^2 =  \left\| \overline{\z}_{t-1} - \eta \widetilde{V}_{\x,t} - \z\right\|^2   = \|\overline{\z}_{t-1}-\z\|^2 -2\eta \left\langle \widetilde{V}_{\x,t},\overline{\z}_{t-1}-\z \right\rangle + \| \overline{\z}_{t-1}-\overline{\z}_t\|^2 \\
 = & \|\overline{\z}_{t-1}-\z\|^2 -2\eta \left\langle {V}_{\x,t},\overline{\z}_{t-1}-\z \right\rangle + \| \overline{\z}_{t-1}-\overline{\z}_t\|^2 +2\eta \left\langle {V}_{\x,t}-\widetilde{V}_{\x,t},\overline{\z}_{t-1}-\z \right\rangle\\
 = & \|\overline{\z}_{t-1}-\z\|^2 -2\eta \left\langle {V}_{\x,t},\overline{\x}_{t}-\z \right\rangle + 2\eta\left\langle {V}_{\x,t},\overline{\x}_{t}-\overline{\z}_{t-1} \right\rangle  + \| \overline{\z}_{t-1}-\overline{\z}_t\|^2 +2\eta \left\langle {V}_{\x,t}-\widetilde{V}_{\x,t},\overline{\z}_{t-1}-\z \right\rangle.
    \end{split}
\end{equation}
Next, we bound each term in \eqref{eqn:kkkkkk41} respectively. For the second term, we have 
\begin{equation}
    \begin{split}
  &-2\eta \left\langle {V}_{\x,t},\overline{\x}_{t}-\z \right\rangle =  -2\eta \frac{1}{M}\sum_{m=1}^M\left\langle V(\x_t^m),\overline{\x}_{t}-\z \right\rangle =  -2\eta \frac{1}{M}\sum_{m=1}^M\left\langle V(\x_t^m),\overline{\x}_{t}-\x_t^m+\x_t^m-\z \right\rangle\\
  = & -2\eta \frac{1}{M}\sum_{m=1}^M\left\langle V(\x_t^m),\overline{\x}_{t}-\x_t^m \right\rangle  -2\eta \frac{1}{M}\sum_{m=1}^M\left\langle V(\x_t^m),\x_t^m-\z \right\rangle\\
  = & \underbrace{-2\eta \frac{1}{M}\sum_{m=1}^M\left\langle V(\overline{\x}_t),\overline{\x}_{t}-\x_t^m \right\rangle}_{=0}  -2\eta \frac{1}{M}\sum_{m=1}^M\left\langle V(\x_t^m),\x_t^m-\z \right\rangle +2 \eta\frac{1}{M}\sum_{m=1}^M \left\langle V(\overline{\x}_t)-V({\x}^m_t),\overline{\x}_{t}-\x_t^m \right\rangle\\
 \leq & 2\eta L \frac{1}{M} \sum_{m=1}^M \|\x_t^m-\overline{\x}_t\|^2 -2\eta \frac{1}{M}\sum_{m=1}^M\left\langle V(\x_t^m),\x_t^m-\z \right\rangle,
    \end{split}
\end{equation}
where the inequality is based on the $L$-smoothness of $V$ and the Cauchy-Schwarz inequality. 
For the third term in \eqref{eqn:kkkkkk41}, we have 
\begin{equation}
    \begin{split}
    2\eta\left\langle {V}_{\x,t},\overline{\x}_{t}-\overline{\z}_{t-1} \right\rangle =  \underbrace{2\eta  \left\langle {V}_{\x,t} -V_{\x^*,t},\overline{\x}_{t}-\overline{\z}_{t-1} \right\rangle}_{A_1} + \underbrace{2\eta \left\langle {V}_{\x^*,t},\overline{\x}^*_{t}-\overline{\z}_{t-1} \right\rangle}_{A_2}   + \underbrace{2\eta \left\langle {V}_{\x^*,t},\overline{\x}_{t}-\overline{\x}^*_{t} \right\rangle.}_{A_3}
    \end{split}
\end{equation}
For $A_1$, we have 
\begin{equation}
\|\overline{\x}_t -\overline{\z}_{t-1}\|\leq     \|\overline{\x}_t -\overline{\x}^*_{t}\| +   \|\overline{\x}^*_t -\overline{\z}_{t-1}\|=\|\overline{\x}_t -\overline{\x}^*_{t}\| +   \eta\|V_{\x^*,t}\|
\end{equation}
therefore 
\begin{equation*}
\begin{split}
 A_1\leq & 2\eta\|V_{\x,t} - V_{\x^*,t}\| \|\overline{\x}_t -\overline{\x}^*_{t}\| + 2\eta^2 \| V_{\x^*,t}\|\|V_{\x,t} - V_{\x^*,t}\| \\
 \leq & \frac{\eta (L+9)}{M}\sum_{m=1}^M\|\x_t^m-\x_t^{*,m}\|^2 + \frac{\eta^2}{4}\|V_{\x^*,t}\|^2 .
\end{split}
\end{equation*}
We also have 
\begin{equation*}
    \begin{split}
A_2 = -2\eta^2\|V_{\x^*,t}\|^2,        
    \end{split}
\end{equation*}
and 
$$ A_3 \leq \frac{\eta^2}{4} \|V_{\x^*,t}\|^2+\frac{8}{M}\sum_{m=1}^M\|\x^m_t-\x_t^{*,m}\|^2.$$
Therefore, we have 
\begin{equation*}
    \begin{split}
 2\eta\left\langle {V}_{\x,t},\overline{\x}_{t}-\overline{\z}_{t-1} \right\rangle \leq &  \frac{15}{M}\sum_{m=1}^M\|\x_t^m-\x_t^{*,m}\|^2-1.5\eta^2\|V_{\x^*,t}\|^2 
    \end{split}
\end{equation*}
Finally, for the fourth term of \eqref{eqn:kkkkkk41}, we have 
\begin{equation}
    \begin{split}
&\|\overline{\z}_{t-1} -\overline{\z}_t\|^2 =    \|\overline{\z}_{t-1} -\overline{\x}_t^*+\overline{\x}_t^* -\overline{\z}_t\|^2 \leq 1.5    \|\overline{\z}_{t-1} -\overline{\x}_t^*\|^2  + 3\|\overline{\x}_t^* -\overline{\z}_t\|^2 \\
=& 1.5\eta^2 \|V_{\x^*,t}\|^2 + 3\eta^2\|V_{\x^*,t}-\widetilde{V}_{\x,t}\|^2 \\
\leq & 1.5\eta^2  \|V_{\x^*,t}\|^2+\frac{6\eta^2 L^2}{M}\sum_{m=1}^M\|\x_t^{*,m}-\x_t^m\|^2 + {6\eta^2}\|V_{\x,t}-\widetilde{V}_{\x,t}\|^2
    \end{split}
\end{equation}
Therefore, combining all the above inequalities, for \eqref{eqn:kkkkkk41} the upper bound becomes
\begin{equation*}
    \begin{split}
\|\overline{\z}_t -\z\|^2 \leq &\|\overline{\z}_{t-1} -\z\|^2  + 2\eta L \frac{1}{M} \sum_{m=1}^M \|\x_t^m-\overline{\x}_t\|^2 -2\eta \frac{1}{M}\sum_{m=1}^M\left\langle V(\x_t^m),\x_t^m-\z \right\rangle   \\
&+ \frac{30}{M}\sum_{m=1}^M\|\x_t^m-\x_t^{*,m}\|^2 +{6\eta^2}\|V_{\x,t}-\widetilde{V}_{\x,t}\|^2+2\eta \left\langle {V}_{\x,t}-\widetilde{V}_{\x,t},\overline{\z}_{t-1}-\z \right\rangle
    \end{split}
\end{equation*}
Rearrange, and taking expectation on both sides, we get 
\begin{equation}
    \begin{split}
\label{eqn:fianalboundppafirst}
&\E\left[\frac{1}{MT}\sum_{t=1}^T\sum_{m=1}^M   \left\langle V(\x_{t}^m),\x_t^m-\z \right\rangle\right]\leq  \frac{\E[\|\overline{\z}_0-\z\|^2]}{2\eta T} +\frac{L}{MT}\sum_{t=1}^T\sum_{m=1}^M\E\left[\|\x_t^m-\overline{\x}_t\|^2\right]\\ &+\frac{15}{\eta M T}\sum_{t=1}^T\sum_{m=1}^M\E\left[\|\x^m_t-\x_t^{*,m}\|^2\right] + \E\left[\frac{1}{T}\sum_{t=1}^T\Gamma_t\right],  
    \end{split}
\end{equation}
where 
$$ \Gamma_t = {3\eta}\|V_{\x,t}-\widetilde{V}_{\x,t}\|^2+\left\langle {V}_{\x,t}-\widetilde{V}_{\x,t},\overline{\z}_{t-1}-\z \right\rangle,$$
is the noise term. To proceed, we need to bound the drift term, and we have the following lemma. The proof is given in Appendix \ref{sec:proof:dtafgt new alg}.
\begin{lemma}
\label{lem::draft new algor}
We have
$$\E[\|\x_t^m-\overline{\x}_t\|^2]\leq \frac{60eG^2K^2L\eta^4+6\eta^2G^2}{4^H} + 60e\eta^4K^2L\sigma^2 + 54e\eta^2K\sigma^2.  $$
\end{lemma}
Finally, we have the following bound for the noise term $\Gamma_t$. The proof is given in Appendix \ref{proofnoowidewkwds}. 
\begin{lemma}
\label{lemma:noeiisder}
 We have 
 \begin{equation*}
    \begin{split} 
    \frac{1}{T}\E\left[\sum_{t=1}^{T}\Gamma_t\right] \leq  \frac{\E[\|\overline{\z}_0-\z^*\|^2]}{2\eta T} + \frac{2\sigma^2\eta}{M}.   
    \end{split}
\end{equation*}
\end{lemma}

We get
\begin{equation}
    \begin{split}
\E\left[\left\langle V(\z^*),\frac{1}{TM}\sum_{t=1}^T\sum_{m=1}^M\x^{m}_t-\z^* \right\rangle\right] \leq & \frac{\E[\|\z_0-\z^*\|^2]}{\eta KR} + \frac{2\sigma^2\eta}{M} + 60eK^2L^2\sigma^2\eta^4 + 54eKL\sigma^2\eta^2\\
&+ 15\eta\sigma^2 + \frac{60eG^2K^2L\eta^4+30\eta G^2}{4^H} 
\end{split}
\end{equation}
Finally, we let 
$$ \eta=\min\left\{\frac{1}{L}, \frac{D\sqrt{M}}{\sigma\sqrt{KR}}, \frac{D^{\frac{2}{5}}}{(60e)^{\frac{1}{5}}K^{\frac{3}{5}}R^{\frac{1}{5}}\sigma^{\frac{2}{5}}L^{\frac{2}{5}}}, \frac{D^{\frac{2}{3}}}{(54e)^{\frac{1}{3}}K^{\frac{2}{3}}R^{\frac{1}{3}}L^{\frac{1}{3}}\sigma^{\frac{2}{3}}},\frac{D}{\sigma\sqrt{ 15KR}}\right\}.$$

\subsection{Proof of Lemma \ref{lem:sdewix}}
\label{sec:profx*prprro}
Define $F(\x)=V(\x)+\frac{1}{\eta}(\x-\z_{t-1}^m)$. Firstly, it is easy to show that $F(x)$ is $\frac{1}{\eta}$ strongly monotone, and $L+\frac{1}{\eta}$ smooth. More specifically, for any $\x,\y$, we have 
\begin{equation}
    \begin{split}
\left\langle F(\x)-F(\y),\x-\y \right\rangle =   \left\langle V(\x)-V(\y),\x-\y \right\rangle + \frac{1}{\eta}\left\langle \x-\y,\x-\y \right\rangle \geq \frac{1}{\eta}\|\x-\y\|^2.      
    \end{split}
\end{equation}
On the other hand, 
\begin{equation}
    \begin{split}
\|F(\x)-F(\y)\|^2 \leq L^2\|\x-\y\|^2+\frac{2L}{\eta^2}  \|\x-\y\|^2 + \frac{1}{\eta}\|\x-\y\|^2 =\left(L+\frac{1}{\eta}\right)^2\|\x-\y\|^2.     
    \end{split}
\end{equation}
Next, we show the convergence. Let $\widetilde{F}(\x)=\widetilde{V}(\x)+\frac{1}{\eta}(\x-\z_{t-1}^m)$. Note that $\|\widetilde{F}(\x)-{F}(\x)\|^2\leq \sigma^2$. Firstly, note that:
$$ F(\x_t^{*,m})  = V(\x_t^{*,m})+\frac{1}{\eta}(\x_{t}^{*,m}-\z_{t-1}^m)=\mathbf{0}.$$
Moreover, since $\gamma=\frac{1}{\eta\left(L+\frac{1}{\eta}\right)^2}$, we have 
\begin{equation}
    \begin{split}
&\E\left[\|\x^{m}_{t,\ell}-\x_t^{*,m} \|^2\right] = \E\left[\|\x^{m}_{t,\ell-1}-\gamma\widetilde{F}(\x^{m}_{t,\ell-1})-\x_t^{*,m} +\gamma F(\x_t^{*,m}) \|^2 \right]\\    
= & \E[\|\x^{m}_{t,\ell-1}-\x_t^{*,m}\|^2]  -2\gamma\E\left[\left\langle F(\x^m_{t,\ell-1}) -F(\x^{*,m}_{t}), \x_{t,\ell-1}^m-\x_t^{*,m}\right\rangle\right] + \gamma^2\E[\|\widetilde{F}(\x_{t,\ell-1}^m)-F(\x_t^{*,m})\|^2]\\
\leq & \E[\|\x^{m}_{t,\ell-1}-\x_t^{*,m}\|^2] -\frac{2\gamma}{\eta} \E[\|\x_{t,\ell-1}^m-\x_t^{*,m}\|^2]+\gamma^2\E[\|{F}(\x_{t,\ell-1}^m)-F(\x_t^{*,m})\|^2+\gamma^2\E[\|{F}(\x_{t,\ell-1}^m)-\widetilde{F}(\x_{t,\ell-1}^{m})\|^2\\
= & \left(1-\frac{2\gamma}{\eta}+\gamma^2\left(L+\frac{1}{\eta}\right)^2 \right) \E[\|\x_{t,\ell-1}^m-\x_t^{*,m}\|^2]+\gamma^2\sigma^2\\
= & \left( 1 - \frac{1}{(\eta L+1)^2}\right)\E[\|\x_{t,\ell-1}^m-\x_t^{*,m}\|^2]+\gamma^2\sigma^2
\leq  \left( 1-\frac{1}{(\eta L+1)^2} \right)^H\E[\|\z^m_{t-1}-\x_t^{*,m}\|^2] + \frac{\sigma^2}{\left(L+\frac{1}{\eta}\right)^2}\\
\leq & \left( 1-\frac{1}{(\eta L+1)^2} \right)^H\E[\|\z^m_{t-1}-\x_t^{*,m}\|^2] + \eta^2\sigma^2
\leq  0.25^{H}\E[\|\z_{t-1}^m-\x_t^{*,m}\|^2] + \eta^2\sigma^2.
    \end{split}
\end{equation}

Since $F(\x)=V(\x)+\frac{1}{\eta}(\x-\z_{t-1}^m)$ is $\frac{1}{\eta}$-strongly monotone, we have 
$$ \left\langle F(\x_t^{*,m})-F(\z_{t-1}^m),\x_{t}^{*,m}-\z_{t-1}^m \right\rangle\geq \frac{1}{\eta}\|\x_{t}^{*,m}-\z_{t-1}^m\|^2.$$
Also, we have $F(\x_{t}^{*,m})=0$, and $F(\z_{t-1}^m)=V(\z_{t-1}^{m}),$ so
$$ G\|\x_{t}^{*,m}-\z_{t-1}^m\|\geq \left\langle F(\x_t^{*,m})-F(\z_{t-1}^m),\x_{t}^{*,m}-\z_{t-1}^m \right\rangle\geq  \frac{1}{\eta}\|\x_{t}^{*,m}-\z_{t-1}^m\|^2, $$
so $\|\x_{t}^{*,m}-\z_{t-1}^m\|\leq \eta G$.

\subsection{Proof of Lemma \ref{lem::draft new algor}}
\label{sec:proof:dtafgt new alg}

We have 
\begin{equation}
 \begin{split}
 \label{eqn:44:final boundplugin}
\|\x_t^m-\x_t^{m'}\|^2\leq 3\|\x_t^m -\x_t^{*,m}\|^2 + 3 \|\x_t^{*,m}-\x_t^{*,m'}\|^2   + 3 \|\x_t^{m'}-\x_t^{*,m'}\|^2,
 \end{split}   
\end{equation}
where the inequality is based on Cauchy-Schwarz and Young's inequalities. Note that, the first and third terms are small based on Lemma \ref{lem:sdewix}. Therefore, in the following we  mainly focus on the middle term. We have 
\begin{equation*}
    \begin{split}
& \|\x_t^{*,m}-\x_t^{*,m'}\|^2  = \|\z_{t-1}^m-\eta {V}(\x_t^{*,m}) - \z_{t-1}^m + \eta V(\x_t^{*,m'})\|^2\\
= & \|\z_{t-1}^m-\z_{t-1}^{m'}\|^2 -2 \eta \left\langle  V(\x_t^{*,m}) -V(\x_t^{*,m'}), \z_{t-1}^m-\z_{t-1}^{m'} \right\rangle + \eta^2 \|V(\x_t^{*,m})-V(\x_t^{*,m'})\|^2\\
= & \|\z_{t-1}^m-\z_{t-1}^{m'}\|^2 -2 \eta \left\langle  V(\x_t^{*,m}) -V(\x_t^{*,m'}), \z_{t-1}^m-\x_t^{*,m}+\x_t^{*,m}-\x_t^{*,m'}+\x_{t}^{*,m'}-\z_{t-1}^{m'} \right\rangle \\
&+ \eta^2 \|V(\x_t^{*,m})-V(\x_t^{*,m'})\|^2\\
= & \|\z_{t-1}^m-\z_{t-1}^{m'}\|^2 - \eta^2\|V(\x_t^{*,m})-V(\x_t^{*,m'})\|^2 -2\eta  \underbrace{\left\langle  V(\x_t^{*,m}) -V(\x_t^{*,m'}), \x_t^{*,m}-\x_t^{*,m'}\right\rangle}_{\geq 0} \\
\leq &\|\z_{t-1}^m-\z_{t-1}^{m'}\|^2.
    \end{split}
\end{equation*}
The inequality above indicates that the drift of $\x_t^{*,m}$ is bounded by the 
drift of $\z_{t-1}^m$. Next, we turn to bound the 
drift of $\z_{t-1}^m$. We have 
\begin{equation}
    \begin{split}
    \label{eqnL:sdwewadasdaqqq}
&\|\z_{t}^m-\z_{t}^{m'}\|^2 = \|\z_{t-1}^m- \eta \widetilde{V}(\x_t^m) - \z_{t-1}^{m'} + \eta \widetilde{V}(\x_t^{m'})\|^2\\
=& \|\z_{t-1}^m-\z_{t-1}^{m'}\|^2 -2\eta \left\langle \widetilde{V}(\x_t^m)-\widetilde{V}(\x_t^{m'}),\z^m_{t-1}-\z_{t-1}^{m'}\right\rangle  + \eta^2 \left\| \widetilde{V}(\x_t^m)-\widetilde{V}(\x_t^{m'})\right\|^2
    \end{split}
\end{equation}
We deal with each term separately. For the third term of \eqref{eqnL:sdwewadasdaqqq}, we have 
\begin{equation}
    \begin{split}
    \label{eqn:midvd;we47441}
& \eta^2 \left\| \widetilde{V}(\x_t^m)-\widetilde{V}(\x_t^{m'})\right\|^2\\
=& \eta^2\left\| \widetilde{V}(\x_t^m) -V(\x_t^m)+V(\x_t^m)-V(\x_t^{*,m})+V(\x_t^{*,m})-V(\x_t^{*,m'})+V(\x_t^{*,m'})-V(\x_t^{m'})+V(\x_t^{m'})-\widetilde{V}(\x_t^{m'})\right\|^2\\
\leq &  8 \eta^2\left\| \widetilde{V}(\x_t^m) -V(\x_t^m)\right\|^2 + 8\eta^2 \left\| V(\x_t^m)-V(\x_t^{*,m})\right\|^2  + 2\eta^2 \left\| V(\x_t^{*,m})-V(\x_t^{*,m'})\right\|^2\\
&+ 8\eta^2\left\| V(\x_t^{*,m'})-V(\x_t^{m'})\right\|^2 + 8\eta^2\left\|V(\x_t^{m'})-\widetilde{V}(\x_t^{m'})\right\|^2.
    \end{split}
\end{equation}
For the second term in \eqref{eqnL:sdwewadasdaqqq}, we have 
\begin{equation}
    \begin{split}
    \label{eqn:midvd;we47442}
    &-2\eta \left\langle \widetilde{V}(\x_t^m)-\widetilde{V}(\x_t^{m'}),\z^m_{t-1}-\z_{t-1}^{m'}\right\rangle =  -2\eta \left\langle {V}(\x_t^{m,*})-{V}(\x_t^{m',*}),\z^m_{t-1}-\z_{t-1}^{m'}\right\rangle  \\
    &-2\eta \left\langle \widetilde{V}(\x_t^m)-V(\x_t^{m,*})+V(\x_t^{m',*})-\widetilde{V}(\x_t^{m'}),\z^m_{t-1}-\z_{t-1}^{m'}\right\rangle\\
= &-2\eta \left\langle {V}(\x_t^{m,*})-{V}(\x_t^{m',*}),\z^m_{t-1}-\x_t^{m,*}+\x_t^{m,*}-\x_t^{m',*}+\x_t^{m',*}-\z_{t-1}^{m'}\right\rangle  \\
    &-2\eta \left\langle \widetilde{V}(\x_t^m)-V(\x_t^{m,*})+V(\x_t^{m',*})-\widetilde{V}(\x_t^{m'}),\z^m_{t-1}-\z_{t-1}^{m'}\right\rangle\\
= & -2\eta^2 \left\| {V}(\x_t^{m,*})-{V}(\x_t^{m',*})\right\|^2  -\underbrace{2\eta \left\langle {V}(\x_t^{m,*})-{V}(\x_t^{m',*}),\x_t^{m,*}-\x_t^{m',*}\right\rangle}_{\geq 0}\\
&-2\eta \left\langle \widetilde{V}(\x_t^m)-V(\x_t^{m,*})+V(\x_t^{m',*})-\widetilde{V}(\x_t^{m'}),\z^m_{t-1}-\z_{t-1}^{m'}\right\rangle\\
    \end{split}
\end{equation}
Note that the first term here in \eqref{eqn:midvd;we47442} cancels the third term in  \eqref{eqn:midvd;we47441}. 
Plugging \eqref{eqn:midvd;we47441} and \eqref{eqn:midvd;we47442} into \eqref{eqnL:sdwewadasdaqqq} and taking expectation on both sides, we get 
\begin{equation*}
    \begin{split}
& \E\left[\|\z_t^m-\z_{t}^{m'}\|^2 \right] \leq \E[\|\z^m_{t-1}-\z^m_{t-1}\|^2] -\E\left[2\eta \left\langle \widetilde{V}(\x_t^m)-V(\x_t^{m,*})+V(\x_t^{m',*})-\widetilde{V}(\x_t^{m'}),\z^m_{t-1}-\z_{t-1}^{m'}\right\rangle\right]  \\     
&+ \E\left[8 \eta^2\left\| \widetilde{V}(\x_t^m) -V(\x_t^m)\right\|^2 + 8\eta^2 \left\| V(\x_t^m)-V(\x_t^{*,m})\right\|^2 
+ 8\eta^2\left\| V(\x_t^{*,m'})-V(\x_t^{m'})\right\|^2 \right]\\
&+ 8\eta^2\E\left[\left\|V(\x_t^{m'})-\widetilde{V}(\x_t^{m'})\right\|^2\right]\\
\leq & \E[\|\z^m_{t-1}-\z^m_{t-1}\|^2] + \E\left[2\eta^2 KL \|\x_t^m-\x_t^{*,m}\|^2 + 2\eta^2 K L\|\x_t^{m'}-\x_t^{*,m'}\|^2+\frac{1}{K}\|\z_{t-1}^m-\z_{t-1}^{m'}\|^2\right]\\
&+ 16\eta^2\sigma^2 + 8\eta^2L\E[\|\x_t^m-\x_{t}^{*,m}\|^2] + 8\eta^2L\E[\|\x_t^{m'}-\x_{t}^{*,m'}\|^2]\\
\leq & \left( 1+\frac{1}{K}\right)\E[\|\z^m_{t-1}-\z^m_{t-1}\|^2]  +  20 \eta^2KL \E[\|\x_t^{m}-\x_t^{*,m}\|^2] + 16\eta^2\sigma^2  
    \end{split}
\end{equation*}
Based on Lemma \ref{lem:sdewix}, we have 
$$ \E\left[\|\x_t^m-\x_{t}^{*,m}\|^2 \right]\leq   0.25^{H}\eta^2G^2 + \eta^2\sigma^2. $$
Thus,
\begin{equation*}
    \begin{split}
        \E\left[\|\z_t^m-\z_{t}^{m'}\|^2 \right]\leq & \left( 1+\frac{1}{K}\right)\E[\|\z^m_{t-1}-\z^m_{t-1}\|^2]  +  20 \eta^2KL\left( 0.25^{H}\eta^2G^2 + \eta^2\sigma^2+\eta^2\sigma^2\right) + 16\eta^2\sigma^2\\
        \leq &  \frac{20eG^2K^2L\eta^4}{4^H} + 20e\eta^4K^2L\sigma^2 + 16e\eta^2K\sigma^2. 
    \end{split}
\end{equation*}
Plugging it back to \eqref{eqn:44:final boundplugin}, and again apply Lemma \ref{lem:sdewix}, and use the fact that $\eta\leq \frac{1}{L}\leq 1$, we get:
$$\E[\|\x_t^m-\x_t^{m'}\|^2]\leq   \frac{60eG^2K^2L\eta^4+6\eta^2G^2}{4^H} + 60e\eta^4K^2L\sigma^2 + 54e\eta^2K\sigma^2. $$

\subsection{Proof of Lemma \ref{lemma:noeiisder}}
\label{proofnoowidewkwds}

let  $$\z^*=\argmax_{\z}\left\langle V(\z),\frac{1}{TM}\sum_{t=1}^T\sum_{m=1}^M\x^{m}_t-\z \right\rangle. $$
Note that $\z^*$ is a random variable that depends on all data. To proceed, we bound each term  respectively. For the second term, recall that 
$$ \Gamma_t = \frac{2\eta}{M}\sum_{m=1}^M\|V(\x_t^m)-\widetilde{V}(\x_t^m)\|^2 +\left\langle {V}_{\x,t}-\widetilde{V}_{\x,t},\overline{\z}_{t-1}-\z \right\rangle, $$
We start from the second term.  Let $\Delta_t=V_{\x,t}-\widetilde{V}_{\x,t}$, and $\p_{t+1}=\p_{t}-\eta \Delta_{t},$ with $\p_1=\overline{\z}_0$. Then we have 
\begin{equation*}
    \begin{split}
\left\langle V_{\x,t}-\widetilde{V}_{\x,t},\overline{\z}_{t-1}-\z^* \right\rangle  = &\left\langle V_{\x,t}-\widetilde{V}_{\x,t},\overline{\z}_{t-1}-\p_t+\p_t-\z^* \right\rangle \\  
= & \left\langle V_{\x,t}-\widetilde{V}_{\x,t},\overline{\z}_{t-1}-\p_t \right\rangle + \left\langle V_{\x,t}-\widetilde{V}_{\x,t},\p_t -\z^* \right\rangle,     
    \end{split}
\end{equation*}
Note that the first term is zero in expectation as it does not depend on $\z^*$, and we mainly focus on the second term. Based on the definition of $\p_t$, we have 
\begin{align*}
\langle \eta \Delta_t, \p^t - \z^* \rangle 
&= \langle \eta \Delta_t, \p^t - \p^{t+1} \rangle + \langle \p^{t+1} - \p^t, \z^* - \p^{t+1} \rangle \\
&= \langle \eta \Delta_t, \p^t - \p^{t+1} \rangle + \frac{1}{2} \|\p^t - \z^*\|^2 - \frac{1}{2} \|\p^{t+1} - \z^*\|^2 - \frac{1}{2} \|\p^t - \p^{t+1} \|^2 \\
&\leq \frac{\eta^2}{2} \|\Delta_t\|^2 + \frac{1}{2} \|\p^t - \p^{t+1} \|^2 + \frac{1}{2} \|\p^t - \z^*\|^2 - \frac{1}{2} \|\p^{t+1} - \z^*\|^2 - \frac{1}{2} \|\p^t - \p^{t+1} \|^2 \\
&= \frac{\eta^2}{2} \|\Delta_t\|^2 + \frac{1}{2} \|\p^t - \z^*\|^2 - \frac{1}{2} \|\p^{t+1} - \z^*\|^2.
\end{align*}
Here, the second equality is because 
\begin{equation*}
    \begin{split}
\|\p^t-\z^*\|^2 =\|\p^t-\p^{t+1}+\p^{t+1}-\z^*\|^2 =  \|\p^t-\p^{t+1}\|^2  + \|\p^{t+1}-\z^*\|^2 + 2\langle \p^{t+1} - \p^t, \z^* - \p^{t+1} \rangle,
    \end{split}
\end{equation*}
and the inequality is based on Cauchy-Schwarz.  
Finally, note that since $\widetilde{V}(\x_t^m)$ are i.i.d. for each $m\in[M]$, we have 
\begin{equation*}
    \begin{split}
    \E\left[\|V_{\x,t}-\widetilde{V}_{\x,t}\|^2\right]\leq \frac{1}{M^2}\sum_{m=1}^M\E\left[\|V(\x_t^m)-\widetilde{V}(\x_t^m)\|^2\right]\leq \frac{\sigma^2}{M},     
    \end{split}
\end{equation*}
Combining the above inequalities, we get
\begin{equation*}
    \begin{split} 
    \frac{1}{T}\E\left[\sum_{t=1}^{T}\Gamma_t\right] \leq  \frac{\E[\|\overline{\z}_0-\z^*\|^2]}{2\eta T} + \frac{2\sigma^2\eta}{M}.   
    \end{split}
\end{equation*}

\section{Proof of Theorem \ref{thm:optimalboundppass}}
The proof is similar to that of Theorem \ref{thm:inexactppa}, and the main difference is that we bound the potential function differently to make it depend on  $\|\overline{\x}_t-\overline{\x}_{t}^{*}\|^2$, for which we introduce a tight bound based on Assumption \ref{ass:boundedasddff}. We refer to Remark \ref{remark:whatisthedifference} for a detailed comparison and discussion.  
We start by recall the following notations.  Let $\overline{\x}_t=\frac{1}{M}\sum_{m=1}^M\x_t^{m}$, and $\overline{\z}_t=\frac{1}{M}\sum_{m=1}^M\z_t^{m}$, and 
\begin{equation*}
    \begin{split}
   \overline{\z}_{t} = \frac{1}{M}\sum_{m=1}^M \z^m_t = \overline{\z}_{t-1} -\eta \frac{1}{M}\sum_{m=1}^M\widetilde{V}(\x^m_{t})=\overline{\z}_{t-1}-\widetilde{V}_{\x,t}.     
    \end{split}
\end{equation*}
where $\widetilde{V}_{\x,t}=\frac{1}{M}\sum_{m=1}^M\widetilde{V}(\x^m_{t})$. Moreover, we have ${V}_{\x,t}=\frac{1}{M}\sum_{m=1}^M{V}(\x^m_{t})$. Also recall the definite of the shadow  proximal  update:
$$\x_t^{*,m}= \z_{t-1}^m -\eta V(\x_t^{*,m}),$$ 
and let $\overline{\x}_t^{*}=\frac{1}{M}\sum_{m=1}^M\x^{*,m}_t$, and $V_{\x^*,t}=\frac{1}{M}\sum_{m=1}^MV(\x_t^{*,m})$. We first introduce the following lemma about the properties of $\x_t^{*,m}$, apart from Lemma \ref{lem:sdewix}. It is the key for obtaining the improved results. The proof is given in Appendix \ref{sec:profx*prprro18}. 

\begin{lemma}
\label{lem:sdewix1818181}
We  have 
$$\E[\|\overline{\x}_t-\overline{\x}_t^*\|^2]\leq 6\Lambda^2\eta^6\sigma^4+\frac{\eta^2\sigma^2}{M} + 0.25^{H}\cdot4\eta^2G^2 +   0.25^H\cdot H^2  \Lambda^2 \eta^6G^4$$
\end{lemma}

\noindent We now begin the proof by first bounding the following potential function.
\begin{equation}
    \begin{split}   \label{eqn:dasdwwqddsads}
&\| \overline{\z}_{t} - \z\|^2 =  \left\| \overline{\z}_{t-1} - \eta \widetilde{V}_{\x,t} - \z\right\|^2   = \|\overline{\z}_{t-1}-\z\|^2 -2\eta \left\langle \widetilde{V}_{\x,t},\overline{\z}_{t-1}-\z \right\rangle + \| \overline{\z}_{t-1}-\overline{\z}_t\|^2 \\
 = & \|\overline{\z}_{t-1}-\z\|^2 -2\eta \left\langle {V}_{\x,t},\overline{\z}_{t-1}-\z \right\rangle + \| \overline{\z}_{t-1}-\overline{\z}_t\|^2 +2\eta \left\langle {V}_{\x,t}-\widetilde{V}_{\x,t},\overline{\z}_{t-1}-\z \right\rangle\\
 = & \|\overline{\z}_{t-1}-\z\|^2 -2\eta \left\langle {V}_{\x,t},\overline{\x}_{t}-\z \right\rangle + 2\eta\left\langle {V}_{\x,t},\overline{\x}_{t}-\overline{\z}_{t-1} \right\rangle  + \| \overline{\z}_{t-1}-\overline{\z}_t\|^2 +2\eta \left\langle {V}_{\x,t}-\widetilde{V}_{\x,t},\overline{\z}_{t-1}-\z \right\rangle.
    \end{split}
\end{equation}
Next, we bound each term in \eqref{eqn:dasdwwqddsads} respectively. For the second term, we have 
\begin{equation}
    \begin{split}
  &-2\eta \left\langle {V}_{\x,t},\overline{\x}_{t}-\z \right\rangle =  -2\eta \frac{1}{M}\sum_{m=1}^M\left\langle V(\x_t^m),\overline{\x}_{t}-\z \right\rangle =  -2\eta \frac{1}{M}\sum_{m=1}^M\left\langle V(\x_t^m),\overline{\x}_{t}-\x_t^m+\x_t^m-\z \right\rangle\\
  = & -2\eta \frac{1}{M}\sum_{m=1}^M\left\langle V(\x_t^m),\overline{\x}_{t}-\x_t^m \right\rangle  -2\eta \frac{1}{M}\sum_{m=1}^M\left\langle V(\x_t^m),\x_t^m-\z \right\rangle\\
  = & \underbrace{-2\eta \frac{1}{M}\sum_{m=1}^M\left\langle V(\overline{\x}_t),\overline{\x}_{t}-\x_t^m \right\rangle}_{=0}  -2\eta \frac{1}{M}\sum_{m=1}^M\left\langle V(\x_t^m),\x_t^m-\z \right\rangle +2 \eta\frac{1}{M}\sum_{m=1}^M \left\langle V(\overline{\x}_t)-V({\x}^m_t),\overline{\x}_{t}-\x_t^m \right\rangle\\
 \leq & 2\eta L \frac{1}{M} \sum_{m=1}^M \|\x_t^m-\overline{\x}_t\|^2 -2\eta \frac{1}{M}\sum_{m=1}^M\left\langle V(\x_t^m),\x_t^m-\z \right\rangle,
    \end{split}
\end{equation}
where the inequality is based on the $L$-smoothness of $V$ and the Cauchy-Schwarz inequality. 
For the third term in \eqref{eqn:dasdwwqddsads}, we have 
\begin{equation}
    \begin{split}
    \label{eqn:34343433}
    2\eta\left\langle {V}_{\x,t},\overline{\x}_{t}-\overline{\z}_{t-1} \right\rangle =  \underbrace{2\eta  \left\langle {V}_{\x,t} -V_{\x^*,t},\overline{\x}_{t}-\overline{\z}_{t-1} \right\rangle}_{A_1} + \underbrace{2\eta \left\langle {V}_{\x^*,t},\overline{\x}^*_{t}-\overline{\z}_{t-1} \right\rangle}_{A_2}   + \underbrace{2\eta \left\langle {V}_{\x^*,t},\overline{\x}_{t}-\overline{\x}^*_{t} \right\rangle.}_{A_3}
    \end{split}
\end{equation}
For $A_1$, we have 
\begin{equation}
\|\overline{\x}_t -\overline{\z}_{t-1}\|\leq     \|\overline{\x}_t -\overline{\x}^*_{t}\| +   \|\overline{\x}^*_t -\overline{\z}_{t-1}\|=\|\overline{\x}_t -\overline{\x}^*_{t}\| +   \eta\|V_{\x^*,t}\|
\end{equation}
therefore 
\begin{equation}
\begin{split}
 A_1\leq & 2\eta\|V_{\x,t} - V_{\x^*,t}\| \|\overline{\x}_t -\overline{\x}^*_{t}\| + 2\eta^2 \| V_{\x^*,t}\|\|V_{\x,t} - V_{\x^*,t}\| \\
 \leq & \eta^2\|V_{\x,t}-V_{\x^*,t}\|^2 + \|\overline{\x}_t-\overline{\x}_t^* \|^2 + \frac{\eta^2}{4}\|V_{\x^*,t}\|^2 + 4\eta^2\|V_{\x,t}-V_{\x^*,t}\|^2\\
 \leq &  \frac{5\eta^2L^2}{M}\sum_{m=1}^M\|\x_t^{m}-\x_{t}^{*}\|^2 + \|\overline{\x}_t-\overline{\x}_t^{*,m}\|^2+\frac{\eta^2}{4}\|V_{\x^*,t}\|^2
\end{split}
\end{equation}
We also have 
\begin{equation*}
    \begin{split}
A_2 = -2\eta^2\|V_{\x^*,t}\|^2,        
    \end{split}
\end{equation*}
and 
$$ A_3 \leq \frac{\eta^2}{4} \|V_{\x^*,t}\|^2+{8}\|\overline{\x}_t-\overline{\x}_t^*\|^2$$
Combining plugging into \eqref{eqn:34343433}, we have 
\begin{equation*}
    \begin{split}
 2\eta\left\langle {V}_{\x,t},\overline{\x}_{t}-\overline{\z}_{t-1} \right\rangle \leq &  9\|\overline{\x}_t-\overline{\x}_t^*\|^2+\frac{5\eta^2L^2}{M}\sum_{m=1}^M\|\x_t^{m}-\x_t^{*,m}\|^2-1.5\eta^2\|V_{\x^*,t}\|^2 
    \end{split}
\end{equation*}
Finally, for the fourth term of \eqref{eqn:dasdwwqddsads}, we have 
\begin{equation}
    \begin{split}
&\|\overline{\z}_{t-1} -\overline{\z}_t\|^2 =    \|\overline{\z}_{t-1} -\overline{\x}_t^*+\overline{\x}_t^* -\overline{\z}_t\|^2 \leq 1.5    \|\overline{\z}_{t-1} -\overline{\x}_t^*\|^2  + 3\|\overline{\x}_t^* -\overline{\z}_t\|^2 \\
=& 1.5\eta^2 \|V_{\x^*,t}\|^2 + 3\eta^2\|V_{\x^*,t}-\widetilde{V}_{\x,t}\|^2 \\
\leq & 1.5\eta^2  \|V_{\x^*,t}\|^2+\frac{6\eta^2 L}{M}\sum_{m=1}^M\|\x_t^{*,m}-\x_t^m\|^2 + {6\eta^2}\|V_{\x,t}-\widetilde{V}_{\x,t}\|^2
    \end{split}
\end{equation}
Therefore, combining all the above inequalities, for \eqref{eqn:dasdwwqddsads} the upper bound becomes
\begin{equation}
    \begin{split}
&\|\overline{\z}_t -\z\|^2 \leq \|\overline{\z}_{t-1} -\z\|^2  + 2\eta L \frac{1}{M} \sum_{m=1}^M \|\x_t^m-\overline{\x}_t\|^2 -2\eta \frac{1}{M}\sum_{m=1}^M\left\langle V(\x_t^m),\x_t^m-\z \right\rangle   \\
&+ 20\|\overline{\x}_t-\overline{\x}_t^*\|^2 + \frac{10\eta^2L^2}{M}\sum_{m=1}^M\|\x_t^m-\x_t^{*,m}\|^2 +{4\eta^2}\|V_{\x,t}-\widetilde{V}_{\x,t}\|^2+2\eta \left\langle {V}_{\x,t}-\widetilde{V}_{\x,t},\overline{\z}_{t-1}-\z \right\rangle
    \end{split}
\end{equation}
Rearrange, and taking expectation on both sides, we get 
\begin{equation}
    \begin{split}
    \label{eqn:hwoddsddfffsdrqrdfafadfasdf}
&\E\left[\frac{1}{MT}\sum_{t=1}^T\sum_{m=1}^M   \left\langle V(\x_{t}^m),\x_t^m-\z \right\rangle\right]\leq  \frac{\E[\|\overline{\z}_0-\z\|^2]}{2\eta T} +\frac{L}{MT}\sum_{t=1}^T\sum_{m=1}^M\E\left[\|\x_t^m-\overline{\x}_t\|^2\right]\\ &+\frac{5\eta L^2}{M T}\sum_{t=1}^T\sum_{m=1}^M\E\left[\|\x^m_t-\x_t^{*,m}\|^2\right]  + \frac{10}{\eta  T}\sum_{t=1}^T\E[\|\overline{\x}_t-\overline{\x}_t^*\|^2]+\E\left[\frac{1}{T}\sum_{t=1}^T\Gamma_t\right],  
    \end{split}
\end{equation}
where 
$ \Gamma_t = {2\eta}\|V_{\x,t}-\widetilde{V}_{\x,t}\|^2+\left\langle {V}_{\x,t}-\widetilde{V}_{\x,t},\overline{\z}_{t-1}-\z \right\rangle.$ 
\begin{Remark}
\label{remark:whatisthedifference}
    Compared with \eqref{eqn:fianalboundppafirst}, we could observe two differences: Firstly, the term 
    $$\frac{15}{\eta M T}\sum_{t=1}^T\sum_{m=1}^M\E\left[\|\x^m_t-\x_t^{*,m}\|^2\right] $$
    in \eqref{eqn:fianalboundppafirst} is improved to 
    $$\frac{5\eta L^2}{M T}\sum_{t=1}^T\sum_{m=1}^M\E\left[\|\x^m_t-\x_t^{*,m}\|^2\right] ,$$
    and the extra $\eta^2$ factor making this term much smaller. On the other hand, there is an extra $\frac{10}{\eta  T}\sum_{t=1}^T\E[\|\overline{\x}_t-\overline{\x}_t^*\|^2]$ term, which we show is small by using Lemma \ref{lem:sdewix1818181}.
\end{Remark}

Plugging Lemmas  \ref{lem:sdewix}, \ref{lem::draft new algor}, \ref{lemma:noeiisder}, and \ref{lem:sdewix1818181}  into \eqref{eqn:hwoddsddfffsdrqrdfafadfasdf}, and using the fact that $\eta\leq 1/L$, we get
\begin{equation}
    \begin{split}
&\E\left[\left\langle V(\z^*),\frac{1}{TM}\sum_{t=1}^T\sum_{m=1}^M\x^{m}_t-\z^* \right\rangle\right]\leq \frac{10H^2\Lambda^2G^4+60G^2+120G^2K^2}{4^H} \\ & + 120\left[\frac{\E[\|\z_0-\z^*\|^2]}{\eta KR} + L^2K^2\sigma^2\eta^4 + K\sigma^2\eta^2L+\eta^3\sigma^2L^2+\Lambda^2 \sigma^4\eta^5+\frac{\eta \sigma^2}{M} \right].
\end{split}
\end{equation}
The proof is finished by setting  
$$ \eta=\min\left\{\frac{1}{L},\frac{D^{\frac{2}{5}}}{K^{\frac{3}{5}}R^{\frac{1}{5}}\sigma^{\frac{2}{5}}L^{\frac{2}{5}}},\frac{D^{\frac{2}{3}}}{K^{\frac{2}{3}}R^{\frac{1}{3}}\sigma^{\frac{2}{3}}L^{\frac{1}{3}}},\frac{D^{\frac{1}{2}}}{K^{\frac{1}{4}}R^{\frac{1}{4}}\sigma^{\frac{1}{2}}L^{\frac{1}{2}}},\frac{D^{\frac{1}{3}}}{\Lambda^{\frac{1}{3}}\sigma^{\frac{2}{3}}R^{\frac{1}{6}}K^{\frac{1}{6}}},\frac{D\sqrt{M}}{K^{\frac{1}{2}}R^{\frac{1}{2}}\sigma}\right\}.$$

\subsection{Proof of Lemma \ref{lem:sdewix1818181}}
\label{sec:profx*prprro18}
Note that conditioned on $\x_{t,0}^m$ for $m=1,\dots,M$, we have
\begin{align*}  \E[\|\overline{\x}_{t}-\overline{\x}^*_{t}\|^2] &= \|\E[\overline{\x}_{t}]-\overline{\x}^*_{t}\|^2 + \text{Var}(\overline{\x}_{t}-\overline{\x}^*_{t}) 
  = \|\E[\overline{\x}_{t}]-\overline{\x}^*_{t}\|^2 + \frac{1}{M^2}\sum_{m=1}^M\text{Var}(\x^m_{t}-\x^{*,m}_{t}) \\
  &\leq \frac{1}{M}\sum_{m=1}^M \|\E[\x^m_{t}]-\x^{*,m}_{t}\|^2 + \frac{1}{M^2}\sum_{m=1}^M\E[\|\x^m_{t}-\x^{*,m}_{t}\|^2].
\end{align*}
Here, $\text{Var}$ denotes the variance of the random variable. Next, we bound each term respectively. For the first term, recall $F(\x)=V(\x)+\frac{1}{\eta}(\x-\z_{t-1}^m)$, $F(\x_t^{*,m})=0$. Let $G(\x)=\x-\gamma F(\x)$, we have $G(\x_t^{*,m})=\x_t^{*,m}$, i.e., $\x_t^{*,m}$ is a fixed point of $G$, and let $$\widetilde{G}(\x)=\x-\gamma \widetilde{F}(\x)=G(\x)+ \gamma({V}(\x)-\widetilde{V}(\x)).$$
Then, based on the update rule in Algorithm \ref{alg:example}, we have $\x^m_{t,\ell+1}=\widetilde{G}(\x^m_{t,\ell})$, so 
$$ \E[\x^m_{t,\ell+1}]=\E[\widetilde{G}(\x^m_{t,\ell})]= \E[{G}(\x^m_{t,\ell})].$$
Moreover, we can expand \(G\) as
\[
G(\x) \;=\; \x - \gamma\!\left(V(\x) + \tfrac{1}{\eta}(\x - \z_{t-1}^m)\right)
= \Bigl(1 - \tfrac{\gamma}{\eta}\Bigr)\x + \tfrac{\gamma}{\eta}\z_{t-1}^m - \gamma V(\x).
\]
Therefore, its Jacobian is
\[
J_G(\y) \;=\; \Bigl(1 - \tfrac{\gamma}{\eta}\Bigr)I - \gamma J_V(\y).
\]
Thus
\begin{align*}
&G(\x) - G(\y) - J_G(\y)(\x-\y) 
=\; \Bigl[\Bigl(1 - \tfrac{\gamma}{\eta}\Bigr)\x + \tfrac{\gamma}{\eta}\z_{t-1}^m - \gamma V(\x)\Bigr]
   - \Bigl[\Bigl(1 - \tfrac{\gamma}{\eta}\Bigr)\y + \tfrac{\gamma}{\eta}\z_{t-1}^m - \gamma V(\y)\Bigr] \\
&- \Bigl[\Bigl(1 - \tfrac{\gamma}{\eta}\Bigr)I - \gamma J_V(\y)\Bigr](\x-\y) 
=\;-\gamma \Bigl(V(\x) - V(\y) - J_V(\y)(\x-\y)\Bigr).
\end{align*}
Taking the norm and applying Assumption \ref{ass:boundedasddff}, we obtain
\[
\|G(\x)-G(\y)-J_G(\y)(\x-\y)\|
\;\le\gamma\Lambda\,\|\x-\y\|^2. 
\]
Let 
$$R(\x_{t,\ell}^m) = G(\x_{t,\ell}^m)-G(\E[\x_{t,\ell}^m])-\left\langle J_{G}(\E[\x_{t,\ell}]),\x_{t,\ell}^m-\E[\x_{t,\ell}^m]\right\rangle,$$
we know that $\E[R(\x_{t,\ell}^m)]=\E[G(\x_{t,\ell}^m)]-G(\E[\x_{t,\ell}^m])$. Thus, combining with Assumption \ref{ass:boundedasddff}, we have 
$$\|\E[G(\x_{t,\ell}^m)]-G(\E[\x_{t,\ell}^m)]\|=\|\E[R(\x_{t,\ell}^m)]\|\leq \E\|R(\x_{t,\ell}^m)\|\leq \lambda \Lambda\E[\|\x_{t,\ell}^m-\E[\x_{t,\ell}^m] \|^2]=\lambda \Lambda \text{Var}(\x_{t,\ell}^m).  $$
To proceed, we have 
\begin{equation}
    \begin{split}    \label{seqn:5348sss}
&\|\E[\x_{t,\ell+1}^{m}]-\x_t^{*,m}\|=\|\E[G(\x_{t,\ell}^m)]-\x_t^{*,m}\|\leq \|G(\E[\x_{t,\ell}^m])-\x_t^{*,m}\|+\|\E[G(\x_{t,\ell}^m)]-G(\E[\x_{t,\ell}^m])\|\\     \leq &\|G(\E[\x_{t,\ell}^m])-\x_t^{*,m}\|+\gamma\Lambda\text{Var}(\x_{t,\ell}^m)\leq 0.5\|\E[\x_{t,\ell}^m]-\x_{t}^{*,m}\|+\gamma\Lambda\text{Var}(\x_{t,\ell}^m)
    \end{split}
\end{equation}
Here, the last inequality is based on the fact that $G$ is a contraction map, and $\x_{t}^{*,m}$ is the fixed point. More specifically, for every $\x,\x'\in\R^d$, 
\begin{equation*}
    \begin{split}
&\|G(\x)-G(\x')\|^2=\left\| \x-\x'-\gamma\left( F(\x)-F(\x')\right)\right\|^2\\
= & \|\x-\x'\|^2 -2\gamma\left\langle \x-\x',F(\x)-F(\x')\right\rangle + \gamma^2 \|F(\x)-F(\x')\|^2\\
\leq &\left(1-\frac{2\gamma}{\eta}+\gamma^2\left(L+\frac{1}{\eta}\right)^2\right)\|\x-\x'\|^2=\left(1-\frac{1}{(\eta L+1)^2}\right)\|\x-\x'\|^2 \leq 0.25 \|\x-\x'\|^2. 
    \end{split}
\end{equation*}
Next, we would like to bound $\text{Var}(\x_{t,\ell}^m).$ We have 
\begin{equation*}
    \begin{split}
& \text{Var}(\x_{t,\ell}^m)=\E[\|\x_{t,\ell}^m-\x^{*,m}_t\|^2]-\|\E[\x_{t,\ell}^m]-\x_t^{*,m} \|^2\leq \E[\|\x_{t,\ell}^m-\x^{*,m}_t\|^2]\\
\leq & 0.25^{\ell}\E[\|\z^{m}_{t-1}-\x_t^{*,m}\|^2]+\eta^2\sigma^2\leq 0.25^{\ell}\eta^2G^2+\eta^2\sigma^2.
    \end{split}
\end{equation*}
So, based on \eqref{seqn:5348sss}, and the conclusions above, we have  
\begin{equation*}
    \begin{split}
&\| \E[\x^m_{t,H}]-\x_t^{*,m}\|\leq  0.5 \|\E[\x_{t,H-1}^m]-\x_{t}^{*,m}\|+\gamma\Lambda0.25^H\eta^2G^2 + \gamma\Lambda \eta^2\sigma^2\\
\leq & 0.5^2\|\E[\x_{t,H-2}^m]-\x_{t}^{*,m}\| + 0.5\cdot\gamma \Lambda0.25^{H-1}\eta^2G^2 + 0.5\cdot\gamma \Lambda \eta^2\sigma^2+\gamma \Lambda0.25^H\eta^2G^2 + \gamma \Lambda \eta^2\sigma^2\\
\leq & 0.5^H\E[\|\z_{t-1}^m-\x_t^{*,m}\|] + H0.5^H\gamma  \Lambda \eta^2G^2 + 2\gamma \Lambda \eta^2\sigma^2\\
\leq & 0.5^H\E[\|\z_{t-1}^m-\x_t^{*,m}\|] + H0.5^H  \Lambda \eta^3G^2 + 2 \Lambda \eta^3\sigma^2\\
\leq & 0.5^{H}\eta G + H0.5^H  \Lambda \eta^3G^2 + 2 \Lambda \eta^3\sigma^2. 
    \end{split}
\end{equation*}

Thus,
$$ \| \E[\x^m_{t,H}]-\x_t^{*,m}\|^2\leq 0.25^{H}\cdot3\eta^2 G^2 + H^20.25^H  \Lambda^2 \eta^6G^4 + 6 \Lambda^2 \eta^6\sigma^4. $$
Finally, we also have 
$$ \frac{1}{M^2}\sum_{m=1}^M\E\left[\|\x_t^m-\x_{t}^{*,m}\|^2 \right]\leq  0.25^{H}\eta^2G^2 + \frac{\eta^2\sigma^2}{M}.$$
Combining all conclusions, we get
$$\E[\|\overline{\x}_t-\overline{\x}_t^*\|^2]\leq 6\Lambda^2\eta^6\sigma^4+\frac{\eta^2\sigma^2}{M} + 0.25^{H}\cdot4\eta^2G^2 +   0.25^H\cdot H^2  \Lambda^2 \eta^6G^4. $$

\section{Proof of Theorem \ref{them:gaussiansmoothing}}
\label{secpfofinscaew}
We start by recall the following notations.  Let $\overline{\x}_t=\frac{1}{M}\sum_{m=1}^M\x_t^{m}$, and $\overline{\z}_t=\frac{1}{M}\sum_{m=1}^M\z_t^{m}$, and 
\begin{equation*}
    \begin{split}
   \overline{\z}_{t} = \frac{1}{M}\sum_{m=1}^M \z^m_t = \overline{\z}_{t-1} -\eta \frac{1}{M}\sum_{m=1}^M\widetilde{V}(\x^m_{t})=\overline{\z}_{t-1}-\widetilde{V}_{\x,t}.     
    \end{split}
\end{equation*}
where we define $\widetilde{V}_{\x,t}=\frac{1}{M}\sum_{m=1}^M\widetilde{V}(\x^m_{t})$. 
and let ${V}_{\x,t}=\frac{1}{M}\sum_{m=1}^M{V}(\x^m_{t})$. Also recall the definite of the shadow  proximal  update:
$$\x_t^{*,m}= \z_{t-1}^m -\eta V(\x_t^{*,m}),$$ 
and let $\overline{\x}_t^{*}=\frac{1}{M}\sum_{m=1}^M\x^{*,m}_t$, and $V_{\x^*,t}=\frac{1}{M}\sum_{m=1}^MV(\x_t^{*,m})$. We first introduce the following lemma. The proof is given in Appendix \ref{sec:GSppa_1}. 

\begin{lemma}
\label{lem:gaussiansmoothing_PPA}
Let $\x_t^{*,m} = \z_{t-1}^m -\eta V(\x_t^{*,m}),$  $\gamma=\frac{1}{\eta\left(L+\frac{1}{\eta}\right)^2}$, and $\eta\leq \frac{1}{L}$.   Then we  have 
$$\E[\|\overline{\x}_t-\overline{\x}_t^*\|^2]\leq 30\left[\frac{\eta^2G^2+\frac{H^2L^2d\eta^6G^4}{\delta^2}}{4^H}+\frac{\eta^2\sigma^2}{M} +\delta^2\eta^2L^2d+\frac{L^2\eta^6\sigma^4}{\delta^2}\right].$$
Moreover, 
and $$\E[\|\x_{t,\ell}^m-\x_{t}^{*,m}\|^2]\leq 4\left[\frac{\eta^2G^2}{4^H}+\eta^2\sigma^2+\eta^2L^2\delta^2d\right]. $$
\end{lemma}
Comparing with Lemma \ref{lem:sdewix1818181}, we can observe that, Gaussian smoothing introduces extra terms related to $\delta$.  We now begin the proof by first bounding the following potential function.
\begin{equation}
    \begin{split}  
&\| \overline{\z}_{t} - \z\|^2 =  \left\| \overline{\z}_{t-1} - \eta \widetilde{V}_{\x,t} - \z\right\|^2   = \|\overline{\z}_{t-1}-\z\|^2 -2\eta \left\langle \widetilde{V}_{\x,t},\overline{\z}_{t-1}-\z \right\rangle + \| \overline{\z}_{t-1}-\overline{\z}_t\|^2 \\
 = & \|\overline{\z}_{t-1}-\z\|^2 -2\eta \left\langle {V}_{\x,t},\overline{\z}_{t-1}-\z \right\rangle + \| \overline{\z}_{t-1}-\overline{\z}_t\|^2 +2\eta \left\langle {V}_{\x,t}-\widetilde{V}_{\x,t},\overline{\z}_{t-1}-\z \right\rangle\\
 = & \|\overline{\z}_{t-1}-\z\|^2 -2\eta \left\langle {V}_{\x,t},\overline{\x}_{t}-\z \right\rangle + 2\eta\left\langle {V}_{\x,t},\overline{\x}_{t}-\overline{\z}_{t-1} \right\rangle  + \| \overline{\z}_{t-1}-\overline{\z}_t\|^2 +2\eta \left\langle {V}_{\x,t}-\widetilde{V}_{\x,t},\overline{\z}_{t-1}-\z \right\rangle.
    \end{split}
\end{equation}
Next, we bound each term in respectively. For the second term, we have 
\begin{equation*}
    \begin{split}
  &-2\eta \left\langle {V}_{\x,t},\overline{\x}_{t}-\z \right\rangle =  -2\eta \frac{1}{M}\sum_{m=1}^M\left\langle V(\x_t^m),\overline{\x}_{t}-\z \right\rangle =  -2\eta \frac{1}{M}\sum_{m=1}^M\left\langle V(\x_t^m),\overline{\x}_{t}-\x_t^m+\x_t^m-\z \right\rangle\\
  = & -2\eta \frac{1}{M}\sum_{m=1}^M\left\langle V(\x_t^m),\overline{\x}_{t}-\x_t^m \right\rangle  -2\eta \frac{1}{M}\sum_{m=1}^M\left\langle V(\x_t^m),\x_t^m-\z \right\rangle\\
  = & \underbrace{-2\eta \frac{1}{M}\sum_{m=1}^M\left\langle V(\overline{\x}_t),\overline{\x}_{t}-\x_t^m \right\rangle}_{=0}  -2\eta \frac{1}{M}\sum_{m=1}^M\left\langle V(\x_t^m),\x_t^m-\z \right\rangle +2 \eta\frac{1}{M}\sum_{m=1}^M \left\langle V(\overline{\x}_t)-V({\x}^m_t),\overline{\x}_{t}-\x_t^m \right\rangle\\
 \leq & 2\eta L \frac{1}{M} \sum_{m=1}^M \|\x_t^m-\overline{\x}_t\|^2 -2\eta \frac{1}{M}\sum_{m=1}^M\left\langle V(\x_t^m),\x_t^m-\z \right\rangle,
    \end{split}
\end{equation*}
where the inequality is based on the $L$-smoothness of $V$ and the Cauchy-Schwarz inequality. 
For the third term, we have 
\begin{equation*}
    \begin{split}
    2\eta\left\langle {V}_{\x,t},\overline{\x}_{t}-\overline{\z}_{t-1} \right\rangle =  \underbrace{2\eta  \left\langle {V}_{\x,t} -V_{\x^*,t},\overline{\x}_{t}-\overline{\z}_{t-1} \right\rangle}_{A_1} + \underbrace{2\eta \left\langle {V}_{\x^*,t},\overline{\x}^*_{t}-\overline{\z}_{t-1} \right\rangle}_{A_2}   + \underbrace{2\eta \left\langle {V}_{\x^*,t},\overline{\x}_{t}-\overline{\x}^*_{t} \right\rangle.}_{A_3}
    \end{split}
\end{equation*}
For $A_1$, we have 
\begin{equation*}
\|\overline{\x}_t -\overline{\z}_{t-1}\|\leq     \|\overline{\x}_t -\overline{\x}^*_{t}\| +   \|\overline{\x}^*_t -\overline{\z}_{t-1}\|=\|\overline{\x}_t -\overline{\x}^*_{t}\| +   \eta\|V_{\x^*,t}\|
\end{equation*}
therefore 
\begin{equation*}
\begin{split}
 A_1\leq & 2\eta\|V_{\x,t} - V_{\x^*,t}\| \|\overline{\x}_t -\overline{\x}^*_{t}\| + 2\eta^2 \| V_{\x^*,t}\|\|V_{\x,t} - V_{\x^*,t}\| \\
 \leq & \eta^2\|V_{\x,t}-V_{\x^*,t}\|^2 + \|\overline{\x}_t-\overline{\x}_t^* \|^2 + \frac{\eta^2}{4}\|V_{\x^*,t}\|^2 + 4\eta^2\|V_{\x,t}-V_{\x^*,t}\|^2\\
 \leq &  \frac{5\eta^2L^2}{M}\sum_{m=1}^M\|\x_t^{m}-\x_{t}^{*}\|^2 + \|\overline{\x}_t-\overline{\x}_t^{*,m}\|^2+\frac{\eta^2}{4}\|V_{\x^*,t}\|^2
\end{split}
\end{equation*}
We also have $A_2 = -2\eta^2\|V_{\x^*,t}\|^2,$
and $A_3 \leq \frac{\eta^2}{4} \|V_{\x^*,t}\|^2+{8}\|\overline{\x}_t-\overline{\x}_t^*\|^2.$ Thus
\begin{equation*}
    \begin{split}
 2\eta\left\langle {V}_{\x,t},\overline{\x}_{t}-\overline{\z}_{t-1} \right\rangle \leq &  9\|\overline{\x}_t-\overline{\x}_t^*\|^2+\frac{5\eta^2L^2}{M}\sum_{m=1}^M\|\x_t^{m}-\x_t^{*,m}\|^2-1.5\eta^2\|V_{\x^*,t}\|^2 
    \end{split}
\end{equation*}
Finally, we have 
\begin{equation*}
    \begin{split}
&\|\overline{\z}_{t-1} -\overline{\z}_t\|^2 =    \|\overline{\z}_{t-1} -\overline{\x}_t^*+\overline{\x}_t^* -\overline{\z}_t\|^2 \leq 1.5    \|\overline{\z}_{t-1} -\overline{\x}_t^*\|^2  + 3\|\overline{\x}_t^* -\overline{\z}_t\|^2 \\
=& 1.5\eta^2 \|V_{\x^*,t}\|^2 + 3\eta^2\|V_{\x^*,t}-\widetilde{V}_{\x,t}\|^2 \\
\leq & 1.5\eta^2  \|V_{\x^*,t}\|^2+\frac{6\eta^2 L}{M}\sum_{m=1}^M\|\x_t^{*,m}-\x_t^m\|^2 + {6\eta^2}\|V_{\x,t}-\widetilde{V}_{\x,t}\|^2
    \end{split}
\end{equation*}
Therefore, combining all the above inequalities, we get 
\begin{equation*}
    \begin{split}
&\|\overline{\z}_t -\z\|^2 \leq \|\overline{\z}_{t-1} -\z\|^2  + 2\eta L \frac{1}{M} \sum_{m=1}^M \|\x_t^m-\overline{\x}_t\|^2 -2\eta \frac{1}{M}\sum_{m=1}^M\left\langle V(\x_t^m),\x_t^m-\z \right\rangle   \\
&+ 20\|\overline{\x}_t-\overline{\x}_t^*\|^2 + \frac{10\eta^2L^2}{M}\sum_{m=1}^M\|\x_t^m-\x_t^{*,m}\|^2 +{4\eta^2}\|V_{\x,t}-\widetilde{V}_{\x,t}\|^2+2\eta \left\langle {V}_{\x,t}-\widetilde{V}_{\x,t},\overline{\z}_{t-1}-\z \right\rangle
    \end{split}
\end{equation*}
Rearrange, and taking expectation on both sides, we get 
\begin{equation*}
    \begin{split}
&\E\left[\frac{1}{MT}\sum_{t=1}^T\sum_{m=1}^M   \left\langle V(\x_{t}^m),\x_t^m-\z \right\rangle\right]\leq  \frac{\E[\|\overline{\z}_0-\z\|^2]}{2\eta T} +\frac{L}{MT}\sum_{t=1}^T\sum_{m=1}^M\E\left[\|\x_t^m-\overline{\x}_t\|^2\right]\\ &+\frac{5\eta L^2}{M T}\sum_{t=1}^T\sum_{m=1}^M\E\left[\|\x^m_t-\x_t^{*,m}\|^2\right]  + \frac{10}{\eta  T}\sum_{t=1}^T\E[\|\overline{\x}_t-\overline{\x}_t^*\|^2]+\E\left[\frac{1}{T}\sum_{t=1}^T\Gamma_t\right],  
    \end{split}
\end{equation*}
where 
$ \Gamma_t = {2\eta}\|V_{\x,t}-\widetilde{V}_{\x,t}\|^2+\left\langle {V}_{\x,t}-\widetilde{V}_{\x,t},\overline{\z}_{t-1}-\z \right\rangle.$ 
Next, we have the following lemma. The proof is given in Appendix \ref{sec:lem:lem::draft new algor:LipeG}. 
\begin{lemma}
\label{lem::draft new algor:LipeG}
We have
$$\E[\|\x_t^m-\overline{\x}_t\|^2]\leq 480\left[\frac{\eta^4K^2LG^2+\eta^2G^2}{4^H}+\eta^4K^2L\sigma^2+\eta^4K^2L^3\delta^2d+\eta^2K\sigma^2+\eta^2L^2\delta^2d\right]  $$
\end{lemma}
 
we get
\begin{equation}
    \begin{split}
&\E\left[\left\langle V(\z^*),\frac{1}{TM}\sum_{t=1}^T\sum_{m=1}^M\x^{m}_t-\z^* \right\rangle\right]\leq C\Bigg[ \frac{\eta G^2+\eta^4K^2L^2G^2+\frac{H^2L^2d\eta^5G^4}{\delta^2}}{4^H}\\
&+\delta^2(\eta L^2d+\eta^4K^2L^4d) + \frac{L^2\eta^5\sigma^4} {\delta^2} + \eta^4K^2L^2\sigma^2+\eta^2LK\sigma^2+\eta^3L^2\sigma^2 + \frac{\eta\sigma^2}{M}\Bigg], 
\end{split}
\end{equation}
where $C\leq 1440$ is a constant. Let $\delta=\frac{\eta\sigma}{d^{\frac{1}{4}}}$, we get 

\begin{equation}
    \begin{split}
&\E\left[\left\langle V(\z^*),\frac{1}{TM}\sum_{t=1}^T\sum_{m=1}^M\x^{m}_t-\z^* \right\rangle\right]\leq 2C\Bigg[ \frac{\eta G^2+\eta^4K^2L^2G^2+\frac{H^2L^2d\eta^5G^4}{\delta^2}}{4^H}\\
&+\eta^3L^2\sigma^2\sqrt{d} +   \eta^4\sigma^2K^2L^2\sqrt{d} +\eta^2L^2K\sigma^2+ \frac{\eta\sigma^2}{M}+\frac{D^2}{\eta T}\Bigg]. 
\end{split}
\end{equation}
Setting
$$\eta= \left\{ \frac{D^{\frac{1}{2}}}{K^{\frac{1}{4}}R^{\frac{1}{4}}L^{\frac{1}{2}}\sigma^{\frac{1}{2}}d^{\frac{1}{8}}},\frac{D^{\frac{2}{5}}}{\sigma^{\frac{2}{5}}L^{\frac{2}{5}}d^{\frac{1}{10}}K^{\frac{3}{5}}R^{\frac{1}{5}}},\frac{D^{\frac{2}{3}}}{K^{\frac{2}{3}}R^{\frac{1}{3}}\sigma^{\frac{2}{3}}L^{\frac{2}{3}}}\right\} $$

$$\frac{G^2+K^2L^2G^2+\frac{H^2}{}}{4^H}+\frac{D^{\frac{3}{2}}L^{\frac{1}{2}}\sigma^{\frac{1}{2}}d^{\frac{1}{8}}}{K^{\frac{3}{4}}R^{\frac{3}{4}}} + \frac{D^{\frac{8}{5}}\sigma^{\frac{2}{5}}L^{\frac{2}{5}}d^{\frac{1}{10}}}{R^{\frac{4}{5}}K^{\frac{2}{5}}}+\frac{D^{\frac{4}{3}}\sigma^{\frac{2}{3}}L^{\frac{2}{3}}}{K^{\frac{1}{3}}R^{\frac{2}{3}}}+\frac{\sigma}{\sqrt{MKR}}.$$

\subsection{Proof of Lemma \ref{lem:gaussiansmoothing_PPA}}
\label{sec:GSppa_1}

Define $\widehat{V}(\x)=V(\x+\delta\s;\xi)$, and 
$$\mathring{V}(\x)=\E_{\s,\xi}[V(\x+\delta\s;\xi)]=\E[\widehat{V}(\x)].$$
It is easy to verify that $\mathring{V}(\x)$ is also monotone, since for all $\x,\y\in\R^d$,
$$ \left\langle \mathring{V}(\x)-\mathring{V}(\y),\x-\y\right\rangle = \E\left[\left\langle V(\x+\delta\s;\xi)-V(\y+\delta\s;\xi),(\y+\delta\s)-(\y+\delta\s)\right\rangle\right]\geq 0. $$
Moreover, $\mathring{V}(\x)$ is also $L$-smooth. Define $\mathring{F}(\x)=\mathring{V}(\x)+\frac{1}{\eta}(\x-\z_{t-1}^m)$. It can be seen that, the first part of Algorithm \ref{alg:example:lipeG} is essentially optimizing the VI defined by $\mathring{F}$ with LSGD. Similar to previous proof, we can observe that $\mathring{F}(\x)$ is $\frac{1}{\eta}$ strongly monotone, and $L+\frac{1}{\eta}$ smooth. Let 
$$\mathring{\x}_t^{*,m}=\z_{t-1}^{m}-\eta \mathring{V}(\mathring{\x}_t^{*,m}),$$
We know $\mathring{F}(\mathring{\x}_t^{*,m})=0$. Moreover, let $\overline{\mathring{\x}_t^*}=\frac{1}{M}\sum_{m=1}^M\mathring{\x}_t^{*,m}$. Let $\widehat{F}(\x)=\widehat{V}(\x)+\frac{1}{\eta}(\x-\z_{t-1}^m)$. Note that we have the variance $$\E\|\widehat{F}(\x)-\mathring{F}(\x)\|^2\leq L^2\delta^2d +\sigma^2. $$

In the following, we will firstly show that $\overline{\x}_t$ is close to $\overline{\mathring{\x}_t^*}$, and then prove that $\overline{\mathring{\x}_t^*}$ is close to $\overline{\x}_t^{*}$. Therefore, $\overline{\x}_t$ is also close to $\overline{\x}_t^{*}$. Note that conditioned on $\x_{t,0}^m$, for $m=1,\dots,M$,
\begin{align*}  \E[\|\overline{\x}_{t}-\overline{\mathring{\x}_t^*}\|^2] &= \|\E[\overline{\x}_{t}]-\overline{\mathring{\x}_t^*}\|^2 + \text{Var}(\overline{\x}_{t}-\overline{\mathring{\x}_t^*}) 
  = \|\E[\overline{\x}_{t}]-\overline{\mathring{\x}_t^*}\|^2 + \frac{1}{M^2}\sum_{m=1}^M\text{Var}(\x^m_{t}-\mathring{\x}^{*,m}_{t}) \\
  &\leq \frac{1}{M}\sum_{m=1}^M \|\E[\x^m_{t}]-\mathring{\x}^{*,m}_{t}\|^2 + \frac{1}{M^2}\sum_{m=1}^M\E[\|\x^m_{t}-{\mathring{\x}_t^{*,m}}\|^2].
\end{align*}
Next, we bound each term respectively. For the first term, recall $\mathring{F}(\x)=\mathring{V}(\x)+\frac{1}{\eta}(\x-\z_{t-1}^m)$, $\mathring{F}(\mathring{\x}_t^{*,m})=0$. Let $\mathring{G}(\x)=\x-\gamma \mathring{F}(\x)$, we have $\mathring{G}(\mathring{\x}_t^{*,m})=\mathring{\x}_t^{*,m}$, so $\mathring{\x}_t^{*,m}$ is a fixed point of $\mathring{G}$, and let $$\widehat{G}(\x)=\x-\gamma \widehat{F}(\x)=G(\x)+ \gamma(\mathring{V}(\x)-\widehat{V}(\x)).$$
Then, based on the update rule in Algorithm \ref{alg:example:lipeG}, we have $\x^m_{t,\ell+1}=\widehat{G}(\x^m_{t,\ell})$, so 
$$ \E[\x^m_{t,\ell+1}]=\E[\widehat{G}(\x^m_{t,\ell})]= \E[\mathring{G}(\x^m_{t,\ell})].$$
Moreover, we can expand \(\mathring{G}\) as
\[
\mathring{G}(\x) \;=\; \x - \gamma\!\left(\mathring{V}(\x) + \tfrac{1}{\eta}(\x - \z_{t-1}^m)\right)
= \Bigl(1 - \tfrac{\gamma}{\eta}\Bigr)\x + \tfrac{\gamma}{\eta}\z_{t-1}^m - \gamma \mathring{V}(\x).
\]
Therefore, its Jacobian is
\[
J_{\mathring{G}}(\x) \;=\; \Bigl(1 - \tfrac{\gamma}{\eta}\Bigr)I - \gamma J_{\mathring{V}}(\x).
\]
Thus
\begin{align*}
&\mathring{G}(\x) - \mathring{G}(\y) - J_{\mathring{G}}(\y)(\x-\y) 
=\; \Bigl[\Bigl(1 - \tfrac{\gamma}{\eta}\Bigr)\x + \tfrac{\gamma}{\eta}\z_{t-1}^m - \gamma \mathring{V}(\x)\Bigr]
   - \Bigl[\Bigl(1 - \tfrac{\gamma}{\eta}\Bigr)\y + \tfrac{\gamma}{\eta}\z_{t-1}^m - \gamma \mathring{V}(\y)\Bigr] \\
&- \Bigl[\Bigl(1 - \tfrac{\gamma}{\eta}\Bigr)I - \gamma J_{\mathring{V}}(\y)\Bigr](\x-\y) 
=\;-\gamma \Bigl(\mathring{V}(\x) - \mathring{V}(\y) - J_{\mathring{V}}(\y)(\x-\y)\Bigr).
\end{align*}
To proceed, we provide the following lemma, which shows the second-order smoothness of $\mathring{V}$. 
\begin{lemma}
We have   
\[
\|\mathring{V}(\x)-\mathring{V}(\y)-J_{\mathring{V}}(\y)(\x-\y)\|
\;\le\; \frac{L\sqrt{d}}{\delta}\|\x-\y\|^2. 
\]
\end{lemma}
\begin{proof}
Let $\mathring{V}(\x)=[\mathring{V}_1(\x),\dots,\mathring{V}_d(\x)]$. Then we have 
\begin{equation*}
    \begin{split}
\mathring{V}_i(\x)=\int_{\R^d}\E_{\xi}[V_i(\x+\delta\s;\xi)]\cdot \frac{1}{(2\pi)^{d/2}} \exp\left( -\frac{\|\s\|^2}{2}\right)d\s= \frac{1}{\delta}\int_{\R^d}\E_{\xi}[V_i(\z;\xi)]\cdot \frac{1}{(2\pi)^{d/2}} \exp\left( -\frac{\|\z-\x\|^2}{2\delta^2}\right)d\z,
    \end{split}
\end{equation*}
where we let $\z=\x+\delta\s$. Therefore, 
\begin{equation*}
    \begin{split}
  \nabla \mathring{V}_i(\x)= & \frac{1}{\delta}\int_{\R^d}\E_{\xi}[V_i(\z;\xi)]\frac{\z-\x}{(2\pi)^{d/2}\delta^2} \exp\left( -\frac{\|\z-\x\|^2}{2\delta^2}\right)d\z\\
  = &\frac{1}{\delta}\int_{\R^d}\E_{\xi}[V_i(\x+\delta\s;\xi)]\frac{\s}{(2\pi)^{d/2}} \exp\left( -\frac{\|\s\|^2}{2}\right)d\s.
    \end{split}
\end{equation*}
Thus, the Jacobian of $\mathring{V}$ can be written in the following form:
$$ J_{\mathring{V}}(\x)=\frac{1}{\delta}\E[V(\x+\delta\s;\xi)\s^{\top}]=\frac{1}{\delta}\E[(V(\x+\delta\s;\xi)-V(\x))\s^{\top}],$$
where the second equality is based on the fact that $\E[\s]=\mathbf{0}$. For any $\x,\y\in\R^d$, 
We have 
\begin{equation*}
\begin{split}
 \|J_{\mathring{V}}(\x) -  J_{\mathring{V}}(\y)\|
 &=\frac{1}{\delta}
 \Big\|
 \E\big[
 (V(\x+\delta\s;\xi)-V(\y+\delta\s;\xi)
 -V(\x;\xi)+V(\y;\xi))
 \s^{\top}
 \big]
 \Big\| \\[4pt]
 &\le \frac{1}{\delta}\E\big[
 \|V(\x+\delta\s;\xi)-V(\y+\delta\s;\xi)\|\|\s\|
 \big]
 +\frac{1}{\delta}\E\big[
 \|V(\x;\xi)-V(\y;\xi)\|\|\s\|
 \big] \\[4pt]
 &\le \frac{1}{\delta}\E[L\|\x-\y\|\|\s\|]
 +\frac{1}{\delta}\E[L\|\x-\y\|\|\s\|] = \frac{2L}{\delta}\,\E\|\s\|\,\|\x-\y\|\leq \frac{2L\sqrt{d}}{\delta}\|\x-\y\|,
\end{split}
\end{equation*}
where we consider the spectral norm for matrices. Therefore, $J_{\mathring{V}}$ is $\frac{2L\sqrt{d}}{\delta}$-Lipschitz, which finishes the proof.
\end{proof}

Taking the norm and applying the lemma above, we obtain
\[
\|\mathring{G}(\x)-\mathring{G}(\y)-J_{\mathring{G}}(\y)(\x-\y)\|
\;\le\; \frac{\gamma L\sqrt{d}}{\delta}\|\x-\y\|^2. 
\]

Let
$$\mathring{R}(\x_{t,\ell}^m) = \mathring{G}(\x_{t,\ell}^m)-\mathring{G}(\E[\x_{t,\ell}^m])-\left\langle J_{\mathring{G}}(\E[\x^m_{t,\ell}]),\x_{t,\ell}^m-\E[\x_{t,\ell}^m]\right\rangle, $$
we know that $\E[\mathring{R}(\x_{t,\ell}^m)]=\E[\mathring{G}(\x_{t,\ell}^m)]-\mathring{G}(\E[\x_{t,\ell}^m])$. Thus
$$\|\E[\mathring{G}(\x_{t,\ell}^m)]-\mathring{G}(\E[\x_{t,\ell}^m])\|=\|\E[\mathring{R}(\x_{t,\ell}^m)]\|\leq \E\|\mathring{R}(\x_{t,\ell}^m)\|\leq \frac{\gamma L\sqrt{d}}{\delta}\E[\|\x_{t,\ell}^m-\E[\x_{t,\ell}^m] \|^2]=\frac{\gamma L\sqrt{d}}{\delta} \text{Var}(\x_{t,\ell}^m).  $$
To proceed, we have 
\begin{equation*}
    \begin{split}
&\|\E[\x_{t,\ell+1}^{m}]-\mathring{\x}_t^{*,m}\|\leq \|\mathring{G}(\E[\x_{t,\ell}^m])-\mathring{\x}_t^{*,m}\|+\|\E[\mathring{G}(\x_{t,\ell}^m)]-\mathring{G}(\E[\x_{t,\ell}^m])\|\\     \leq &\|\mathring{G}(\E[\x_{t,\ell}^m])-\mathring{\x}_t^{*,m}\|+\frac{\gamma L\sqrt{d}}{\delta}\text{Var}(\x_{t,\ell}^m)\leq 0.5\|\E[\x_{t,\ell}^m]-\mathring{\x}_{t}^{*,m}\|+\frac{\gamma L\sqrt{d}}{\delta}\text{Var}(\x_{t,\ell}^m)
    \end{split}
\end{equation*}
Here, the last inequality is based on the fact that $\mathring{G}$ is also a contraction map, more specifically, for every $\x,\x'\in\R^d$, 
\begin{equation*}
    \begin{split}
&\|\mathring{G}(\x)-\mathring{G}(\x')\|^2=\left\| \x-\x'-\gamma\left( \mathring{F}(\x)-\mathring{F}(\x')\right)\right\|^2\\
= & \|\x-\x'\|^2 -2\gamma\left\langle \x-\x',\mathring{F}(\x)-\mathring{F}(\x')\right\rangle + \gamma^2 \|\mathring{F}(\x)-\mathring{F}(\x')\|^2\\
\leq &\left(1-\frac{2\gamma}{\eta}+\gamma^2\left(L+\frac{1}{\eta}\right)^2\right)\|\x-\x'\|^2=\left(1-\frac{1}{(\eta L+1)^2}\right)\|\x-\x'\|^2 \leq 0.25 \|\x-\x'\|^2. 
    \end{split}
\end{equation*}
Next, we would like to bound $\text{Var}(\x_{t,\ell}^m).$ We have 
\begin{equation*}
    \begin{split}
& \text{Var}(\x_{t,\ell}^m)=\E[\|\x_{t,\ell}^m-\mathring{\x}^{*,m}_t\|^2]-\|\E[\x_{t,\ell}^m]-\mathring{\x}_t^{*,m} \|^2\leq \E[\|\x_{t,\ell}^m-\mathring{\x}^{*,m}_t\|^2]\leq 0.25^{\ell}\eta^2G^2 + \eta^2(L^2\delta^2d+\sigma^2).
    \end{split}
\end{equation*}
where the last inequality is based on the following lemma. The proof is given in Appendix \ref{sec:lem:ppacloastoora}. 
\begin{lemma}
\label{lem:optimizeconvex-lipeg}
We have 
$$ \E[\|\x_{t,\ell}^m-\mathring{\x}_t^{*,m}\|^2]\leq 0.25^{\ell}\eta^2G^2 + \eta^2(L^2\delta^2d+\sigma^2). $$    
Moreover, we have $\|\mathring{\x}_{t}^{*,m}-\z_{t-1}^m\|\leq \eta G$, and 
\begin{equation*}
\E[\|\mathring{\x}_t^{*,m}-\x_t^{*,m}\|^2]
\le 
\eta^2L^2\delta^2 d,
\end{equation*}
and $$\E[\|\x_{t,\ell}^m-\x_{t}^{*,m}\|^2]\leq \frac{2\eta^2G^2}{4^{\ell}}+2\eta^2\sigma^2+4\eta^2L^2\delta^2d. $$
\end{lemma}
Therefore, we have
\begin{equation*}
    \begin{split}
&\| \E[\x^m_{t,H}]-\mathring{\x}_t^{*,m}\|\leq  0.5 \|\E[\x_{t,H-1}^m]-\mathring{\x}_{t}^{*,m}\|+ \frac{0.25^{H}\gamma L\sqrt{d}\eta^2G^2}{\delta} + \frac{\gamma L\sqrt{d}\eta^2(L^2\delta^2d+\sigma^2)}{\delta}\\
\leq & 0.5^2\|\E[\x_{t,H-2}^m]-\mathring{\x}_{t}^{*,m}\| + 0.5\frac{0.25^{H-1}\gamma L\sqrt{d}\eta^2G^2}{\delta} + 0.5\frac{\gamma L\sqrt{d}\eta^2(L^2\delta^2d+\sigma^2)}{\delta}   + \frac{0.25^{H}\gamma L\sqrt{d}\eta^2G^2}{\delta} + \frac{\gamma L\sqrt{d}\eta^2(L^2\delta^2d+\sigma^2)}{\delta}\\
\leq & 0.5^H\E[\|\z_{t-1}^m-\mathring{\x}_t^{*,m}\|] + H0.5^{H}\frac{\gamma L\sqrt{d}\eta^2G^2}{\delta}+\frac{2\gamma L\eta^2(L^2\delta^2d+\sigma^2)}{\delta}\\
\leq & 0.5^H\E[\|\z_{t-1}^m-\mathring{\x}_t^{*,m}\|]  + H0.5^{H}\frac{ L\sqrt{d}\eta^3G^2}{\delta}+\frac{2 L\eta^3(L^2\delta^2d+\sigma^2)}{\delta} \\
\leq & 0.5^{H}\eta G+ H0.5^{H}\frac{ L\sqrt{d}\eta^3G^2}{\delta}+\frac{2 L\eta^3(L^2\delta^2d+\sigma^2)}{\delta}.
    \end{split}
\end{equation*}

Thus,
$$ \| \E[\x^m_{t,H}]-\x_t^{*,m}\|^2\leq 0.25^{H}\cdot 3\eta^2 G^2+ 3H^20.25^{H}\frac{ L^2{d}\eta^6G^4}{\delta^2}+\frac{6 L^2\eta^6(L^2\delta^2d+\sigma^2)^2}{\delta^2} $$
Finally, we also have 
$$ \frac{1}{M^2}\sum_{m=1}^M\E\left[\|\x_t^m-\x_{t}^{*,m}\|^2 \right]\leq  0.25^{H}\eta^2G^2 + \frac{\eta^2(L^2\delta^2d+\sigma^2)}{M}.$$
Combining all conclusions, we get
$$\E[\|\overline{\x}_t-\overline{\mathring{\x}}_t^*\|^2]\leq 0.25^H\cdot4\eta^2G^2+ \frac{\eta^2(L^2\delta^2d+\sigma^2)}{M}+ 3H^20.25^{H}\frac{ L^2{d}\eta^6G^4}{\delta^2}+\frac{6 L^2\eta^6(L^2\delta^2d+\sigma^2)^2}{\delta^2}.$$
Thus combining with Lemma \ref{lem:optimizeconvex-lipeg},
\begin{equation*}
    \begin{split}
 \E[\|\overline{\x}_t-\overline{{\x}}_t^*\|^2]\leq &0.25^H\cdot8\eta^2G^2+ \frac{2\eta^2(L^2\delta^2d+\sigma^2)}{M}+ 6H^20.25^{H}\frac{ L^2{d}\eta^6G^4}{\delta^2}+\frac{12 L^2\eta^6(L^2\delta^2d+\sigma^2)^2}{\delta^2}+2\eta^2L^2\delta^2 d\\
 \leq & \frac{8\eta^2G^2+\frac{6H^2L^2d\eta^6G^4}{\delta^2}}{4^H}+\frac{2\eta^2\sigma^2}{M} +30\delta^2\eta^2L^2d+\frac{24L^2\eta^6\sigma^4}{\delta^2}
    \end{split}
\end{equation*}

\subsubsection{Proof of Lemma \ref{lem:optimizeconvex-lipeg}}
\label{sec:lem:ppacloastoora}
Recall that $$\mathring{V}(\x)=\E[V(\x+\delta\s;\xi)]=\E[\widehat{V}(\x)].$$
Moreover, $\mathring{V}(\x)$ is monotone, $L$-smooth. Recall $\mathring{F}(\x)=\mathring{V}(\x)+\frac{1}{\eta}(\x-\z_{t-1}^m)$, which is $\frac{1}{\eta}$ strongly monotone, and $L+\frac{1}{\eta}$ smooth. Let 
$$\mathring{\x}_t^{*,m}=\z_{t-1}^{m}-\eta \mathring{V}(\mathring{\x}_t^{*,m}),$$
We know $\mathring{F}(\x_t^{*,m})=0$. Next, we show the convergence. Let $\widehat{F}(\x)=\widehat{V}(\x)+\frac{1}{\eta}(\x-\z_{t-1}^m)$. Note that we have the variance $$\E\|\widehat{F}(\x)-\mathring{F}(\x)\|^2\leq L^2\delta^2d +\sigma^2. $$

since $\gamma=\frac{1}{\eta\left(L+\frac{1}{\eta}\right)^2}$, we have 
\begin{equation}
\begin{split}
&\E\big[\|\x^{m}_{t,\ell}-\mathring{\x}_t^{*,m} \|^2\big]
= \E\big[\|\x^{m}_{t,\ell-1}-\gamma\widehat{F}(\x^{m}_{t,\ell-1})
-\mathring{\x}_t^{*,m} +\gamma \mathring{F}(\mathring{\x}_t^{*,m}) \|^2 \big]\\
=&\; \E\big[\|\x^{m}_{t,\ell-1}-\mathring{\x}_t^{*,m}\|^2\big]
-2\gamma\,\E\!\left[\left\langle \mathring{F}(\x^m_{t,\ell-1}) -\mathring{F}(\mathring{\x}_t^{*,m}),\; \x_{t,\ell-1}^m-\mathring{\x}_t^{*,m}\right\rangle\right]\\
&\quad +\;\gamma^2\,\E\big[\|\widehat{F}(\x_{t,\ell-1}^m)-\mathring{F}(\mathring{\x}_t^{*,m})\|^2\big]\\
\le&\; \E\big[\|\x^{m}_{t,\ell-1}-\mathring{\x}_t^{*,m}\|^2\big]
-\frac{2\gamma}{\eta}\, \E\big[\|\x_{t,\ell-1}^m-\mathring{\x}_t^{*,m}\|^2\big]
+\gamma^2\,\E\big[\|\mathring{F}(\x_{t,\ell-1}^m)-\mathring{F}(\mathring{\x}_t^{*,m})\|^2\big]\\
&\quad +\;\gamma^2\,\E\big[\|\widehat{F}(\x_{t,\ell-1}^{m})-\mathring{F}(\x_{t,\ell-1}^{m})\|^2\big]\\
\le&\; \Big(1-\tfrac{2\gamma}{\eta}+\gamma^2\!\left(L+\tfrac{1}{\eta}\right)^2 \Big)\,
\E\big[\|\x_{t,\ell-1}^m-\mathring{\x}_t^{*,m}\|^2\big]
+\gamma^2\big(L^2\delta^2 d+\sigma^2\big)\\[2mm]
=&\; \left( 1 - \frac{1}{(\eta L+1)^2}\right)\,
\E\big[\|\x_{t,\ell-1}^m-\mathring{\x}_t^{*,m}\|^2\big]
+\gamma^2\big(L^2\delta^2 d+\sigma^2\big).
\end{split}
\end{equation}

Since $\mathring{F}(\x)=\mathring{V}(\x)+\frac{1}{\eta}(\x-\z_{t-1}^m)$ is $\frac{1}{\eta}$-strongly monotone, we have 
$$ \left\langle \mathring{F}(\mathring{\x}_t^{*,m})-\mathring{F}(\z_{t-1}^m),\mathring{\x}_{t}^{*,m}-\z_{t-1}^m \right\rangle\geq \frac{1}{\eta}\|\mathring{\x}_{t}^{*,m}-\z_{t-1}^m\|^2.$$
Also, we have $\mathring{F}(\x_{t}^{*,m})=0$, and $\mathring{F}(\z_{t-1}^m)=V(\z_{t-1}^{m}),$ so
$$ G\|\mathring{\x}_{t}^{*,m}-\z_{t-1}^m\|\geq \left\langle \mathring{F}(\mathring{\x}_t^{*,m})-\mathring{F}(\z_{t-1}^m),\mathring{\x}_{t}^{*,m}-\z_{t-1}^m \right\rangle\geq  \frac{1}{\eta}\|\mathring{\x}_{t}^{*,m}-\z_{t-1}^m\|^2, $$
so $\|\mathring{\x}_{t}^{*,m}-\z_{t-1}^m\|\leq \eta G$. Finally, note that
\begin{equation*}
\E[\|\x_{t}^m-\x_{t}^{*,m} \|^2]\leq 2\E[\|\x_{t}^m-\mathring{\x}_{t}^{*,m} \|^2]+ 2 \E[\|\mathring{\x}_{t}^{*,m}-\x_{t}^{*,m}\|^2].
\end{equation*}
In the following, we show that $\mathring{\x}_t^{*,m}$ and $\x_t^{*,m}$ are close. 
Since 
$\mathring{F}(\x)=\mathring{V}(\x)+\frac{1}{\eta}(\x-\z_{t-1}^m)$ are 
$\frac{1}{\eta}$-strongly monotone, we have 
\begin{align*}
\frac{1}{\eta}\|\mathring{\x}_t^{*,m}-\x_t^{*,m}\|^2& \le 
\langle \mathring{F}(\mathring{\x}_t^{*,m})-\mathring{F}(\x_t^{*,m}),\,
\mathring{\x}_t^{*,m}-\x_t^{*,m}\rangle= -\langle \mathring{F}(\x_t^{*,m}),\,\mathring{\x}_t^{*,m}-\x_t^{*,m}\rangle\\
&\le 
\|\mathring{F}(\x_t^{*,m})\|\,\|\mathring{\x}_t^{*,m}-\x_t^{*,m}\|\le 
\frac{\eta}{2}\|\mathring{F}(\x_t^{*,m})\|^2
+\frac{1}{2\eta}\|\mathring{\x}_t^{*,m}-\x_t^{*,m}\|^2
\end{align*}
Therefore, $\frac{1}{\eta}\|\mathring{\x}_t^{*,m}-\x_t^{*,m}\|^2
\le 
\eta\|\mathring{F}(\x_t^{*,m})\|^2$, so 
$$\frac{1}{\eta}\|\mathring{\x}_t^{*,m}-\x_t^{*,m}\|^2
\le 
\eta\|\mathring{F}(\x_t^{*,m})\|^2
=
\eta\left\|\mathring{V}(\x_t^{*,m})-\frac{1}{\eta}(\x_t^{*,m}-\z_{t-1}^{m})\right\|^2=\eta\|\mathring{V}(\x_t^{*,m})-{V}(\x_t^{*,m})\|^2.$$

Since $V$ is $L$-Lipschitz and $\mathring{V}(\x)=\E_s[V(\x+\delta \s)]$, it holds that 
$\|\mathring{V}(\x)-V(\x)\|\le L\delta\,\E\|\s\|$. Thsus, 
\begin{equation*}
\E[\|\mathring{\x}_t^{*,m}-\x_t^{*,m}\|^2]
\le 
\eta^2L^2\delta^2\,\E\|\s\|^2
\le 
\eta^2L^2\delta^2 d.
\end{equation*}
So 
$$\E[\|\x_{t,\ell}^m-\x_{t}^{*,m}\|^2]\leq \frac{2\eta^2G^2}{4^{\ell}}+2\eta^2\sigma^2+4\eta^2L^2\delta^2d. $$

\subsection{Proof of Lemma \ref{lem::draft new algor:LipeG}}
\label{sec:lem:lem::draft new algor:LipeG}

We have 
\begin{equation*}
 \begin{split}
\|\x_t^m-\x_t^{m'}\|^2\leq 3\|\x_t^m -\x_t^{*,m}\|^2 + 3 \|\x_t^{*,m}-\x_t^{*,m'}\|^2   + 3 \|\x_t^{m'}-\x_t^{*,m'}\|^2,
 \end{split}   
\end{equation*}
Note that, the first and third terms are small. Similar to the proof of Lemma \ref{lem::draft new algor}, we get 
\begin{equation*}
    \begin{split}
& \|\x_t^{*,m}-\x_t^{*,m'}\|^2  \leq \|\z_{t-1}^m-\z_{t-1}^{m'}\|^2;
    \end{split}
\end{equation*}
Next, we turn to bound the 
drift of $\z_{t-1}^m$. We have 
\begin{equation*}
    \begin{split}
&\|\z_{t}^m-\z_{t}^{m'}\|^2 = \|\z_{t-1}^m- \eta \widetilde{V}(\x_t^m) - \z_{t-1}^{m'} + \eta \widetilde{V}(\x_t^{m'})\|^2\\
=& \|\z_{t-1}^m-\z_{t-1}^{m'}\|^2 -2\eta \left\langle \widetilde{V}(\x_t^m)-\widetilde{V}(\x_t^{m'}),\z^m_{t-1}-\z_{t-1}^{m'}\right\rangle  + \eta^2 \left\| \widetilde{V}(\x_t^m)-\widetilde{V}(\x_t^{m'})\right\|^2
    \end{split}
\end{equation*}
and similarly 

\begin{equation*}
    \begin{split}
& \E\left[\|\z_t^m-\z_{t}^{m'}\|^2 \right] \leq \E[\|\z^m_{t-1}-\z^m_{t-1}\|^2] -\E\left[2\eta \left\langle \widetilde{V}(\x_t^m)-V(\x_t^{m,*})+V(\x_t^{m',*})-\widetilde{V}(\x_t^{m'}),\z^m_{t-1}-\z_{t-1}^{m'}\right\rangle\right]  \\     
&+ \E\left[8 \eta^2\left\| \widetilde{V}(\x_t^m) -V(\x_t^m)\right\|^2 + 8\eta^2 \left\| V(\x_t^m)-V(\x_t^{*,m})\right\|^2 
+ 8\eta^2\left\| V(\x_t^{*,m'})-V(\x_t^{m'})\right\|^2 \right]\\
&+ 8\eta^2\E\left[\left\|V(\x_t^{m'})-\widetilde{V}(\x_t^{m'})\right\|^2\right]\\
\leq & \E[\|\z^m_{t-1}-\z^m_{t-1}\|^2] + \E\left[2\eta^2 KL \|\x_t^m-\x_t^{*,m}\|^2 + 2\eta^2 K L\|\x_t^{m'}-\x_t^{*,m'}\|^2+\frac{1}{K}\|\z_{t-1}^m-\z_{t-1}^{m'}\|^2\right]\\
&+ 16\eta^2\sigma^2 + 8\eta^2L\E[\|\x_t^m-\x_{t}^{*,m}\|^2] + 8\eta^2L\E[\|\x_t^{m'}-\x_{t}^{*,m'}\|^2]\\
\leq & \left( 1+\frac{1}{K}\right)\E[\|\z^m_{t-1}-\z^m_{t-1}\|^2]  +  20 \eta^2KL \E[\|\x_t^{m}-\x_t^{*,m}\|^2] + 16\eta^2\sigma^2  
    \end{split}
\end{equation*}
Based on Lemma \ref{lem:optimizeconvex-lipeg}, we have 
$$ \E\left[\|\x_t^m-\x_{t}^{*,m}\|^2 \right]\leq  4\left[\frac{\eta^2G^2}{4^H}+\eta^2\sigma^2+\eta^2L^2\delta^2d\right]. $$
Thus,
\begin{equation*}
    \begin{split}
        \E\left[\|\z_t^m-\z_{t}^{m'}\|^2 \right]\leq & \left( 1+\frac{1}{K}\right)\E[\|\z^m_{t-1}-\z^m_{t-1}\|^2]  +  80 \eta^2KL\left( \frac{\eta^2G^2}{4^H}+\eta^2\sigma^2+\eta^2L^2\delta^2d\right) + 16\eta^2\sigma^2\\
        \leq & 160\left[\frac{\eta^4K^2LG^2}{4^H}+\eta^4K^2L\sigma^2+\eta^4K^2L^3\delta^2d+\eta^2K\sigma^2\right]. 
    \end{split}
\end{equation*}
Therefore, we get:
$$\E[\|\x_t^m-\x_t^{m'}\|^2]\leq 480\left[\frac{\eta^4K^2LG^2+\eta^2G^2}{4^H}+\eta^4K^2L\sigma^2+\eta^4K^2L^3\delta^2d+\eta^2K\sigma^2+\eta^2L^2\delta^2d\right]$$

\section{Proof of Theorem \ref{thm:::cocococococ}}
\label{appendix:cocococococo}
Similar to pervious proof, we define the shadow updates:  $\overline{\x}_{t+1}=\frac{1}{M}\sum_{m=1}^M\x_{t+1}^{m}$. We have 
\begin{equation}
\begin{split}
  \overline{\x}_{t+1}= \frac{1}{M}\sum_{m=1}^M\x_{t+1}^{m} = \overline{\x}_{t}-\eta \frac{1}{M}\sum_{m=1}^M\widetilde{V}(\x_{t-1}^m)=\overline{\x}_{t}-\eta\widetilde{V}_{\x,t}.  
\end{split}    
\end{equation}
Also, we define $V_{\x,t}=\frac{1}{M}\sum_{m=1}^MV(\x_t^m)$. 
Next, we start the proof by bounding the following potential function. For any $\u\in\R^d$, we have 
\begin{equation}
\begin{split}
\label{eqn:sdsafgewioitpio}
\|\overline{\x}_{t+1} -\u\|^2  = &  \|\overline{\x}_{t} - \eta\widetilde{V}_{\x,t} -\u\|^2  = \| \overline{\x}_t -\u\|^2 -2 \eta\left\langle \widetilde{V}_{\x,t} , \overline{\x}_t -\u \right\rangle + \eta^2 \left\|  \widetilde{V}_{\x,t} \right\|^2.
\end{split}
\end{equation}
Next, we bound each term in the above inequality respectively. For the third term of \eqref{eqn:sdsafgewioitpio}, we have 
\begin{equation*}
    \begin{split}
&\left\|  \widetilde{V}_{\x,t} \right\|^2 =        \left\|  \widetilde{V}_{\x,t} -V(\u)+V(\u)\right\|^2 =  \left\|  \widetilde{V}_{\x,t} -V(\u)\right\|^2 + 2 \left\langle\widetilde{V}_{\x,t} -V(\u),V(\u)  \right\rangle  + \|V(\u)\|^2\\
\leq &  \left\|  \widetilde{V}_{\x,t} -V(\u)\right\|^2 + 2 \left\langle\widetilde{V}_{\x,t} ,V(\u)  \right\rangle  
=  \left\|  \widetilde{V}_{\x,t} -V(\u)\right\|^2 + \frac{2}{M}\sum_{m=1}^M \left\langle\widetilde{V}\left(\x_t^{(m)}\right) ,V(\u)  \right\rangle  \\
\leq &2\|\widetilde{V}_{\x,t}-V_{\x,t}\|^2 + \frac{2}{M}\sum_{m=1}^M \|V(\x_t^{(m)}) - V(\u)\|^2 + \frac{2}{M\eta} \sum_{m=1}^M\left\langle  \x_{t}^{(m)}-\x_{t+1}^{(m)},V(\u)\right\rangle
    \end{split}
\end{equation*}
For the second term of \eqref{eqn:sdsafgewioitpio}, we have
\begin{equation}
    \begin{split}
       & -2 \eta\left\langle \widetilde{V}_{\x,t} , \overline{\x}_t -\u \right\rangle =  -\frac{2 \eta}{M}\sum_{m=1}^M\left\langle V(\x_t^{(m)}) , \overline{\x}_t -\u \right\rangle + 2\eta \left\langle V_{\x,t}-\widetilde{V}_{\x,t},\overline{\x}_t-\u\right\rangle\\
        = & -\frac{2 \eta}{M}\sum_{m=1}^M\left\langle V(\x_t^{(m)}) , \overline{\x}_t -\x_t^{(m)}+\x_{t}^{(m)}-\u \right\rangle  + 2\eta \left\langle V_{\x,t}-\widetilde{V}_{\x,t},\overline{\x}_t-\u\right\rangle\\
       = &  -\frac{2 \eta}{M}\sum_{m=1}^M\left\langle V(\x_t^{(m)}) ,\x_{t}^{(m)}-\u \right\rangle -\frac{2 \eta}{M}\sum_{m=1}^M\left\langle V(\x_t^{(m)}) , \overline{\x}_t -\x_t^{(m)} \right\rangle  + 2\eta \left\langle V_{\x,t}-\widetilde{V}_{\x,t},\overline{\x}_t-\u\right\rangle\\
       = & -\frac{2 \eta}{M}\sum_{m=1}^M\left\langle V(\x_t^{(m)}) ,\x_{t}^{(m)}-\u \right\rangle \underbrace{-\frac{2 \eta}{M}\sum_{m=1}^M\left\langle V(\overline{\x}_t) , \overline{\x}_t -\x_t^{(m)} \right\rangle}_{=0} + 2\eta \left\langle V_{\x,t}-\widetilde{V}_{\x,t},\overline{\x}_t-\u\right\rangle\\
       & +\frac{2\eta}{M} \sum_{m=1}^M\left\langle V(\overline{\x}_t) -V(\x_t^{(m)}) , \overline{\x}_t -\x_t^{(m)} \right\rangle
    \end{split}
\end{equation}
Combining all inequalities above, we have 
\begin{equation}
\begin{split}
0 \leq &\|\overline{\x}_t - \u\|^2 - \|\overline{\x}_{t+1} - \u\|^2  - \frac{2\eta}{M}\sum_{m=1}^M  \left\langle V(\x_t^{(m)}) ,\x_t^{(m)}-\u\right\rangle + \frac{2\eta}{M}\sum_{m=1}^M\left\langle \x_t^{(m)}-\x_{t+1}^{m},V(\u) \right\rangle \\
& + \frac{2\eta^2}{M} \sum_{m=1}^M\|V(\x_t^{(m)})-V(\u)\|^2 + {2\eta^2} \|\widetilde{V}_{\x,t}-V_{\x,t}\|^2 + 2\eta \left\langle V_{\x,t}-\widetilde{V}_{\x,t},\overline{\x}_t-\u\right\rangle\\
       & +\frac{2\eta}{M} \sum_{m=1}^M\left\langle V(\overline{\x}_t) -V(\x_t^{(m)}) , \overline{\x}_t -\x_t^{(m)} \right\rangle
\end{split}
\end{equation}
Adding $\frac{2\eta}{M}\sum_{m=1}^M\left\langle V(\u),\x_{t+1}^{(m)} -\u \right\rangle$ on both sides, we have 
\begin{equation}
\begin{split}
&\frac{2\eta}{M}\sum_{m=1}^M\left\langle V(\u),\x_{t+1}^{(m)} -\u \right\rangle \leq \|\overline{\x}_t - \u\|^2 - \|\overline{\x}_{t+1} - \u\|^2 -\frac{2\eta}{M} \sum_{m=1}^M \left\langle V(\x_t^{(m)})-V(\u),\x_t^{(m)}-\u \right\rangle  \\
& + \frac{2\eta^2}{M} \sum_{m=1}^M\|V(\x_t^{(m)})-V(\u)\|^2 + 2 \eta^2\|\widetilde{V}_{\x,t}-V_{\x,t}\|^2 + 2\eta \left\langle V_{\x,t}-\widetilde{V}_{\x,t},\overline{\x}_t-\u\right\rangle\\
       & +\frac{2\eta}{M} \sum_{m=1}^M\left\langle V(\overline{\x}_t) -V(\x_t^{(m)}) , \overline{\x}_t -\x_t^{(m)} \right\rangle\\
  \leq & \|\overline{\x}_t - \u\|^2 - \|\overline{\x}_{t+1} - \u\|^2  + 2\eta^2 \|\widetilde{V}_{\x,t}-V_{\x,t}\|^2 + 2\eta \left\langle V_{\x,t}-\widetilde{V}_{\x,t},\overline{\x}_t-\u\right\rangle\\
       & +\frac{2\eta}{M} \sum_{m=1}^M\left\langle V(\overline{\x}_t) -V(\x_t^{(m)}) , \overline{\x}_t -\x_t^{(m)} \right\rangle\\ 
       \leq & \|\overline{\x}_t - \u\|^2 - \|\overline{\x}_{t+1} - \u\|^2  + 2 \eta^2\|\widetilde{V}_{\x,t}-V_{\x,t}\|^2 + 2\eta \left\langle V_{\x,t}-\widetilde{V}_{\x,t},\overline{\x}_t-\u\right\rangle + \frac{2\eta L}{M}\sum_{m=1}^M\|\x_t^{m}-\overline{\x}_t\|^2,
\end{split}
\end{equation}
where in the second inequality we used the fact that $V$ is $\beta$-co-coercive and $\eta\leq \frac{1}{\beta}$. Summing it up from $0$ to $T-1$, dividing $2\eta T$ on both sides, letting $\u=\z^*$ and taking expecation on both sides, we have: 
we have
\begin{equation}
    \begin{split}
 \E\left[\frac{1}{MT}\left\langle  V(\u),\frac{1}{MT}\sum_{t=1}^T\sum_{m=1}^M\x^m_{t}-\z^*\right\rangle\right]\leq \frac{\E[\|\overline{\x}_0-\z^*\|^2]}{\eta T} + \frac{L}{M}\sum_{m=1}^M\E[\|\x_t^m-\overline{\x}_t\|^2]+\frac{1}{ T}\sum_{t=1}^T\E[\Gamma_t],     
    \end{split}
\end{equation}
where 
$$\Gamma_t =  \eta\|\widetilde{V}_{\x,t}-V_{\x,t}\|^2 +  \left\langle V_{\x,t}-\widetilde{V}_{\x,t},\overline{\x}_t-\u\right\rangle,$$
is the noise term. To proceed, we introduce the upper bound for the drift and $\Gamma_t$. For the drift term, we have the following bound. The proof can be found in Appendix .
\begin{lemma}
\label{lem:dddddasddjb}
    We have 
    \begin{equation}
 \E\left[\frac{ L}{M}\sum_{m=1}^M\|\x_t^m-\overline{\x}_t\|^2\right]\leq 4L\sigma^2K\eta^2 .     
    \end{equation}
\end{lemma}

Finally, we provide the upper bound for $\Gamma_t$, which is exactly the same as Lemma \ref{lemma:noeiisder}. 
\begin{lemma}
 We have 
 \begin{equation*}
    \begin{split} 
    \frac{1}{T}\E\left[\sum_{t=1}^{T}\Gamma_t\right] \leq  \frac{\E[\|\overline{\x}_0-\z^*\|^2]}{2\eta T} + \frac{2\sigma^2\eta}{M}.  
    \end{split}
\end{equation*}
\end{lemma}
Combining all conclusions, we get: 
\begin{equation*}
    \begin{split}
 \E\left[\frac{1}{MT}\left\langle  V(\u),\frac{1}{MT}\sum_{t=1}^T\sum_{m=1}^M\x^m_{t}-\z^*\right\rangle\right]\leq \frac{2D^2}{\eta T}+\frac{2\sigma^2\eta}{M}+ 4L\sigma^2K\eta^2.   
    \end{split}
\end{equation*}
Set 
$$ \eta = \min\left\{ \frac{1}{\beta}, \frac{D\sqrt{M}}{\sqrt{KR}\sigma}, \frac{D^{\frac{2}{3}}}{2^{\frac{1}{2}}K^{\frac{2}{3}}R^{\frac{1}{3}}L^{\frac{1}{3}}\sigma^{\frac{2}{3}}}\right\},$$
we get 
$$  \E\left[\frac{1}{MT}\left\langle  V(\u),\frac{1}{MT}\sum_{t=1}^T\sum_{m=1}^M\x^m_{t}-\z^*\right\rangle\right]\leq \frac{4\beta D^2}{KR}+\frac{4\sigma D}{\sqrt{MKR}} + \frac{8D^{\frac{4}{3}}L^{\frac{1}{3}}\sigma^{\frac{2}{3}}}{K^{\frac{1}{3}}R^{\frac{2}{3}}}. $$
\subsection{Proof of Lemma \ref{lem:dddddasddjb}}
We have 
\begin{equation*}
    \begin{split}
&\E[\|\x_{t}^m-\x_{t}^{m'}\|^2]=  \E\left[\left\| \x_{t-1}^m -\x_{t-1}^{m'} - \eta\left( \widetilde{V}(\x_{t-1}^{m})-\widetilde{V}(\x_{t-1}^{m'})\right) \right\|^2 \right] \\
= & \E\left[\left\|\x_{t-1}^m - \x_{t-1}^{m'}\right\|^2\right] + \eta^2\E\left[\left\|\widetilde{V}(\x_{t-1}^{m})-\widetilde{V}(\x_{t-1}^{m'}) \right\|^2\right] - 2\eta\E\left[\left\langle \x_{t-1}^m -\x_{t-1}^{m'} , \widetilde{V}(\x_{t-1}^{m})-\widetilde{V}(\x_{t-1}^{m'})  \right\rangle\right]\\
\leq & \E\left[\left\|\x_{t-1}^m - \x_{t-1}^{m'}\right\|^2\right] + \eta^2\E\left[\left\|{V}(\x_{t-1}^{m})-{V}(\x_{t-1}^{m'}) \right\|^2\right] + 4\eta^2\sigma^2  - 2\eta\E\left[\left\langle \x_{t-1}^m -\x_{t-1}^{m'} , {V}(\x_{t-1}^{m})-{V}(\x_{t-1}^{m'})  \right\rangle\right]\\
\leq & 4\eta^2\sigma^2K,
      \end{split}
\end{equation*}
where the last inequality is based on the co-coresivity of the operator.

\section{Proof of Theorem \ref{thm555}}
\label{sec:copositeprooffff}
We first recall the algorithm:
\begin{equation}
\begin{split}
        \x_{t}^{m} = & 
    \begin{cases}
\z_{t-1}^m - \eta\widetilde{V}(\u^m_{t-1}), & \text{mod}(t,K)\not= 0;\\
\frac{1}{M}\sum_{m=1}^M\left(\z_{t-1}^m - \eta\widetilde{V}(\u^m_{t-1}) \right), & \text{mod}(t,K)= 0.
    \end{cases}\\
    \z^m_{t} = & 
        \begin{cases}
\z_{t-1}^m - \eta\widetilde{V}(\v^m_t), & \text{mod}(t,K)\not= 0;\\
\frac{1}{M}\sum_{m=1}^M\left(\z_{t-1}^m - \eta\widetilde{V}(\v^m_{t}) \right), & \text{mod}(t,K)= 0.
    \end{cases}\\
 \end{split}   
\end{equation}
Here, $\u^m_t=\nabla h_t^*(\z^m_t)$, and $\v^m_t=\nabla h_t^*(\x^m_t)$ are the primal variables. 
We start by defining the following shadow updates: $\overline{\x}_t=\frac{1}{M}\sum_{m=1}^M\x_t^{m}$, and $\overline{\z}_t=\frac{1}{M}\sum_{m=1}^M\z_t^{m}$. We have 
\begin{equation*}
\begin{split}
  \overline{\x}_t= \frac{1}{M}\sum_{m=1}^M\x_t^{m} = \overline{\z}_{t-1}-\eta \frac{1}{M}\sum_{m=1}^M\widetilde{V}(\u_{t-1}^m)=\overline{\z}_{t-1}-\widetilde{V}_{\u,t-1},  
\end{split}    
\end{equation*}
and 
\begin{equation*}
    \begin{split}
   \overline{\z}_{t} = \frac{1}{M}\sum_{m=1}^M \z^m_t = \overline{\z}_{t-1} -\eta \frac{1}{M}\sum_{m=1}^M\widetilde{V}(\v^m_{t})=\overline{\z}_{t-1}-\widetilde{V}_{\v,t}.     
    \end{split}
\end{equation*}
where we define $\widetilde{V}_{\u,t-1}=\frac{1}{M}\sum_{m=1}^M\widetilde{V}(\u_{t-1}^m)$, and $\widetilde{V}_{\v,t}=\frac{1}{M}\sum_{m=1}^M\widetilde{V}(\v^m_{t})$. Moreover, let ${V}_{\u,t-1}=\frac{1}{M}\sum_{m=1}^M{V}(\u_{t-1}^m)$, and ${V}_{\v,t}=\frac{1}{M}\sum_{m=1}^M{V}(\v^m_{t})$. Finally, let $\overline{\u}_t=\nabla h_t^*(\overline{\z}_t)$, and $\overline{\v}_t=\nabla h_t^*(\overline{\x}_t)$.  Note that,  $\x$ and $\z$ are in dual, $\u$ and $\v$ are in primal. Define generalized Bregman divergence \citep{flammarion2017stochastic,bai2024local}:
$$ \widetilde{G}_{h_t}(\u,\z) = h_t(\u)-h_t(\nabla h_t^*(\z)) - \left\langle \z, \u -  \nabla h_t^*(\z) \right\rangle $$
Note that, here $\u$ is in the primal space, while $\z$ is in the dual space. 
\begin{equation}
    \begin{split}
\widetilde{G}_{h_{t}}(\overline{\u}_t,\overline{\x}_{t}) = & h_{t}(\overline{\u}_t)  - h_{t}(\overline{\v}_{t})  -\left\langle \overline{\x}_{t}, \overline{\u}_t - \overline{\v}_{t}  \right\rangle\\ 
= & h_{t}(\overline{\u}_t)  - h_{t}(\overline{\v}_{t})  -\left\langle \overline{\z}_{t-1} -\eta\widetilde{V}_{\u,t-1} , \overline{\u}_t - \overline{\v}_{t}  \right\rangle\\
= & h_{t-1}(\overline{\u}_t)  - h_{t-1}(\overline{\v}_{t}) + \eta\left(\phi(\overline{\u}_t)-\phi(\overline{\v}_t) \right) -\left\langle \overline{\z}_{t-1} -\eta\widetilde{V}_{\u,t-1} , \overline{\u}_t - \overline{\v}_{t}  \right\rangle.
    \end{split}
\end{equation}
Similarly, for any $\u$, 
\begin{equation}
    \begin{split}
\widetilde{G}_{h_t}(\u,\overline{\z}_t) = & h_t(\u)-h_t(\overline{\u}_t) - \left\langle \overline{\z}_t, \u_t -\overline{\u}_t\right\rangle  \\
= & h_{t-1}({\u})  - h_{t-1}(\overline{\u}_{t}) + \eta\left(\phi({\u})-\phi(\overline{\u}_t) \right) -\left\langle \overline{\z}_{t-1} -\eta\widetilde{V}_{\v,t} , {\u} - \overline{\u}_{t}  \right\rangle.
    \end{split}
\end{equation}
Therefore,
\begin{equation}
    \begin{split}
 &\widetilde{G}_{h_{t}}(\overline{\u}_t,\overline{\x}_{t}) + \widetilde{G}_{h_t}(\u,\overline{\z}_t) \\
 =  & \underbrace{h_{t-1}({\u})  - h_{t-1}(\overline{\v}_t)  
 - \left\langle \overline{\z}_{t-1},\u -\overline{\v}_t\right\rangle}_{A_1}   + \eta (\phi({\u})-\phi(\overline{\v}_t)) \\
 &\underbrace{+ \eta  \left\langle \widetilde{V}_{\u,t-1} , \overline{\u}_t - \overline{\v}_{t}  \right\rangle + \eta\left\langle \widetilde{V}_{\v,t} , {\u} - \overline{\u}_{t}  \right\rangle}_{A_2}
    \end{split}
\end{equation}
We have 
\begin{equation}
    \begin{split}
 A_1 = &   h_{t-1}(\u) -  h_{t-1}(\overline{\u}_{t-1})  -\left\langle \overline{\z}_{t-1}, \u - \overline{\u}_{t-1} \right\rangle  - (h_{t-1}(\overline{\v}_{t}
 ) -h_{t-1}(\overline{\u}_{t-1})-\left\langle \overline{\z}_{t-1}, \overline{\v}_{t}-\overline{\u}_{t-1}\right\rangle)\\
 = & \widetilde{G}_{h_{t-1}}(\u,\overline{\z}_{t-1}) - \widetilde{G}_{h_{t-1}}(\overline{\v}_{t},\overline{\z}_{t-1}),
    \end{split}
\end{equation}
and 
\begin{equation}
    \begin{split}
&\eta  \left\langle \widetilde{V}_{\u,t-1} , \overline{\u}_t - \overline{\v}_{t}  \right\rangle + \eta\left\langle \widetilde{V}_{\v,t} , {\u} - \overline{\u}_{t}  \right\rangle   =  -\eta\left\langle \widetilde{V}_{\v,t}, \overline{\v}_{t}-\u \right\rangle    + \eta \left\langle \widetilde{V}_{\v,{t}} -\widetilde{V}_{\u,t-1}, \overline{\v}_t -\overline{\u}_t  \right\rangle \\
\leq & - \eta\left\langle \widetilde{V}_{\v,t}, \overline{\v}_{t}-\u \right\rangle + \eta^2 \left\| \widetilde{V}_{\v,{t}} -\widetilde{V}_{\u,t-1}\right\|_*^2 + \frac{1}{2}\left\|  \overline{\v}_t -\overline{\u}_t\right\|^2
    \end{split}
\end{equation}
To summarize, we have 
\begin{equation}
    \begin{split}
&\eta \left(\phi(\overline{\v}_t)-\phi({\u}) \right) +    \widetilde{G}_{h_{t}}(\overline{\u}_t,\overline{\x}_{t}) + \widetilde{G}_{h_t}(\u,\overline{\z}_t)\\
\leq  &  \widetilde{G}_{h_{t-1}}(\u,\overline{\z}_{t-1}) - \widetilde{G}_{h_{t-1}}(\overline{\v}_{t},\overline{\z}_{t-1})  - \eta\left\langle \widetilde{V}_{\v,t}, \overline{\v}_{t}-\u \right\rangle + \eta^2 \left\| \widetilde{V}_{\v,{t}} -\widetilde{V}_{\u,t-1}\right\|_*^2 + \frac{1}{2}\left\|  \overline{\v}_t -\overline{\u}_t\right\|^2\\
\leq & \widetilde{G}_{h_{t-1}}(\u,\overline{\z}_{t-1}) - G_{h}(\overline{\v}_{t},\overline{\u}_{t-1})  - \eta\left\langle \widetilde{V}_{\v,t}, \overline{\v}_{t}-\u \right\rangle + \eta^2 \left\| \widetilde{V}_{\v,{t}} -\widetilde{V}_{\u,t-1}\right\|_*^2 + \frac{1}{2}\left\|  \overline{\v}_t -\overline{\u}_t\right\|^2\\
\leq & \widetilde{G}_{h_{t-1}}(\u,\overline{\z}_{t-1}) - \frac{1}{2}\left\|\overline{\v}_{t}-\overline{\u}_{t-1}\right\|^2  - \eta\left\langle \widetilde{V}_{\v,t}, \overline{\v}_{t}-\u \right\rangle + \eta^2 \left\| \widetilde{V}_{\v,{t}} -\widetilde{V}_{\u,t-1}\right\|_*^2 + \frac{1}{2}\left\|  \overline{\v}_t -\overline{\u}_t\right\|^2.
    \end{split}
\end{equation}
On the other hand, based on the strong convexity of $h$, we have 
\begin{equation}
    \begin{split}
\widetilde{G}_{h_t}(\overline{\u}_t,\overline{\x}_t)\geq G_h(\overline{\u}_t,\overline{\v}_t)\geq \frac{1}{2} \|\overline{\u}_t -\overline{\v}_t\|^2.       
    \end{split}
\end{equation}
Thus,  we get 
\begin{equation}
    \begin{split}
    \label{eqn:mirror:noddeqw}
\eta \left(\phi(\overline{\v}_t)-\phi({\u}) \right) \leq    \widetilde{G}_{h_{t-1}}(\u,\overline{\z}_{t-1})  -  \widetilde{G}_{h_{t}}(\u,\overline{\z}_{t})  - \eta\left\langle \widetilde{V}_{\v,t}, \overline{\v}_{t}-\u \right\rangle     - \frac{1}{2}\left\|\overline{\v}_{t}-\overline{\u}_{t-1}\right\|^2 + \eta^2 \left\| \widetilde{V}_{\v,{t}} -\widetilde{V}_{\u,t-1}\right\|_*^2
    \end{split}
\end{equation}

We first deal with the last term. 
\begin{equation*}
\begin{aligned}
&\left\| \widetilde{V}_{\u,t} - \widetilde{V}_{\v,t-1} \right\|_*^2 \\
= &\Big\|  \widetilde{V}_{\u,t} - V_{\u,t}
+ V_{\u,t} - V(\overline{\u}_t)  + V(\overline{\u}_t) - V(\overline{\v}_{t-1}) 
+ V(\overline{\v}_{t-1}) - V_{\v,t-1}
+ V_{\v,t-1} - \widetilde{V}_{\v,t-1} \Big\|_*^2 \\
\leq & 5\left\| \widetilde{V}_{\u,t} - V_{\u,t}\right\|_*^2 
+ 5\left\| V_{\u,t} - V(\overline{\u}_t) \right\|_*^2 
+ 5\left\| V(\overline{\u}_t) - V(\overline{\v}_{t-1}) \right\|_*^2 
+ 5\left\| V(\overline{\v}_{t-1}) - V_{\v,t-1} \right\|_*^2 \\
& + 5\left\| V_{\v,t-1} - \widetilde{V}_{\v,t-1} \right\|^2 \\
\leq & 5\left\| \widetilde{V}_{\u,t} - V_{\u,t}\right\|_*^2 
+ \frac{5L^2}{M}\sum_{m=1}^M\|\u_t^m - \overline{\u}_t\|^2 
+ 5L^2 \|\overline{\u}_t - \overline{\v}_{t-1}\|^2 
+ \frac{5L^2}{M}\sum_{m=1}^M\|\v_{t-1}^m - \overline{\v}_{t-1}\|_*^2 \\
& + 5\left\| V_{\v,t-1} - \widetilde{V}_{\v,t-1} \right\|_*^2.
\end{aligned}
\end{equation*}
Next, we focus on the third term at the R.H.S. of \eqref{eqn:mirror:noddeqw}. We have 
\begin{equation}
    \begin{split}
    \label{eqn:third term}
    &-   \eta \left\langle \widetilde{V}_{\v,t}, \overline{\v}_t - \u \right\rangle  
    = {} -\frac{\eta}{M}\sum_{m=1}^M \left\langle V(\v_t^m), \overline{\v}_t - \u \right\rangle 
    + \eta \left\langle V_{\v,t} - \widetilde{V}_{\v,t}, \overline{\v}_t - \u \right\rangle\\
    = {} &  -\frac{\eta}{M}\sum_{m=1}^M \left\langle V(\v_t^m), \overline{\v}_t - \v_t^m + \v_t^m - \u \right\rangle 
    + \eta \left\langle V_{\v,t} - \widetilde{V}_{\v,t}, \overline{\v}_t - \u \right\rangle\\
    = & -\frac{\eta}{M} \sum_{m=1}^M \left\langle V(\v_t^m), \v_t^m - \u \right\rangle 
    - \frac{\eta}{M} \sum_{m=1}^M \left\langle V(\v_t^m), \overline{\v}_t - \v_t^m \right\rangle 
    + \eta \left\langle V_{\v,t} - \widetilde{V}_{\v,t}, \overline{\v}_t - \u \right\rangle\\
    = & -\frac{\eta}{M} \sum_{m=1}^M \left\langle V(\v_t^m), \v_t^m - \u \right\rangle 
    - \underbrace{\frac{\eta}{M} \sum_{m=1}^M \left\langle V(\overline{\v}_t), \overline{\v}_t - \v_t^m \right\rangle}_{=0} 
    + \eta \left\langle V_{\v,t} - \widetilde{V}_{\v,t}, \overline{\v}_t - \u \right\rangle\\
    & + \frac{\eta}{M} \sum_{m=1}^M \left\langle V(\overline{\v}_t) - V(\v_t^m), \overline{\v}_t - \v_t^m \right\rangle.
    \end{split}
\end{equation}

Note that 
\begin{equation}
    \begin{split}
    \label{eqn:Evvt}
\frac{\eta}{M}\sum_{m=1}^M\left\langle V(\overline{\v}_t)-V(\v_t^m),\overline{\x}_t-\x_t^m\right\rangle\leq &\frac{\eta}{M}\sum_{m=1}^M \| V(\overline{\v}_t)-V(\v_t^m)\|_*\|\overline{\v}_t-\v_t^m\|
\leq \frac{\eta L}{M}\sum_{m=1}^M\|\overline{\v}_t-\v_t^m\|^2.
    \end{split}
\end{equation}

Combining, and the fact that $\eta\leq \frac{1}{\sqrt{10}L} $, we get 

\begin{equation}
    \begin{split}
&\left(\phi(\overline{\v}_t)-\phi({\u}) \right) + \frac{1}{M} \sum_{m=1}^M \left\langle V(\v_t^m), \v_t^m - \u \right\rangle 
\leq  \frac{\widetilde{G}_{h_{t-1}}(\u,\overline{\z}_{t-1})  -  \widetilde{G}_{h_{t}}(\u,\overline{\z}_{t}) }{\eta}  + \frac{L}{M}\sum_{m=1}^M\|\overline{\v}_t-\v_t^m\|^2 \\ &+ \frac{5L^2\eta}{M}\sum_{m=1}^M\|\u_t^m - \overline{\u}_t\|^2+ \frac{5L^2}{M}\sum_{m=1}^M\|\v_{t-1}^m - \overline{\v}_{t-1}\|^2 + \Gamma_t
    \end{split} 
\end{equation}
where 
$$\Gamma_t =  \left\langle V_{\v,t} - \widetilde{V}_{\v,t}, \overline{\v}_t - \u \right\rangle + \left\| \widetilde{V}_{\u,t} - V_{\u,t}\right\|_*^2 
 + 5\left\| V_{\v,t-1} - \widetilde{V}_{\v,t-1} \right\|_*^2.  $$
Note that we also have 
\begin{equation}
\begin{split}
    &\frac{1}{T}\sum_{t=1}^T\left(\left(\phi(\overline{\v}_t)-\phi({\u}) \right) + \frac{1}{M} \sum_{m=1}^M \left\langle V(\v_t^m), \v_t^m - \u \right\rangle\right)  \\
    \geq &\phi\left( \frac{1}{TM}\sum_{t=1}^T\sum_{m=1}^M\v^m_t\right) -\phi(\u)
     + \left\langle V(\u), \frac{1}{TM}\sum_{t=1}^T \sum_{m=1}^M\v_t^m-\u \right\rangle.
\end{split}
\end{equation}
With similar techniques, one can show that 
\begin{equation}
    \begin{split} 
    \label{eqn:gammanoicz}
    \frac{1}{T}\E\left[\sum_{t=1}^{T}\Gamma_t\right] \leq  \frac{D^2}{2\eta T} + \frac{6\sigma^2\eta}{M}.    
    \end{split}
\end{equation}

Therefore, combining with the bounded gradient assumption, it gives as 
\begin{equation*}
    \begin{split}
  &\E\left\{\sup_{\u}\left[\phi\left( \frac{1}{TM}\sum_{t=1}^T\sum_{m=1}^M\v^m_t\right) -\phi(\u)
     + \left\langle V(\u), \frac{1}{TM}\sum_{t=1}^T \sum_{m=1}^M\v_t^m-\u \right\rangle\right]\right\}\\
     \leq & \frac{D^2}{\eta KR} + \frac{6\sigma^2\eta}{M}+ {17L^2}G^2\eta^2K^2. 
    \end{split}
\end{equation*}
setting 
$$ \eta = \left\{ \frac{D\sqrt{M}}{\sigma\sqrt{6KR}},\frac{D^{\frac{2}{3}}}{17^{\frac{1}{3}}K^{\frac{1}{3}}R^{\frac{1}{3}}L^{\frac{2}{3}}G^{\frac{2}{3}}},\frac{1}{\sqrt{10}L} \right\}, $$
yields
\begin{equation*}
    \begin{split}
  &\E\left\{\sup_{\u}\left[\phi\left( \frac{1}{TM}\sum_{t=1}^T\sum_{m=1}^M\v^m_t\right) -\phi(\u)
     + \left\langle V(\u), \frac{1}{TM}\sum_{t=1}^T \sum_{m=1}^M\v_t^m-\u \right\rangle\right]\right\}\\
     \leq &  \frac{2\sqrt{6}D\sigma}{\sqrt{MKR}}+ \frac{17^{\frac{1}{3}}L^{\frac{2}{3}}D^{\frac{4}{3}}G^{\frac{2}{3}}}{R^{\frac{2}{3}}} + \frac{\sqrt{10}D^2L}{KR}. 
    \end{split}
\end{equation*}

\section{Extension to Heterogeneous Setting}
\label{sec:heterogeneous}

In this section, we consider the heterogeneous case, where $V(\x)=\frac{1}{M}\sum_{m=1}^MV_m(\x)$. We introduce the following standard assumptions.

\begin{ass}[Bounded Heterogeneity]
\label{ass:hetero}
There exists a constant $\xi>0$ such that for any $\x\in\R^d$ and $m\in[M]$, we have:
$$ \|V_m(\x)-V(\x)\|\leq \xi.$$
\end{ass}

\begin{ass}[Unbiased Stochastic Oracle]
\label{ass:unbiased:hetero}
Each client $m\in[M]$ has access to a stochastic oracle $\widetilde{V}_m(\x)$ such that $\E[\widetilde{V}_m(\x)]=V_m(\x)$ and $\E[\|\widetilde{V}_m(\x)-V_m(\x)\|^2]\leq \sigma^2$ for all $\x\in\R^d$.
\end{ass}

Under this setting, the LESGD is given by:
\begin{equation*}
\begin{split}
        \x_{t}^{m} = & 
    \begin{cases}
\z_{t-1}^m - \eta\widetilde{V}_m(\z^m_{t-1}), & \text{mod}(t,K)\not= 0;\\
\frac{1}{M}\sum_{m'=1}^M\left(\z_{t-1}^{m'} - \eta\widetilde{V}_{m'}(\z^{m'}_{t-1}) \right), & \text{mod}(t,K)= 0.
    \end{cases}\\
    \z^m_{t} = & 
        \begin{cases}
\z_{t-1}^m - \eta\widetilde{V}_m(\x^m_t), & \text{mod}(t,K)\not= 0;\\
\frac{1}{M}\sum_{m'=1}^M\left(\z_{t-1}^{m'} - \eta\widetilde{V}_{m'}(\x^{m'}_{t}) \right), & \text{mod}(t,K)= 0.
    \end{cases}\\
 \end{split}   
\end{equation*}
We have the following conclusion for LESGD under the hetrogenous setting.

\begin{theorem}[LESGD for Heterogeneous Setting]
\label{thm:ESGD:hetero}
Suppose $V(\x)=\frac{1}{M}\sum_{m=1}^MV_m(\x)$, where each $V_m$ is $L$-smooth and monotone. Suppose Assumptions \ref{ass:hetero} and \ref{ass:unbiased:hetero} hold. Let $\Z_D=\{\z\in\R^d:\|\z-\z_0\|\leq D\}$, where $D>0$ is any constant picked by the user. Let $\x_o=\frac{1}{MT}\sum_{t=1}^T\sum_{m=1}^M \x_t^m$ be the output. where $T=KR$. Set: 
\begin{equation*}
    \eta=\min\left\{\frac{1}{\sqrt{K}L}, \frac{D\sqrt{M}}{\sigma\sqrt{KR}}, \frac{D^{\frac{2}{3}}}{K^{\frac{2}{3}}R^{\frac{1}{3}}\sigma^{\frac{2}{3}}L^{\frac{1}{3}}}, \frac{D^{\frac{2}{3}}}{\xi^{\frac{2}{3}}KR^{\frac{1}{3}}L^{\frac{1}{3}}}, \frac{D}{(\xi\sigma)^{\frac{1}{2}}K^{\frac{3}{4}}R^{\frac{1}{2}}}, \frac{D}{\xi K\sqrt{R}} \right\},
\end{equation*}
 Then we have  
\begin{equation*}
    \begin{split}
&\E\left[\sup_{\z\in\Z_{D}}\left\langle V(\z),\x_o-\z \right\rangle \right] 
= O\left(\frac{LD^2}{\sqrt{K}R}+ \frac{D\sigma}{\sqrt{MKR}} + \frac{D^{\frac{4}{3}}\sigma^{\frac{2}{3}}L^{\frac{1}{3}}}{K^{\frac{1}{3}}R^{\frac{2}{3}}} + \frac{D^{\frac{4}{3}}\xi^{\frac{2}{3}}L^{\frac{1}{3}}}{R^{\frac{2}{3}}} + \frac{D(\xi\sigma)^{\frac{1}{2}}}{K^{\frac{1}{4}}R^{\frac{1}{2}}} + \frac{D\xi}{\sqrt{R}}\right).
\end{split}
\end{equation*}
\end{theorem}
We note that when $\xi=0$, the convergence rate reduces to the bound provided in Theorem \ref{thm:ESGD}. By contrast, the $O(\frac{\sigma}{\sqrt{R}})$ term in the bound established by previous work \citep{beznosikov2022decentralized} is independent of $\xi$, and thus persists even when $\xi=0$. On the other hand, for LSGD in federated optimization under the heterogeneous setting, \cite{woodworth2020minibatch} provided the following bound: $O\left(\frac{LD^2}{{K}R}+ \frac{D\sigma}{\sqrt{MKR}} + \frac{D^{\frac{4}{3}}\sigma^{\frac{2}{3}}L^{\frac{1}{3}}}{K^{\frac{1}{3}}R^{\frac{2}{3}}} + \frac{D^{\frac{4}{3}}\xi^{\frac{2}{3}}L^{\frac{1}{3}}}{R^{\frac{2}{3}}} \right).$ Compared with this bound, apart from the $\frac{1}{\sqrt{K}R}$ term which reflects the fundamental limitations of LESGD, the bound provided in Theorem \ref{thm:ESGD:hetero} contains two additional terms: $\frac{D(\xi\sigma)^{\frac{1}{2}}}{K^{\frac{1}{4}}R^{\frac{1}{2}}} + \frac{D\xi}{\sqrt{R}}$. We believe this is due to the inherent difficulty of solving VIs in the federated setting, which we discuss briefly below. Specifically, in the proof of Theorem \ref{thm:ESGD:hetero}, we show that  our objective can be decomposed as: 
\begin{equation*}
    \begin{split}   
{\left\langle V(\z), {\x}_o-\z \right\rangle}
= & {\frac{1}{TM}\sum_{t=1}^T\sum_{m=1}^M\left\langle V_m(\z),   \x_t^m -\z \right\rangle}
+ {\frac{1}{TM}\sum_{t=1}^T\sum_{m=1}^M\left\langle V_m(\z)-V(\z), \overline{\x}_t - \x^m_t
\right\rangle}.
    \end{split}
\end{equation*}
Here, the first term is what we control in the homogeneous case, while the second term arises from heterogeneity (note that the above decomposition is equality). Since this term involves an arbitrary $\z$, we can only bound it using Cauchy-Schwarz as follows:  
\begin{equation*}
    \begin{split}   
{\left\langle V(\z), {\x}_o-\z \right\rangle}
= & {\frac{1}{TM}\sum_{t=1}^T\sum_{m=1}^M\left\langle V_m(\z),   \x_t^m -\z \right\rangle}
+ {\frac{1}{TM}\sum_{t=1}^T\sum_{m=1}^M\left\langle V_m(\z)-V(\z), \overline{\x}_t - \x^m_t
\right\rangle}\\
\leq& {\frac{1}{TM}\sum_{t=1}^T\sum_{m=1}^M\left\langle V_m(\z),   \x_t^m -\z \right\rangle} + {\frac{\xi}{TM}\sum_{t=1}^T\sum_{m=1}^M\|\overline{\x}_t - \x^m_t\|},
    \end{split}
\end{equation*}
which gives rise to the additional terms.
\subsection{Proof of Theorem \ref{thm:ESGD:hetero}}
\label{sec:app:proof:hetero}

We start by defining the following shadow updates: $\overline{\x}_t=\frac{1}{M}\sum_{m=1}^M\x_t^{m}$, and $\overline{\z}_t=\frac{1}{M}\sum_{m=1}^M\z_t^{m}$. We have 
\begin{equation*}
\begin{split}
  \overline{\x}_t= \frac{1}{M}\sum_{m=1}^M\x_t^{m} = \overline{\z}_{t-1}-\eta \frac{1}{M}\sum_{m=1}^M\widetilde{V}_m(\z_{t-1}^m)=\overline{\z}_{t-1}-\eta\widetilde{V}_{\z,t-1},  
\end{split}    
\end{equation*}
and 
\begin{equation*}
    \begin{split}
   \overline{\z}_{t} = \frac{1}{M}\sum_{m=1}^M \z^m_t = \overline{\z}_{t-1} -\eta \frac{1}{M}\sum_{m=1}^M\widetilde{V}_m(\x^m_{t})=\overline{\z}_{t-1}-\eta\widetilde{V}_{\x,t}.     
    \end{split}
\end{equation*}
where we define $\widetilde{V}_{\z,t-1}=\frac{1}{M}\sum_{m=1}^M\widetilde{V}_m(\z_{t-1}^m)$, and $\widetilde{V}_{\x,t}=\frac{1}{M}\sum_{m=1}^M\widetilde{V}_m(\x^m_{t})$. Moreover, let ${V}_{\z,t-1}=\frac{1}{M}\sum_{m=1}^M{V}_m(\z_{t-1}^m)$, and ${V}_{\x,t}=\frac{1}{M}\sum_{m=1}^M{V}_m(\x^m_{t})$. Let ${\x}_o=\frac{1}{MT}\sum_{t=1}^T\sum_{m=1}^M\x_t^m=\frac{1}{T}\sum_{t=1}^T\overline{\x}_t$ be the final output. We first establish a connection between the error and the client-specific operators. We have 
\begin{equation*}
    \begin{split}   
{\left\langle V(\z), {\x}_o-\z \right\rangle} = {} & \frac{1}{T}\sum_{t=1}^T\left\langle V(\z), \overline{\x}_t-\z \right\rangle =  \frac{1}{TM}\sum_{t=1}^T\sum_{m=1}^M\left\langle V_m(\z), \overline{\x}_t-\z \right\rangle\\
= & \frac{1}{TM}\sum_{t=1}^T\sum_{m=1}^M\left\langle V_m(\z), \overline{\x}_t - \x^m_t + \x_t^m -\z \right\rangle \\
= & {\frac{1}{TM}\sum_{t=1}^T\sum_{m=1}^M\left\langle V_m(\z),   \x_t^m -\z \right\rangle} +  \underbrace{\frac{1}{TM}\sum_{t=1}^T\sum_{m=1}^M\left\langle V(\z), \overline{\x}_t - \x^m_t
\right\rangle}_{=0} \\
&+ {\frac{1}{TM}\sum_{t=1}^T\sum_{m=1}^M\left\langle V_m(\z)-V(\z), \overline{\x}_t - \x^m_t
\right\rangle}\\
\leq& {\frac{1}{TM}\sum_{t=1}^T\sum_{m=1}^M\left\langle V_m(\z),   \x_t^m -\z \right\rangle} + {\frac{\xi}{TM}\sum_{t=1}^T\sum_{m=1}^M\|\overline{\x}_t - \x^m_t\|},
    \end{split}
\end{equation*}
where the last inequality uses Assumption \ref{ass:hetero}. To proceed, we introduce the following lemma.

\begin{lemma}
\label{lem:potential:hetero}
We have
 \begin{equation*}
    \begin{split}
&\E\left[\frac{1}{TM}\sum_{t=1}^T\sum_{m=1}^M\left\langle V_m(\x_t^m),\x^m_t-\z^* \right\rangle \right]   
 \leq  \frac{\E[\|\overline{\z}_0-\z^*\|^2]}{2\eta T} + \frac{\E[\sum_{t=1}^T\Gamma_t]}{T}\\ 
 & + \frac{6L}{MT}\sum_{t=1}^T\sum_{m=1}^M\left[\E[\|\z^m_{t-1}-\overline{\z}_{t-1}\|^2] +\E[\|\x_t^m-\overline{\x}_t\|^2]\right]+\frac{2\xi}{MT}\sum_{t=1}^T\sum_{m=1}^M\E[\|\overline{\x}_t-\x_t^m\|],
    \end{split}
\end{equation*}   
where 
\begin{equation*}
    \begin{split}
 \Gamma_t= &\left\langle V_{\x,t}-\widetilde{V}_{\x,t},\overline{\x}_t-\z^* \right\rangle + \frac{5\eta}{2}\left[ \left\| \widetilde{V}_{\x,t} - V_{\x,t}\right\|^2 +\left\| V_{\z,t-1} - \widetilde{V}_{\z,t-1} \right\|^2 \right],\\
 \z^*= & \argmax_{\z\in\Z_D}\left\langle V(\z),\x_o-\z \right\rangle. 
    \end{split}
\end{equation*}
\end{lemma}

Combining Lemma \ref{lem:potential:hetero} with the above, and using the monotonicity of $V$, we get 
\begin{equation*}
    \begin{split}
&\E\left[\sup_{\z\in\Z_D}\left\langle V(\z),\x_o-\z\right\rangle \right]   
 \leq  \frac{\E[\|\overline{\z}_0-\z^*\|^2]}{2\eta T} + \frac{\E[\sum_{t=1}^T\Gamma_t]}{T}\\ 
 & + \frac{6L}{MT}\sum_{t=1}^T\sum_{m=1}^M\left[\E[\|\z^m_{t-1}-\overline{\z}_{t-1}\|^2] +\E[\|\x_t^m-\overline{\x}_t\|^2]\right]+\frac{3\xi}{MT}\sum_{t=1}^T\sum_{m=1}^M\E[\|\overline{\x}_t-\x_t^m\|].
    \end{split}
\end{equation*}
For the second term, following the same argument as in \eqref{eqn:gammanoiczth1}, we have
\begin{equation*}
    \begin{split} 
    \frac{1}{T}\E\left[\sum_{t=1}^{T}\Gamma_t\right] \leq  \frac{\E[\|\overline{\z}_0-\z^*\|^2]}{2\eta T} + \frac{6\sigma^2\eta}{M}.
    \end{split}
\end{equation*}
It remains to bound the client drift terms. We have the following lemma.

\begin{lemma}[Client Drift for Heterogeneous Setting]
\label{lem:drift:hetero}
For $\eta\leq \frac{1}{\sqrt{150K}L}$, we have 
\begin{equation*}
\E[\|\x_t^m -\x_t^{m'} \|^2]\leq C\eta^2\sigma^2K + C\eta^2\xi^2K^2,
\end{equation*}
and 
\begin{equation*}
\E[\|\z_t^m -\z_t^{m'} \|^2]\leq C\eta^2\sigma^2K + C\eta^2\xi^2K^2,
\end{equation*}
where $C>0$ is an absolute constant.
\end{lemma}

Applying Lemma \ref{lem:drift:hetero} to the above, and using Jensen's inequality for expectation, we obtain
\begin{equation*}
    \begin{split}
\E\left[\sup_{\z\in\Z_D}\left\langle V(\z),\x_o-\z\right\rangle \right]   
 = & O\left( \frac{D^2}{\eta KR} + \frac{\sigma^2\eta}{M} + L\eta^2\sigma^2K + L\eta^2\xi^2K^2 + \xi\eta\sigma\sqrt{K} + \xi\eta K\right).
    \end{split}
\end{equation*}
Setting 
$$\eta=\min\left\{\frac{1}{\sqrt{K}L}, \frac{D\sqrt{M}}{\sigma\sqrt{KR}}, \frac{D^{\frac{2}{3}}}{K^{\frac{2}{3}}R^{\frac{1}{3}}\sigma^{\frac{2}{3}}L^{\frac{1}{3}}}, \frac{D^{\frac{2}{3}}}{\xi^{\frac{2}{3}}KR^{\frac{1}{3}}L^{\frac{1}{3}}}, \frac{D}{(\xi\sigma)^{\frac{1}{2}}K^{\frac{3}{4}}R^{\frac{1}{2}}}, \frac{D}{\sqrt{\xi} K\sqrt{R}} \right\},$$
we get 
\begin{equation*}
    \begin{split}
\E\left[\sup_{\z\in\Z_D}\left\langle V(\z),\x_o-\z\right\rangle \right]   
 \leq O\left(\frac{LD^2}{\sqrt{K}R}+ \frac{D\sigma}{\sqrt{MKR}} + \frac{D^{\frac{4}{3}}\sigma^{\frac{2}{3}}L^{\frac{1}{3}}}{K^{\frac{1}{3}}R^{\frac{2}{3}}} + \frac{D^{\frac{4}{3}}\xi^{\frac{2}{3}}L^{\frac{1}{3}}}{R^{\frac{2}{3}}} + \frac{D(\xi\sigma)^{\frac{1}{2}}}{K^{\frac{1}{4}}R^{\frac{1}{2}}} + \frac{D\sqrt{\xi}}{\sqrt{R}}\right).
    \end{split}
\end{equation*}

\subsection{Proof of Lemma \ref{lem:potential:hetero}}

Following similar procedures as in Appendix \ref{sec:app:proof the1}, we have for all $\z\in\R^d$, 
\begin{equation*}
\|\overline{\z}_t-\z\|^2\leq \|\overline{\z}_{t-1}-\z\|^2 -\|\overline{\x}_t-\overline{\z}_{t-1}\|^2 - 2\eta \left\langle \widetilde{V}_{\x,t}, \overline{\x}_t-\z\right\rangle 
+ \eta^2 \left\| \widetilde{V}_{\x,t} - \widetilde{V}_{\z,t-1}\right\|^2.
\end{equation*}
For the last term of the above, note that $V(\x)=\frac{1}{M}\sum_{m=1}^MV_m(\x)$ is $L$-smooth. We have
\begin{equation*}
\begin{aligned}
&\left\| \widetilde{V}_{\x,t} - \widetilde{V}_{\z,t-1} \right\|^2 \\
\leq & 5\left\| \widetilde{V}_{\x,t} - V_{\x,t}\right\|^2 + \frac{5L^2}{M}\sum_{m=1}^M\|\x_t^m-\overline{\x}_t\|^2 + 5L^2 \|\overline{\x}_t-\overline{\z}_{t-1}\|^2 + \frac{5L^2}{M}\sum_{m=1}^M\|\z_{t-1}^m-\overline{\z}_{t-1}\|^2 \\
&+ 5\left\| V_{\z,t-1} - \widetilde{V}_{\z,t-1} \right\|^2.
\end{aligned}
\end{equation*}
For the third term, we have
\begin{equation*}
    \begin{split}
    &-2\eta \left\langle \widetilde{V}_{\x,t}, \overline{\x}_t-\z\right\rangle  = -\frac{2\eta}{M}\sum_{m=1}^M \left\langle V_m(\x_t^m),\overline{\x}_t-\z  \right\rangle + 2\eta \left\langle V_{\x,t}-\widetilde{V}_{\x,t},\overline{\x}_t-\z \right\rangle\\
= &    -\frac{2\eta}{M} \sum_{m=1}^M\left\langle V_m(\x_t^m),\x_t^m-\z\right\rangle \underbrace{- {\frac{2\eta}{M}\sum_{m=1}^M\left\langle V(\overline{\x}_t),\overline{\x}_t-\x_t^m\right\rangle}}_{=0} + 2\eta \left\langle V_{\x,t}-\widetilde{V}_{\x,t},\overline{\x}_t-\z \right\rangle\\
    & +\frac{2\eta}{M}\sum_{m=1}^M\left\langle V_m(\overline{\x}_t)-V_m(\x_t^m),\overline{\x}_t-\x_t^m\right\rangle+ \frac{2\eta}{M}\sum_{m=1}^M\left\langle V(\overline{\x}_t)-V_m(\overline{\x}_t),\overline{\x}_t -\x_t^m \right\rangle.
    \end{split}
\end{equation*}
By the $L$-smoothness of each $V_m$ and Cauchy-Schwarz, we have 
\begin{equation*}
    \begin{split}
\frac{2\eta}{M}\sum_{m=1}^M\left\langle V_m(\overline{\x}_t)-V_m(\x_t^m),\overline{\x}_t-\x_t^m\right\rangle
\leq \frac{2\eta L}{M}\sum_{m=1}^M\|\overline{\x}_t-\x_t^m\|^2.
    \end{split}
\end{equation*}
By Assumption \ref{ass:hetero},
\begin{equation*}
    \begin{split}
 \frac{2\eta}{M}\sum_{m=1}^M\left\langle V(\overline{\x}_t)-V_m(\overline{\x}_t),\overline{\x}_t-\x_t^m\right\rangle\leq \frac{2\eta\xi}{M}\sum_{m=1}^M\|\overline{\x}_t-\x_t^m\|. 
    \end{split}
\end{equation*}
Plugging all bounds into the potential function inequality and rearranging, we get
\begin{equation*}
    \begin{split}
 &\frac{1}{M}\sum_{m=1}^M \left\langle V_m(\x_t^m),\x_t^m -\z \right\rangle  \leq \frac{\left\| \overline{\z}_{t-1} -\z\right\|^2  - \|\overline{\z}_t-\z\|^2}{2\eta} + \frac{(5L^2\eta^2-1)}{2\eta}\|\overline{\x}_t-\overline{\z}_{t-1}\|^2\\
 &+   \Gamma_t + \frac{5L^2\eta}{2M}\sum_{m=1}^M\|\z^m_{t-1}-\overline{\z}_{t-1}\|^2+ \frac{(5L^2\eta + 2L)}{2M}\sum_{m=1}^M\|\x_t^m-\overline{\x}_t\|^2+\frac{2\xi}{M}\sum_{m=1}^M\|\overline{\x}_t-\x_t^m\|.
    \end{split}
\end{equation*}
Summing over $t=1,\dots,T$, dividing both sides by $T$, taking expectation, and using $\eta\leq \frac{1}{\sqrt{14}L}$, we obtain the desired result.

\subsection{Proof of Lemma \ref{lem:drift:hetero}}

We first bound the drift for $\x$. We have
\begin{equation*}
    \begin{split}
       & \E[\|\x_t^m -\x_t^{m'} \|^2]= \E\left[\left\|\z_{t-1}^m - \z_{t-1}^{m'} - \eta\left( \widetilde{V}_{m}(\z_{t-1}^m) -
       \widetilde{V}_{m'}(\z_{t-1}^{m'}) \right) \right\|^2\right]\\
         =&   \E\bigg[\bigg\|\z_{t-1}^m - \z_{t-1}^{m'} - \eta\bigg( \widetilde{V}_{m}(\z_{t-1}^m) -{V}_{m}(\z_{t-1}^m) + {V}_{m}(\z_{t-1}^m) - {V}(\z_{t-1}^m) + {V}(\z_{t-1}^m) -{V}(\z_{t-1}^{m'}) \\
      &+ {V}(\z_{t-1}^{m'}) -{V}_{m'}(\z_{t-1}^{m'}) +{V}_{m'}(\z_{t-1}^{m'}) -  \widetilde{V}_{m'}(\z_{t-1}^{m'})\bigg) \bigg\|^2\bigg] \\
      \leq & (1+\eta^2L^2) \E\left[\|\z_{t-1}^m - \z_{t-1}^{m'} \|^2\right] + 12\eta^2\sigma^2 + 12 \eta^2\xi^2\\
      \leq & 2\E\left[\|\z_{t-1}^m - \z_{t-1}^{m'} \|^2\right]   + 12\eta^2\sigma^2 + 12 \eta^2\xi^2,
    \end{split}
\end{equation*}
where in the first inequality we used the monotonicity of $V$ and Cauchy-Schwarz, and the last inequality follows from $\eta\leq \frac{1}{L}$. Next, we bound the drift for $\z$. We have  
\begin{equation*}
    \begin{split}
&\E[\|\z_{t}^m-\z_{t}^{m'}\|^2]=  \E\left[\left\| \z_{t-1}^m -\z_{t-1}^{m'} - \eta\left( \widetilde{V}_m(\x_{t}^{m})-\widetilde{V}_{m'}(\x_{t}^{m'})\right) \right\|^2 \right] \\
= & \E\left[\left\|\z_{t-1}^m - \z_{t-1}^{m'}\right\|^2\right] + \eta^2\E\left[\left\|\widetilde{V}_m(\x_{t}^{m})-\widetilde{V}_{m'}(\x_{t}^{m'}) \right\|^2\right]- 2\eta{\E\left[\left\langle \x_t^m -\x_t^{m'}, \widetilde{V}_m(\x_{t}^{m})-\widetilde{V}_{m'}(\x_{t}^{m'})  \right\rangle\right]}\\
& - 2\eta\E\left[\left\langle \z_{t-1}^m - \x_t^m+\x_{t}^{m'}-\z_{t-1}^{m'}, \widetilde{V}_m(\x_{t}^{m})-\widetilde{V}_{m'}(\x_{t}^{m'})  \right\rangle\right]\\
= & \E\left[\left\|\z_{t-1}^m - \z_{t-1}^{m'}\right\|^2\right] + \eta^2\E\left[\left\|\widetilde{V}_m(\x_{t}^{m})-\widetilde{V}_{m'}(\x_{t}^{m'}) \right\|^2\right]- 2\eta{\E\left[\left\langle \x_t^m -\x_t^{m'}, {V}(\x_{t}^{m})-{V}(\x_{t}^{m'})  \right\rangle\right]}\\
& - 2\eta{\E\left[\left\langle \x_t^m -\x_t^{m'}, ({V}_m(\x_{t}^{m})-V(\x_t^m))-(V_{m'}(\x_{t}^{m'})-V(\x_t^{m'}))  \right\rangle\right]}\\
& - 2\eta\underbrace{\E\left[\left\langle \x_t^m -\x_t^{m'}, (\widetilde{V}_m(\x_{t}^{m})-V_m(\x_t^m))-(\widetilde{V}_{m'}(\x_{t}^{m'})-V_{m'}(\x_t^{m'}))  \right\rangle\right]}_{=0}\\
& - 2\eta^2{\E\left[\left\langle \widetilde{V}_m(\z_{t-1}^m) -\widetilde{V}_{m'}(\z_{t-1}^{m'}), \widetilde{V}_m(\x_{t}^{m})-\widetilde{V}_{m'}(\x_{t}^{m'})  \right\rangle\right]}\\
\leq & \E\left[\left\|\z_{t-1}^m - \z_{t-1}^{m'}\right\|^2\right] + 2\eta^2\E\left[\left\|\widetilde{V}_m(\x_{t}^{m})-\widetilde{V}_{m'}(\x_{t}^{m'}) \right\|^2\right]+ \frac{1}{K}{\E\left[\left\|\x_t^m -\x_t^{m'}\right\|^2\right]} + 4K\eta^2\xi^2\\
& + \eta^2 \E\left[\left\|\widetilde{V}_m(\z_{t-1}^{m})-\widetilde{V}_{m'}(\z_{t-1}^{m'}) \right\|^2\right], 
      \end{split}
\end{equation*}
 To proceed, note that for any $\u,\u'\in\R^d$,
\begin{equation*}
    \begin{split}
&\E\left[\left\|\widetilde{V}_m(\u)-\widetilde{V}_{m'}(\u') \right\|^2\right] \\
\leq & 5\E\left[\|\widetilde{V}_m(\u) - V_m(\u)\|^2\right] + 5\E\left[\|V_m(\u) - V(\u)\|^2\right] + 5\E\left[\|V(\u) - V(\u')\|^2\right] + 5\E\left[\|V(\u') - V_{m'}(\u')\|^2\right] \\
&+ 5\E\left[\|V_{m'}(\u') - \widetilde{V}_{m'}(\u')\|^2\right]
\leq  5L^2\E\left[\|\u-\u'\|^2\right] + 10\sigma^2 + 10\xi^2.
    \end{split}
\end{equation*}
Thus,
\begin{equation*}
    \begin{split}
&\E[\|\z_{t}^m-\z_{t}^{m'}\|^2]
\leq \E\left[\left\|\z_{t-1}^m - \z_{t-1}^{m'}\right\|^2\right] + \frac{1}{K}{\E\left[\left\|\x_t^m -\x_t^{m'}\right\|^2\right]} + 4K\eta^2\xi^2\\
& + \eta^2 \E\left[ 5L^2\left\|\z_{t-1}^m-\z_{t-1}^{m'} \right\|^2 + 10\sigma^2+10\xi^2\right] + 2\eta^2 \E\left[ 5L^2\left\|\x_{t}^{m}-\x_{t}^{m'} \right\|^2 + 10\sigma^2+10\xi^2\right]\\
\leq & \left(1+5\eta^2L^2\right)\E\left[\left\|\z_{t-1}^m - \z_{t-1}^{m'}\right\|^2\right] + \left(\frac{1}{K}+10\eta^2L^2\right)\E\left[\left\|\x_{t}^{m}-\x_{t}^{m'} \right\|^2\right]+ 30 \eta^2\sigma^2 + 30\eta^2\xi^2 + 4K\eta^2\xi^2\\
\leq & \left(1+5\eta^2L^2\right)\E\left[\left\|\z_{t-1}^m - \z_{t-1}^{m'}\right\|^2\right] + 30 \eta^2\sigma^2 + 34K\eta^2\xi^2 \\
&+ \left(\frac{1}{K}+10\eta^2L^2\right)\left(2\E\left[\left\|\z_{t-1}^m - \z_{t-1}^{m'} \right\|^2\right]   + 12\eta^2\sigma^2 + 12 \eta^2\xi^2\right)\\
\leq & \left(1+\frac{3}{K}+25\eta^2L^2\right)\E\left[\left\|\z_{t-1}^m - \z_{t-1}^{m'}\right\|^2\right] + 50 \eta^2\sigma^2 + 50K\eta^2\xi^2\\
\leq & \left(1+\frac{4}{K}\right)\E\left[\left\|\z_{t-1}^m - \z_{t-1}^{m'}\right\|^2\right] + 50 \eta^2\sigma^2 + 50K\eta^2\xi^2
\leq 50e^{4}\eta^2\sigma^2K + 50e^{4}\eta^2\xi^2K^2,
      \end{split}
\end{equation*}
where the second-to-last inequality requires $\eta\leq \frac{1}{\sqrt{25K}L}$, and the last inequality follows from unrolling the recursion and using $\z_0^m = \z_0^{m'}$ along with $(1+\frac{4}{K})^K \leq e^{4}$. Combining with the above, we have
$$ \E[\|\x_t^m -\x_t^{m'} \|^2]\leq 100e^{4}\eta^2\sigma^2K + 100e^{4}\eta^2\xi^2K^2. $$
\end{document}